\documentclass[letterpaper]{article}
\let\appendixaddcontentsline\addcontentsline
\usepackage[preprint]{aaai2027}

\usepackage[hyphens]{url}
\usepackage{graphicx}
\usepackage{natbib}
\usepackage{caption}
\usepackage{subcaption}  %
\usepackage{float}
\usepackage{array}
\usepackage{booktabs}
\usepackage{tabularx}

\usepackage{amsmath}
\usepackage{amsthm}
\usepackage{thmtools}
\usepackage{amsfonts}
\usepackage{amssymb}
\usepackage{xspace}
\usepackage{braket}
\usepackage{tikz}
\usetikzlibrary{positioning, arrows.meta,calc,fit}
\usepackage{todonotes}

\usepackage{cleveref}
\crefname{algocf}{algorithm}{algorithms}
\Crefname{algocf}{Algorithm}{Algorithms}
\crefname{appendix}{appendix}{appendices}
\Crefname{appendix}{Appendix}{Appendices}
\usepackage{mathtools}

\usepackage[ruled,vlined]{algorithm2e}

\usepackage{xfrac}
\providecommand{\nicefrac}[2]{\sfrac{#1}{#2}}

\usepackage{comment}

\newtheorem{theorem}{Theorem}
\newtheorem*{theorem*}{Theorem}
\newtheorem{lemma}[theorem]{Lemma}
\newtheorem*{lemma*}{Lemma}

\newtheorem*{proposition*}{Proposition}

\newtheorem{definition}{Definition}
\newtheorem*{definition*}{Definition}
\newtheorem{example}{Example}
\newtheorem*{example*}{Example}
\newtheorem{problem}{Problem}

\theoremstyle{remark}

\crefname{conjecture}{conjecture}{conjectures}
\Crefname{conjecture}{Conjecture}{Conjectures}

\newcommand{\Rreg}{\mathrm{Rreg}}
\newcolumntype{L}[1]{>{\raggedright\arraybackslash}p{#1}}

\newcommand{\bQ}{\mathbb{Q}}
\newcommand{\bR}{\mathbb{R}}
\newcommand{\bE}{\mathbb{E}}

\newcommand{\sNP}{\mathsf{NP}}
\newcommand{\scoNP}{\mathsf{coNP}}
\newcommand{\sPSPACE}{\mathsf{PSPACE}}
\newcommand{\sETR}{\exists \bR}
\newcommand{\sUTR}{\forall \bR}

\newcommand{\sEUTR}{\exists \forall \bR}
\newcommand{\scoETR}{\mathsf{co}\sETR}
\newcommand{\scoUTR}{\mathsf{co}\sUTR}
\newcommand{\sSigmaTwoP}{\Sigma_2^p}

\newcommand{\sSRS}{\mathsf{SQRS}}
\newcommand{\scoSRS}{\mathsf{coSQRS}}
\newcommand{\sSignedSRS}{\mathsf{SQRS}_{\pm}}

\newcommand{\UNSAT}{\mathrm{UNSAT}}

\newcommand{\cD}{\mathcal{D}}

\newcommand{\cU}{\mathcal{U}}

\newcommand{\Cmp}{\operatorname{Cmp}}
\newcommand{\Lift}{\operatorname{Lift}}
\newcommand{\Restrict}{\operatorname{Restrict}}

\newcommand{\ie}{\emph{i.e.}\@\xspace}

\newcommand{\tup}[1]{(#1)}

\newcommand{\sinit}{s_{\iota}}

\renewcommand{\vec}[1]{{\boldsymbol{#1}}}

\newcommand{\polDet}{\Pi^{\mathrm{MD}}}
\newcommand{\polRand}{\Pi^{\mathrm{MR}}}

\newcommand{\Vmax}{{V_\mathrm{max}}}

\title{Adaptive Policy Portfolios for Robust Markov Decision Processes}
\author {
    Kasper Engelen\textsuperscript{\rm 1},
    Sebastian Junges\textsuperscript{\rm 2},
    Guillermo A. P\'{e}rez\textsuperscript{\rm 1}, 
    Marnix Suilen\textsuperscript{\rm 1}
}
\affiliations {
    \textsuperscript{\rm 1}University of Antwerp, Belgium\\
    \textsuperscript{\rm 2}Radboud University, Nijmegen, The Netherlands\\
    \{kasper.engelen,guillermo.perez,marnix.suilen\}@uantwerpen.be, sebastian.junges@ru.nl
}

\begin{document}
\maketitle

\begin{abstract}
Robust Markov decision processes optimize one policy against a set of plausible transition functions.
This can be conservative when the unknown dynamics are fixed and become partially identifiable after deployment.
We study \emph{adaptive policy portfolios}: finite sets of memoryless randomized policies synthesized offline and paired with a lightweight online selector.
Robust regret is a natural measure of portfolio quality: for each plausible environment, it measures the loss of the best portfolio member relative to the policy that would have been optimal had that environment been known.
Related regret objectives were studied by Ghavamzadeh et al. (2016) with an emphasis on approximations and relaxations for safe policy improvement.
We give a complexity-theoretic account of portfolio certification and synthesis.
Certifying a given portfolio is $\sUTR$-complete already for deterministic portfolios in acyclic $(s,a)$-rectangular RMDPs.
Synthesizing a portfolio of unary-bounded size is $\sEUTR$-complete for general rational polytopes, even with fixed discount and acyclic dynamics.
The single-policy case is already hard, both combinatorially and algebraically.
Finally, we present an offline portfolio construction that is amenable to runtime specialization.
\end{abstract}

\section{Introduction}\label{Introduction}

Markov decision processes (MDPs) are the standard formalism for sequential decision-making under uncertainty, but they require precise knowledge of transition probabilities.
\emph{Robust MDPs} (RMDPs, for short)~\citep{DBLP:journals/mor/WiesemannKR13} relax this requirement by optimizing against an uncertainty set of possible transition functions.
A robust policy is optimized against the worst transition dynamics in that set.
Consequently, robust policies can be overly conservative because their behavior is dominated by the hardest environments, even when those environments are unlikely or quickly ruled out by observation.

\emph{Robust regret}~\citep{DBLP:conf/nips/AhmedVAJ13,DBLP:conf/aaai/RigterLH21} offers an alternative.
Instead of maximizing value under the worst transition function, it minimizes the largest value shortfall relative to the policy that would have been optimal had the true transition function been known.
This yields a policy with uniformly small relative loss.
It still commits, however, to one policy for every possible transition function and cannot exploit evidence gathered online about which model governs the environment.

We propose \emph{adaptive policy portfolios}.
Offline, we compute a set of policies tailored to different regions of the uncertainty set.
Online, accumulated evidence is used to select among these already validated policies.
The portfolio is synthesized by minimizing robust regret.
We focus on portfolios of memoryless policies, the simplest practically interesting policies for sequential decision making.

\begin{figure}[t]
    \centering
    \includegraphics[width=.85\linewidth]{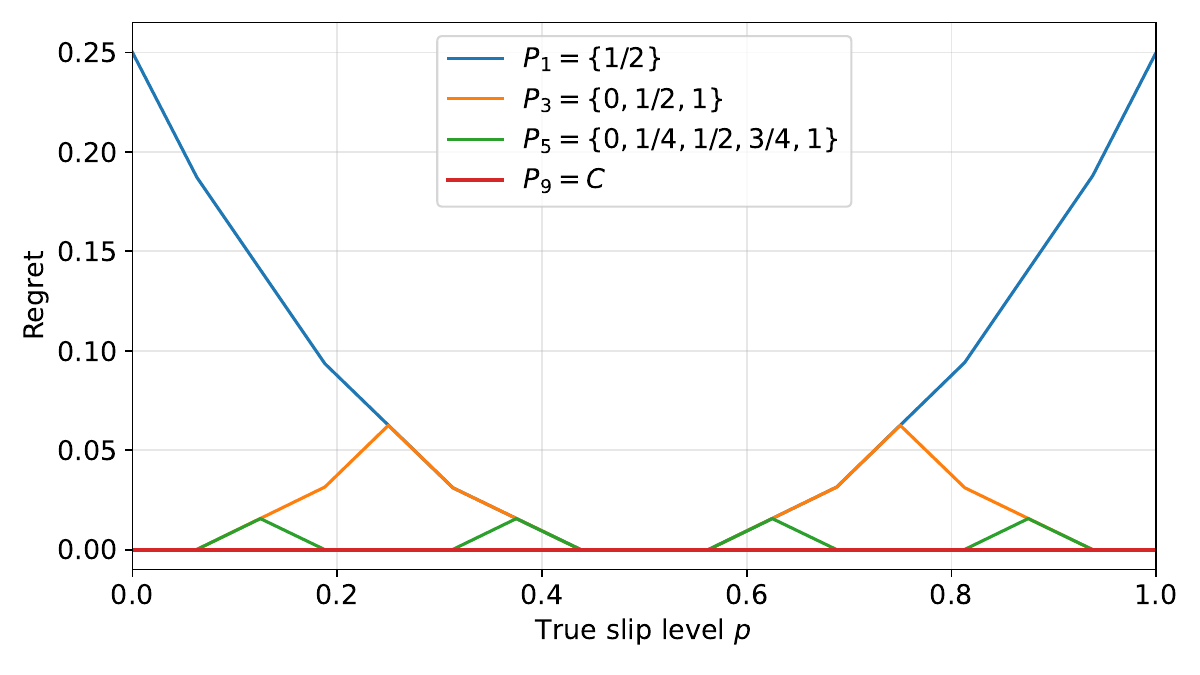}
    \caption{Portfolio regret in the actuator example.}
    \label{fig:intro:actuator}
\end{figure}
As a motivating example, consider an actuator whose true slip level is an unknown parameter $p\in[0,1]$.
Viewed as a single-state RMDP, each action is a calibration $c\in C=\{0,\nicefrac{1}{8},\dots,1\}$, and the unknown $p$ sets the success probability $1-(p-c)^2$ of reaching an absorbing success state.
The regret of a calibration is its excess squared distance to $p$ over the closest calibration in $C$.
A single robust-regret policy chooses $c=\nicefrac{1}{2}$ and has worst-case regret $\nicefrac{1}{4}$.
The portfolio $\{0,\nicefrac{1}{2},1\}$ reduces this to $\nicefrac{1}{16}$, while $\{0,\nicefrac{1}{4},\nicefrac{1}{2},\nicefrac{3}{4},1\}$ reduces it to $\nicefrac{1}{64}$.
Increasing the portfolio size provides a controlled form of adaptation while online selection remains over a finite set.
\Cref{fig:intro:actuator} plots these regret curves.

\begin{table*}[t]
\centering
\small
\begin{tabularx}{\textwidth}{@{}p{.18\textwidth}p{.15\textwidth}p{.23\textwidth}X@{}}
\toprule
Task
& Nonrect. uncertainty
& $(s,a)$-rectangular uncertainty
& Hardness already holds for \\
\midrule
Certify one given policy
& $\sUTR$-complete
& $\scoNP$-hard and $\scoSRS$-hard; in $\sUTR$
& deterministic policies; acyclic models whose uncertain choices are two-Dirac, for $\scoNP$ \\
\addlinespace
Certify a given portfolio
& $\sUTR$-complete
& $\sUTR$-complete
& deterministic portfolios with acyclic graphs and two-successor choices \\
\addlinespace
\midrule
Synthesize one randomized policy
& $\sUTR$-hard and in $\sEUTR$
& $\sNP$-hard, $\scoNP$-hard, and $\sSignedSRS$-hard; in $\sEUTR$
& two-Dirac models for the Boolean bounds \\
\addlinespace
Synthesize at most $k$ policies
& $\sEUTR$-complete
& hard already for $k=1$; exact complexity open
& regret threshold $2$, fixed discount, acyclic dynamics, two-successor uncertain choices \\
\bottomrule
\end{tabularx}
\caption{\textbf{Complexity of regret certification and synthesis.}
Certification is established by \Cref{thm:portfolio-regret-membership,thm:fixed-regret-hardness,thm:fixed-regret-forallr,thm:portfolio-regret-atr}.
Synthesis is established by \Cref{thm:min-regret-portfolio-eatr-membership,thm:min-regret-combinatorial,thm:min-regret-sqrs,thm:min-regret-forallr-hard,thm:min-portfolio-eatr}.}
\label{tab:complexity-landscape}
\end{table*}

In this work, we study the computational complexity of synthesizing minimal-regret policy portfolios and the certification and comparison subproblems that arise along the way.
\Cref{tab:complexity-landscape} summarizes our main results: the robust-regret problems are hard for different levels of the hierarchy of the theory of the reals \cite{DBLP:journals/mst/SchaeferS24}.
The $\scoNP$ lower bound was stated by \citet{DBLP:conf/nips/GhavamzadehPC16} under a policy restriction removed here.
Our square-root-sum and theory-of-the-reals bounds are new, and neither the Boolean nor the square-root-sum bound is known to be comparable with the theory-of-the-reals bounds.
This rules out efficient general-purpose algorithms under standard complexity assumptions.
These worst-case bounds still leave room for the practical offline construction developed later.

The lower bounds persist even under structural assumptions that usually simplify RMDPs.
Rectangularity is the independence assumption that transition uncertainty at one state-action pair does not constrain uncertainty at another, and it commonly lowers complexity and underlies much of the RMDP literature~\citep{DBLP:journals/mor/Iyengar05,DBLP:journals/mor/WiesemannKR13,DBLP:conf/birthday/SuilenBB0025}.
Here it does not rescue portfolio certification: although rectangularity removes the real quantifier alternation from comparison between two acyclic policies, the pointwise maximum over portfolio members restores it (\Cref{thm:portfolio-regret-atr}).
Consequently, checking a given portfolio is $\sUTR$-complete already for deterministic portfolios in acyclic $(s,a)$-rectangular RMDPs.
Choosing a portfolio under a unary size budget adds one existential real block and is $\sEUTR$-complete for general rational polytopic uncertainty.
The reduction couples choices, so it does not establish $\sEUTR$-completeness under rectangular uncertainty~(\Cref{thm:min-portfolio-eatr}).

Despite these worst-case lower bounds, we give a practical construction that discretizes the uncertainty into cells, computes a candidate policy for each cell, and evaluates every candidate against all cells.
It then clusters the resulting regret profiles to the desired portfolio budget.
At runtime, UCB \cite{pred-learn-games} selects among the portfolio members as observations accumulate.

\paragraph{Contributions.}
Our contributions are as follows.
\begin{enumerate}
    \item We introduce adaptive policy portfolios as a finite and certifiable form of adaptation for RMDPs.
    \item We prove the two exact portfolio classifications summarized in \Cref{tab:complexity-landscape}: given-portfolio certification is $\sUTR$-complete and bounded-portfolio synthesis is $\sEUTR$-complete (\Cref{thm:portfolio-regret-atr,thm:min-portfolio-eatr}).
    \item We identify the structural sources of hardness: uncertain transitions used by both policies yield Boolean hardness, their recurrence on cycles yields hardness from sums of square roots, and choosing a policy adds further Boolean and signed-square-root-sum hardness.
    \item We present an offline portfolio construction that is amenable to runtime specialization.
\end{enumerate}
The main text gives the principal arguments.
Complete proofs and supporting comparison results appear in the appendix.

\subsection*{Related Work}
Several approaches reduce the conservativeness of worst-case robust policies by optimizing over less conservative uncertainty sets~\citep{DBLP:journals/mor/WiesemannKR13,DBLP:journals/corr/abs-2602-03381}.
When multiple optimal robust policies exist, best-effort approaches can select among them without sacrificing worst-case value~\citep{DBLP:conf/aaai/AbateBGF26}.

A more adaptive response to fixed but initially unknown transition dynamics is provided by Bayes-adaptive MDPs~\citep{DBLP:conf/nips/GuezSD12}, hidden-model MDPs~\citep{DBLP:conf/aaai/ChadesCMNSB12}, and hidden-parameter MDPs~\citep{DBLP:conf/ijcai/Doshi-VelezK16}.
These models plan over a belief on the unknown dynamics and update it as evidence accumulates.
This formulation is expressive, but planning becomes intractable outside narrow special cases.
Its belief representation can be continuous or high-dimensional, and online belief-space planning can itself be expensive.

Adaptive policy portfolios occupy an intermediate position.
Their expensive computation is performed offline, while online adaptation reduces to selecting among a fixed set of policies.
Regret is natural in this setting because the portfolio should contain a policy that performs near-optimally for whichever environment turns out to be true (cf.\ \citet{DBLP:conf/nips/GhavamzadehPC16}).

\section{Problem Statement}\label{sec:preliminaries}

For a finite set $X$, we write $|X|$ for its cardinality, and $\cD(X)$ for the set of discrete probability distributions over $X$, \ie, functions $\mu\colon X\to[0,1]$ with $\sum_{x\in X}\mu(x)=1$.
Vectors and matrices are boldface.
For $\vec x\in\bR^n$, $\vec x(i)$ denotes its $i$-th entry.
Throughout, $S$ and $A$ denote finite sets of states and actions of a decision process, and we call a pair $(s,a)\in S\times A$ a \emph{choice}.
We fix an arbitrary ordering of the triples $(s,a,s')\in S\times A\times S$, so that a vector $\vec u\in\bR^{S\times A\times S}$ assigns a value $\vec u(s,a,s')$ to each triple.
We call $\vec u$ a \emph{transition vector} if $\vec u(s,a,\cdot)\in\cD(S)$ for every choice $(s,a)$.

\begin{definition}[RMDP]
A \emph{robust Markov decision process} (RMDP) is a tuple $\tup{S,A,\cU,R,\sinit,\gamma}$, where $R\colon S\times A\to\bR$ is the reward function, $\sinit\in S$ is the initial state, $\gamma\in(0,1)$ is a discount factor, and
\[
\cU=\{\vec u\in\bR^{S\times A\times S} \mid \vec F\vec u\leq \vec g\}
\]
is a convex polytope of transition vectors, for $\vec F\in\bR^{m\times n}$, $\vec g\in\bR^m$ ($n=|S||A||S|$).
The inequalities $\vec F\vec u\leq\vec g$ include the constraints defining a transition vector, so every $\vec u\in\cU$ is a valid transition function.
\end{definition}
A policy maps paths to distributions over actions, $\pi\colon(SA)^*S\to\cD(A)$.
A \emph{memoryless randomized} policy $\pi\colon S\to\cD(A)$ depends only on the current state.
The set of such policies is $\polRand$.
A \emph{memoryless deterministic} policy $\pi\colon S\to A$ further selects a single action per state.
The set of such policies is $\polDet\subseteq\polRand$.
Unless stated otherwise, candidate policies range over $\polRand$.

Each $\vec u\in\cU$ induces a classical MDP $M_{\vec u}=\tup{S,A,P_{\vec u},R,\sinit,\gamma}$ with $P_{\vec u}(s,a)(s')=\vec u(s,a,s')$.
For a policy $\pi$, its value $V^\pi_{\vec u}\colon S\to\bR$ in $M_{\vec u}$ is the expected discounted cumulative reward,
\[
V^\pi_{\vec u}(s) = \bE_{\pi}\Bigl[\sum_{t=0}^{\infty}\gamma^t R(s_t,a_t) \Bigm| s_0=s\Bigr],
\]
and $V^*_{\vec u}(s)=\sup_{\pi}V^\pi_{\vec u}(s)$ is the optimal value.
This supremum is attained by a policy in $\polDet$ and $V^*_{\vec u}$ is computable in polynomial time for fixed $\vec u$~\cite{DBLP:books/wi/Puterman94}.

\paragraph{Types of uncertainty.}
By default, $\cU$ places no further structure on how uncertainty interacts across choices.
We call this the \emph{general} (\emph{non-rectangular}, \emph{convex-polytopic}) case.
The special case where $\cU$ decomposes into independent per-choice uncertainty sets is \emph{$(s,a)$-rectangular}:
\[
\cU=\bigtimes_{(s,a)\in S\times A}\cU_{(s,a)}, \qquad \cU_{(s,a)}\subseteq\bR^{\{s\}\times\{a\}\times S}.
\]
We call $\cU_{(s,a)}$ the uncertainty set of choice $(s,a)$, and the choice \emph{uncertain} if $\cU_{(s,a)}$ is not a singleton.
Intuitively, $(s,a)$-rectangularity is an independence assumption across choices.
The general case allows $\cU$ to couple choices, whereas $(s,a)$-rectangularity prohibits this.

\paragraph{Subclasses of RMDPs.}
An uncertain choice is \emph{two-successor} if every distribution in its uncertainty set is supported on the same two states, and \emph{two-Dirac} if that set is the line segment between two Dirac distributions.
An RMDP has either property when all its uncertain choices do.
An RMDP is \emph{acyclic} when its possible-transition graph has no directed cycle except self-loops at absorbing final states.
A policy graph is acyclic under the analogous restriction to choices used by that policy.
Constructions that describe only their intended choices are completed on every other state-action pair by the ruinous-sink convention of \Cref{app:primitives}.

\paragraph{Robust Value \& Regret.}
For a policy $\pi$, its \emph{robust value} is
\[
V^{\pi,\mathrm{rob}}(s)=\inf_{\vec u\in\cU}V_{\vec u}^\pi(s),
\]
the worst-case value $\pi$ can guarantee against any realization in $\cU$.
The \emph{robust optimal value} is $V^{\mathrm{rob}}(s)=\sup_{\pi}V^{\pi,\mathrm{rob}}(s)$, and a policy attaining it is \emph{robust optimal}.
Robust value compares policies in absolute terms, but it does not distinguish a policy performing poorly from a realization that is simply hard for every policy: even the best policy for a fixed $\vec u$ may attain a low value there.
We therefore turn to a comparative notion, measuring a policy's shortfall against the best policy for each particular realization, rather than its raw worst-case value.
That is, for $\pi\in\polRand$, its \emph{robust regret} (at $\sinit$) is
\[
\Rreg(\pi)=\sup_{\vec u\in\cU}\bigl(V_{\vec u}^*(\sinit)-V_{\vec u}^\pi(\sinit)\bigr),
\]
the largest shortfall of $\pi$ from the policy optimal at $\vec u$, over every realization of the uncertainty.
We call the choice of realization $\vec u$ nature's move.

\begin{definition}[Adaptive policy portfolio]
An \emph{adaptive policy portfolio} is a finite set $\Pi\subseteq\polRand$ of candidate policies synthesized offline and paired with a runtime mechanism that selects among its members as evidence accumulates.
\end{definition}
Robust regret lifts to portfolios by comparing against the best member \emph{for that realization}:
\begin{equation}\label{eqn:rreg}
\Rreg(\Pi)=\sup_{\vec u\in\cU}\Bigl(V_{\vec u}^*(\sinit)-\max_{\pi\in\Pi}V_{\vec u}^\pi(\sinit)\Bigr).
\end{equation}
This is an offline coverage guarantee: it idealizes away the cost of identifying the best member and assumes that member is selected.
The online identification cost is evaluated separately in \Cref{sec:implementation}.
This quantity cannot in general be recovered from the single-policy regret values $\{\Rreg(\pi):\pi\in\Pi\}$.
Portfolio regret evaluates all members under the same realization before taking the worst case, whereas each singleton regret has already taken its own supremum.
Note that the bound $\Rreg(\Pi)\leq\min_{\pi\in\Pi}\Rreg(\pi)$ follows.

We seek the best guarantee achievable by a portfolio of bounded size $k\ge1$:
\[
\rho_k=\inf_{\Pi\subseteq\polRand: |\Pi|\leq k}\Rreg(\Pi).
\]
Note $\rho_1=\inf_{\pi\in\polRand}\Rreg(\pi)$: minimal single-policy regret is the $k=1$ case of minimal portfolio regret.

\section{Computational Complexity}\label{sec:complexity}

 We first recall the required classes, then turn to the certification and synthesis of portfolios.
 We note that the results yield a rich landscape, but all studied problems lie in $\sPSPACE$.

\subsection{Complexity Preliminaries}
We assume basic familiarity with standard complexity classes such as $\sNP$, $\scoNP$, and $\sPSPACE$~\citep{DBLP:books/daglib/0072413,DBLP:books/daglib/0023084}.
\paragraph{Input conventions.}
The rewards, $\gamma$, threshold $t$, polytope data $\vec F,\vec g$, and policy matrices in $\mathbb{Q}^{S\times A}$ are binary-encoded rationals.
The portfolio budget $k$ is encoded in unary.

\paragraph{Theory of the reals.}
We use the first-order theory of the reals over $(\bR,0,1,+,\cdot,<)$.
In this logic,  we can ask whether a system of polynomial equations and inequalities has a real solution, or holds for every real assignment.
We consider three fragments, given a quantifier-free formula $\varphi$:
\begin{itemize}
    \item \emph{existential} sentences $\exists x_1,\dots,x_n\ \varphi(x_1,\dots,x_n)$;
    \item \emph{universal} sentences $\forall x_1,\dots,x_n\ \varphi(x_1,\dots,x_n)$;
    \item \emph{existential-universal} sentences\\ $\exists x_1,\dots,x_n\ \forall y_1,\dots,y_m\ \varphi(x_1,\dots,x_n,y_1,\dots,y_m)$.
\end{itemize}
The truth problem of each fragment defines a complexity class, denoted $\sETR$, $\sUTR$, and $\sEUTR$ respectively~\citep{DBLP:journals/mst/SchaeferS24}.
These classes sit between $\sNP$ and $\sPSPACE$~\citep{DBLP:conf/stoc/Canny88,basu2006algorithms}, and $\sUTR=\scoETR$ and $\sETR=\scoUTR$.

\paragraph{Square-root-sum.}
We also consider variants of the well-known square-root-sum ($\sSRS$) decision problem and their associated complexity classes.
For the variants below, assume finite lists of positive integers $a_1,\dots,a_m$ and $b_1,\dots,b_n$, and an integer $k$, all encoded in binary.
\begin{itemize}
    \item $\sSRS$.
    Decide whether $\sum_{i=1}^m \sqrt{a_i}\leq k$.
    $\scoSRS$ is the complement class.
    \item $\sSignedSRS$.
    Given an operator $\bowtie\in\{\leq,\geq\}$, decide whether $\sum_{i=1}^m\sqrt{a_i}\bowtie\sum_{j=1}^n\sqrt{b_j}$.
\end{itemize}
Both $\sSRS$ and $\scoSRS$ reduce to $\sSignedSRS$: $\sSRS$ uses $\vec b=(k^2)$; for $\scoSRS$, test equality in polynomial time ($\sum_i\sqrt{a_i}\in\mathbb Q$ iff every $a_i$ is a perfect square), then use $\geq$.

\paragraph{Polynomial hierarchy and common upper bound.}
The polynomial hierarchy $\mathsf{PH}$ is the union of the classes obtained by alternating polynomially bounded existential and universal Boolean quantifiers in front of a polynomial-time predicate.
In particular, $\sSigmaTwoP$ begins with an existential block followed by a universal block~\citep{DBLP:books/daglib/0072413,DBLP:books/daglib/0023084}.
Every level of $\mathsf{PH}$, every fixed level of the real hierarchy considered above, and $\sSignedSRS$ lie in $\sPSPACE$~\citep{DBLP:conf/stoc/Canny88,DBLP:journals/siamcomp/AllenderBKM09}.
Thus all classes used in this paper share a $\sPSPACE$ upper bound.

\subsection{Policy and Portfolio Comparison}\label{sec:policy-portfolio-comparison}

\paragraph{Policy comparison.}
We call the following problem \emph{policy comparison}: given two policies $\pi_1,\pi_2\in\polRand$ and a threshold $t$, decide whether
\[
\Delta_\cU(\pi_1,\pi_2)=\sup_{\vec u\in\cU}\bigl(V_{\vec u}^{\pi_1}(\sinit)-V_{\vec u}^{\pi_2}(\sinit)\bigr)\leq t.
\]
Policy comparison is tractable when $\pi_1$ and $\pi_2$ share no uncertain choice (\Cref{app:comparison-separation}): in this case the robust value of each policy can be evaluated individually and this is known to be feasible in polynomial time \cite{DBLP:journals/mor/Iyengar05,DBLP:journals/corr/abs-2604-26748}.
Otherwise, it is $\scoNP$-hard\footnote{This was stated by \citet{DBLP:conf/nips/GhavamzadehPC16} but proved only under a restrictive constraint imposed on the policies.
We give an unconditional proof of the fact in \Cref{app:shared-row-boolean}.}
already in acyclic $(s,a)$-rectangular RMDPs, and $\scoNP$-complete when no uncertain choice used by both policies lies on a cycle of either policy graph (\Cref{app:shared-row-boolean,app:comparison-membership}).
In this cycle-free regime the supremum defining $\Delta_\cU$ is attained at a rational vertex of the uncertainty polytope, so a witness realization can be written down exactly.

\begin{proof}[Sketch of $\scoNP$-hardness]
Reduce from $\mathrm{UNSAT}$.
For a 3-CNF formula $\varphi$, the RMDP keeps two local bits per literal occurrence, one certifying each value for the bits, and lets nature fix them independently through separate two-Dirac choices.
Both policies read the same copies but ask different questions of them.
The \emph{clause policy} $\pi_C$ accepts when every clause has a locally true literal.
The \emph{consistency policy} $\pi_K$ uniformly selects a variable and accepts when all its copies agree with one global value.
Thus $\pi_C$ contributes either $0$ or $1$ to $\Delta_\cU$, while $\pi_K$ contributes the fraction of variables whose audits pass.
Nature attains $2$ exactly when $\varphi$ is satisfiable; otherwise the difference is at most $2-1/n$.
\end{proof}

Even in the rectangular case, optimality at the vertices no longer holds once a shared choice lies on a cycle: the supremum can then be attained in the interior, at an irrational value, a sum of square roots, which is exactly the phenomenon $\scoSRS$-hardness captures (see \Cref{app:shared-cycle-sqrs} and the example below).
Membership in $\sUTR$ follows from a known adaptation of the Bellman equations for the value functions to the parametric setting \cite{DBLP:journals/jcss/JungesK0W21}.

\begin{example}\label{ex:comparison-interior}
Consider this RMDP with discount $\gamma=1/2$.
\begin{center}
\begin{tikzpicture}[
  >=Latex,
  state/.style={circle,draw,minimum size=8mm,inner sep=1pt},
  distribution/.style={circle,fill,inner sep=1.5pt},
  every node/.style={font=\small}
]
\node[state] (s) at (0,0) {$\sinit$};
\node[distribution] (d) at (1.3,0) {};
\node[state] (c) at (3.3,0.6) {$c$};
\node[state] (g) at (3.3,-0.6) {$g$};

\draw[->] ($(s.west)+(-6mm,0)$) -- (s.west);
\draw[-] (s) -- node[above] {$a$} (d);
\draw[->] (d) -- node[above,sloped] {$x$} (c);
\draw[->] (d) -- node[below,sloped] {$1-x$} (g);
\draw[->] (c) to[bend right=32] node[above] {$\pi_1(c)$} (s);
\draw[->] (c) edge[loop right] node[right] {$\pi_2(c)$} (c);
\draw[->] (g) edge[loop right] node[right] {$R=1$} (g);
\end{tikzpicture}
\end{center}
Fix the realization $\vec u\in\cU$ corresponding to $x$.
The value of $g$ is $2$.
Under $\pi_1$, reaching $c$ gives one zero-reward step back to $\sinit$.
Hence
\[
V_{\vec u}^{\pi_1}(\sinit)
 =\frac12\left(x\frac12V_{\vec u}^{\pi_1}(\sinit)
 +(1-x)2\right),
\]
hence \( V_{\vec u}^{\pi_1}(\sinit)=\frac{1-x}{1-x/4}.
\) Under $\pi_2$, the state $c$ has value zero, so \( V_{\vec u}^{\pi_2}(\sinit)=1-x.
\) Thus \( V_{\vec u}^{\pi_1}(\sinit)-V_{\vec u}^{\pi_2}(\sinit) =\frac{x(1-x)}{4-x}.
\) This difference is zero at $x=0$ and $x=1$, but positive between them.
Its maximum over $[0,1]$ is attained at \( x^*=4-2\sqrt3.
\) Substitution gives \( \Delta_{\cU}(\pi_1,\pi_2)=7-4\sqrt3.
\) The square root appears because $\pi_1$ can revisit the same uncertain choice, so its value is a rational function of $x$ whose maximum lies inside the interval.
\end{example}

\paragraph{Nonrectangularity makes things harder.}\label{par:polynomial-evaluation}
Without rectangularity, coupling lets one parameter recur at different choices. 
Then comparison is even $\sUTR$-complete (\Cref{app:comparison-real}).
To obtain hardness, we encode polynomial nonpositivity and build an RMDP whose value at each realization equals a given polynomial.
A monomial $p_{i_1}\cdots p_{i_d}$ becomes a chain of $d$ uncertain choices that survives with probability equal to the monomial.
A \emph{certain splitter} is a state whose sole action draws among these monomial branches with fixed rational probabilities.
On each branch, a terminal payoff cancels its certain splitter probability and discount, so the branch contributes its signed monomial.
Comparing this polynomial evaluator with a policy that goes to a zero-reward sink decides whether a given polynomial $f$ is nonpositive at every point in the uncertainty domain (\Cref{app:comparison-real}).

\paragraph{Portfolio comparison.}\label{par:portfolio-comparison}
Portfolio comparison remains $\sUTR$-complete even under $(s,a)$-rectangular uncertainty.
The pointwise maximum $\max_{\pi\in\Pi}V_{\vec u}^\pi$ over the explicit portfolio, rather than parameter reuse or recurrence, supplies the real quantifier alternation (\Cref{app:portfoliocomparison}).

\subsection{Regret Certification}\label{sec:robust:regret}

Certification asks whether a given policy or portfolio meets a target regret value.

\begin{problem}[Robust-regret certification]\label{prob:regret-certification}
Given an RMDP, a policy $\pi\in\polRand$, and a threshold $t \in \mathbb{Q}$, decide whether $\Rreg(\pi)\leq t$.
\end{problem}

\begin{problem}[Portfolio-regret certification]\label{prob:portfolio-regret}
Given an RMDP, a portfolio $\Pi\subseteq\polRand$, and a threshold $t \in \mathbb{Q}$, decide whether $\Rreg(\Pi)\leq t$.
\end{problem}
We show $\forall\mathbb{R}$-membership for both problems:
\begin{restatable}[Membership]{theorem}{portfolioregretmembership}\label{thm:portfolio-regret-membership}
Regret certification is in $\sUTR$.
\end{restatable}
The idea is to universally quantify the realization, an optimal policy at that realization, and one Bellman system per portfolio member, comparing the optimal value against the best of them.
See \Cref{app:portfolio-regret-membership}.

\paragraph{From comparison to regret.}
The obstacle to transferring comparison lower bounds is that it fixes both policies.
Regret measures a candidate against the best policy at the fixed realization.
For a policy $\pi$ and portfolio $\Pi$, write
\[
\Delta_\cU(\pi,\Pi)=\sup_{\vec u\in\cU}\left(V^\pi_{\vec u}(\sinit)-\max_{\pi'\in\Pi}V^{\pi'}_{\vec u}(\sinit)\right).
\]
For a source RMDP $N$ and reference policy $\pi_0$, $\widehat N=\Lift(N,\pi_0)$ makes $\pi_0$ optimal at every realization (\Cref{app:exact-lift}).
Each source choice becomes a tag state that selects an action and a selector state that carries its transition uncertainty.
For a source portfolio $\Pi$, let $\widehat\Pi$ be its image under $\Lift$.
Certification in $\widehat N$ then becomes comparison against $\pi_0$, up to a known constant $\Lambda$:
\(
\Rreg(\widehat\Pi)=\Lambda+\Delta_\cU(\pi_0,\Pi).
\)

\paragraph{Lower bounds.}
The following lower bounds combine comparison hardness with the lift construction.
\begin{restatable}{theorem}{fixedregrethardness}\label{thm:fixed-regret-hardness}
Given-policy robust regret is $\scoNP$-hard and $\scoSRS$-hard under $(s,a)$-rectangular uncertainty.
$\scoNP$-hardness already holds when every uncertain choice is two-Dirac and the only other stochastic rows are certain uniform splitters.
\end{restatable}
\begin{proof}[Proof sketch]
Apply $\Lift$ to $\Cmp(\varphi)$ and to the shared-cycle square-root-sum construction, respectively.
The transformation converts each comparison threshold by one known additive constant.
See \Cref{app:fixed-regret-hardness}.
\end{proof}

The pointwise maximum over members makes portfolio comparison $\sUTR$-hard even when rectangularity removes the corresponding alternation for one acyclic policy pair.
The transformation $\Lift$ transfers this hardness to certification:
\begin{restatable}{theorem}{portfolioregretatr}\label{thm:portfolio-regret-atr}
Portfolio-regret certification is $\sUTR$-hard, already for deterministic portfolios and acyclic $(s,a)$-rectangular RMDPs with two-successor uncertain choices.
\end{restatable}
\begin{proof}[Proof sketch]
Apply $\Lift$ to the portfolio-comparison construction in \Cref{par:portfolio-comparison}.
See \Cref{app:portfolio-atr}.
\end{proof}
\begin{restatable}{theorem}{fixedregretforallr}\label{thm:fixed-regret-forallr}
Given-policy robust regret is $\sUTR$-complete under general rational polytopic uncertainty, even for deterministic policies.
\end{restatable}

\begin{proof}[Proof sketch]
Membership is the singleton-portfolio case of \Cref{thm:portfolio-regret-membership}.
For hardness, reduce from the $\sUTR$-complete problem of deciding whether a polynomial is nonpositive throughout a bounded domain, introduced in \Cref{par:polynomial-evaluation}.
At each realization, we first force the given policy to take a zero-valued branch, while an optimal policy at that realization may choose either that branch or one evaluating a polynomial function.
Up to the fixed initial discount, the regret is therefore the positive part of the polynomial's value.
Consequently, the policy has regret at most zero exactly when the polynomial is nonpositive.
See \Cref{app:regret-certification-real}.
\end{proof}

\subsection{Minimal Robust Regret \& Bounded Synthesis}
\label{sec:minimal:robust:regret}
Certification fixes the candidate set whereas in the minimization and synthesis problems we need to choose it.

\begin{problem}[Minimal robust regret]\label{prob:min-regret}
Given an RMDP and a threshold $t \in \mathbb{Q}$, decide whether $\inf_{\pi\in\polRand}\Rreg(\pi)\leq t$.
\end{problem}

\begin{problem}[Bounded portfolio synthesis]\label{prob:min-portfolio-regret}
Given an RMDP, a budget $k$ encoded in unary, and a threshold $t \in \mathbb{Q}$, decide whether $\rho_k\leq t$.
\end{problem}

Portfolio regret is monotone under inclusion, and $\Rreg(\{\pi\})=\Rreg(\pi)$.
Thus minimal robust regret is the special case of bounded portfolio synthesis with $k=1$.
\begin{restatable}[Membership]{theorem}{minregretportfoliomembership}\label{thm:min-regret-portfolio-eatr-membership}
Minimal robust regret and bounded portfolio synthesis belong to $\sEUTR$.
\end{restatable}
For single-policy minimization, we existentially quantify the candidate policy table.
For bounded synthesis, we quantify the $k$ portfolio tables.
We then universally quantify the realization, an optimal policy at that realization, and the corresponding Bellman systems.
Since $k$ is encoded in unary, both formulas have polynomial size.
See \Cref{app:portfolio-regret-membership}.

\paragraph{Restricting the synthesized policy.}
Synthesis chooses its candidate, so a reduction must keep the synthesized policy in the role assigned by its encoding, either naming a valuation or following the consistency policy, rather than letting it escape to an unintended lower-regret policy.
The policy-restriction transformation $\Restrict$ makes disallowed actions so costly that minimizing over all policies comes within any rational $\varepsilon>0$ of minimizing over the allowed policies (\Cref{app:min-regret-boolean}).

\paragraph{Combinatorial hardness under rectangular uncertainty.}
\begin{restatable}{theorem}{minregretcombinatorial}\label{thm:min-regret-combinatorial}
Minimal robust regret is $\sNP$-hard and $\scoNP$-hard on $(s,a)$-rectangular RMDPs in which every uncertain choice is two-Dirac and the only other stochastic rows are certain uniform splitters.
\end{restatable}
\begin{proof}[Proof sketch]
For $\scoNP$-hardness, lift $\Cmp(\varphi)$ with the clause policy as reference and restrict the synthesized policy to the lifted consistency policy.
The resulting threshold separates satisfiable from unsatisfiable formulas.
For $\sNP$-hardness, the fixed reference finds a falsified clause and the synthesized policy names a valuation whose local copies nature can audit.
The most probable randomized choices still name a valuation with probability at least $2^{-n}$, which supplies the required gap.
The transformations $\Lift$ and $\Restrict$ preserve both gaps.
Both reductions retain independent two-Dirac choices.
See \Cref{app:min-regret-boolean}.
\end{proof}

\paragraph{Square-root-sum hardness under rectangular uncertainty.}
Beyond the combinatorial hardness, irrational regret values let us encode problems with signed sums of roots.

\begin{restatable}{theorem}{minregretsqrs}\label{thm:min-regret-sqrs}
Minimal robust regret is $\sSignedSRS$-hard under $(s,a)$-rectangular uncertainty.
\end{restatable}
\begin{proof}[Proof sketch]
The reduction uses local RMDPs in which the policy chooses a mixing probability $x$ between two actions.
For suitable rational $A,B,q>0$, its regret is
\[
\max\left\{
\frac{A(1-x)}{1-(1-q)x},
\frac{Bx}{q+(1-q)x}
\right\}.
\]
As $x$ increases the first term decreases while the second increases, so their maximum is minimized at an interior balance point, producing a square root.
Reward signs realize the two forms $D_b-\sqrt b$ and $\sqrt b-C_b$.
A certain uniform splitter adds the local minima, which realizes a signed sum.
See \Cref{app:min-regret-sqrs}.
\end{proof}

\paragraph{Hardness under general polytopic uncertainty.}
Under general rational polytopic uncertainty, policy comparison reduces to single-policy regret minimization.
\begin{restatable}{theorem}{minregretforallrhard}\label{thm:min-regret-forallr-hard}
Single-policy minimal robust regret is $\sUTR$-hard for arbitrary rational polytopic uncertainty.
\end{restatable}
\begin{proof}[Proof sketch]
Reduce from the $\sUTR$-complete general-polytope comparison problem described in \Cref{par:polynomial-evaluation}.
The construction leaves the synthesized policy one decision: a mixing probability $x$ between forced copies of the two source policies.
Two absorbing states of values $\pm Z$ anchor the branches so that only this mixing probability matters.
They make its regret $\max\{(1-x)D,xE\}$, where $E$ is a fixed positive rational and $D$ is an affine, strictly increasing function of the source comparison value $\Delta$.
The minimizing policy balances the two terms, giving $DE/(D+E)$, again strictly increasing in $\Delta$, so the comparison threshold transfers by a rational transformation.
See \Cref{app:regret-minimization-real}.
\end{proof}

Together with the upper bound above, this places single-policy minimal robust regret between $\sUTR$-hardness and $\sEUTR$-membership under general rational polytopic uncertainty.
The current bounds do not establish completeness.

Restricting the search to memoryless deterministic policies on acyclic $(s,a)$-rectangular RMDPs makes minimal robust regret $\sSigmaTwoP$-complete (\Cref{app:min-regret-deterministic}).
This is a different problem from \Cref{prob:min-regret}, which minimizes over $\polRand$, and neither classification implies the other.
In particular, the $\sSigmaTwoP$-hardness construction leaves the synthesized policy a free choice at every existential variable state, so a randomized policy may mix there and undercut every deterministic candidate.
The $\sNP$-hardness of \Cref{thm:min-regret-combinatorial} thus rests on a separate construction, whose $2^{-n}$ gap randomization cannot close.

\paragraph{Bounded portfolio synthesis.}
For bounded portfolios, the general upper bound is tight.

\begin{restatable}{theorem}{minportfolioeatr}\label{thm:min-portfolio-eatr}
Bounded portfolio synthesis is $\sEUTR$-complete under general rational polytopic uncertainty, with $k$
encoded in unary.
Hardness holds already for regret threshold two, the fixed discount $\gamma=\tfrac12$, RMDPs acyclic
apart from absorbing final states, and uncertain choices with two successors.
\end{restatable}
\begin{proof}[Proof sketch]
Membership is the encoding above.
For hardness, normalize $\exists x\,\forall y:F(x,y)\ge0$ into tests $g_i(x,\eta)$ affine in $x$, and
use one portfolio member per test.
A case distribution lets nature audit roles, witness agreement, and tests through fixed-reward
polynomial evaluators.
The regret threshold is met exactly when some $x$ passes every test for every $\eta$.
See \Cref{app:portfolio-eatr}.
\end{proof}

The equality ties couple choices at different state-action pairs, so the reduction does not settle the rectangular case.
Under rectangular uncertainty, bounded synthesis is hard already for the fixed budget $k=1$.

\section{Portfolio Construction and Evaluation}
\label{sec:implementation}

We show that even a simple offline pipeline yields portfolios with substantially lower empirical regret,
despite the intractability of portfolio construction. To this end, we implemented the three-phase prototype summarized in \Cref{alg:experimental-pipeline}.
Two research questions guide the experiments: whether portfolios reduce empirical robust regret as the budget grows \textbf{(RQ1)}, and whether the best member can be identified online in a fixed but unknown environment \textbf{(RQ2)}.
All details on the benchmarks, experiment protocol, and results can be found in \Cref{app:experiments}.

\begin{algorithm}[t]
\caption{Experimental pipeline}
\label{alg:experimental-pipeline}
\KwIn{RMDP $M$, parameter box $D$, portfolio size $K$, and evaluation seeds}
\textbf{Construct:} split $D$ into cells and compute an optimal midpoint policy for each cell\;
Robustly evaluate every candidate on every cell, giving one regret profile per policy\;
Cluster the profiles and select the policy nearest each of the $K$ centers\;
\textbf{Evaluate:} for 3 seeded samples of $D$, estimate min regret in portfolio, maximized over samples\;
\textbf{Deploy:} treat members as bandit arms and run UCB in each fixed sampled environment\;
\end{algorithm}

\paragraph{Benchmarks.}
We introduce two new benchmarks: \emph{datacenter climate control} and \emph{UAV control}.
The datacenter benchmark contains two uncertain parameters, cooling effectiveness and ambient pressure, both ranging over a box domain.
The rewards combine a per-action energy cost with penalties activating near the temperature, humidity, and queue ceilings, giving $\Vmax=2080$.
The UAV benchmark is a planning problem through a three-dimensional grid under uncertain wind intensity $p$ and actuator-drop probability $q$, again over a box domain.
Reaching the goal is the only rewarded event; thus, the value is the discounted landing probability, and $\Vmax=100$.
We consider three different instances of this benchmark, \texttt{uav-small}, \texttt{uav-medium}, and \texttt{uav-large}. 
Every grid position offers seven actions, so even \texttt{uav-small} admits $7^{95}$
memoryless deterministic policies, ruling out exhaustive search.

\paragraph{Portfolio construction.}
The benchmarks are nonrectangular RMDPs whose transition probabilities are affine in a parameter vector ranging over a box $D$. Hence, every valuation $\theta\in D$ induces a realization $\vec u\in\cU$.
Discretizing each parameter into $10$ bins gives a set $\mathcal{C}$ of $100$ cells $c$ for our two-parameter benchmarks, and the optimal policy at each cell midpoint $\operatorname{mid}(c)$ becomes a candidate $\pi \in \Pi(\mathcal{C})$ in the candidate set.
Robust policy evaluation then assigns each candidate $\pi$ a midpoint-based approximate regret vector $L_\pi$ with one entry per cell:
\(
L_{\pi,c}=V^*_{\operatorname{mid}(c)}(\sinit)-\inf_{\theta\in c}V^\pi_\theta(\sinit).
\)
We then use K-means over these vectors
to select the policy nearest each center, resulting in a set $\Pi_K \subseteq \Pi(\mathcal{C})$ of policies.
Clustering is a computationally cheap approximation for obtaining portfolios of deterministic policies.

\paragraph{Portfolio evaluation.}
A uniform parameter sample $\widehat{D} \subset D$ with $\lvert \widehat{D} \rvert = 1000$
is drawn to estimate portfolio regret by
\[
\widehat\Rreg_{\widehat{D}}(\Pi)=\max_{\theta\in \widehat{D}}\Big(V^*_\theta(\sinit)-\max_{\pi\in\Pi}V^\pi_\theta(\sinit)\Big),
\]
which measures coverage without charging for online identification and under-approximates the true regret, since it maximizes over finitely many samples (cf. \Cref{eqn:rreg}).
For deployment, UCB~\citep{pred-learn-games} treats portfolio members as bandit arms whose returns come from fixed-length trajectories in a fixed but unknown environment. We use it with error $\varepsilon=0.001$ (as a fraction of the return bound), confidence $\delta=0.1$, and empirical-Bernstein bounds~
\citep{10.1145/1390156.1390241} 
to speed up convergence.
For the UCB-based deployment, we uniformly draw $30$ valuations and run UCB once per valuation with a trajectory of length $H=100$ sampled on each iteration.
We plot, per iteration, the fraction of UCB runs where the final recommended arm is
the best portfolio member and the difference between the best and recommended policy, normalized between $0$ (best) and $1$ (worst).

\paragraph{Setup.}
The Python prototype uses Stormvogel~\citep{VolkEtAl26StormTutorial}, robust value iteration~\citep{DBLP:journals/mor/Iyengar05}, and scikit-learn K-means.
Experiments were run on a 2022 MacBook Pro M1 under macOS Tahoe 26.5.2, using 6 threads for robust policy evaluation, regret approximation, and UCB evaluation.
All experiments are done for three different seeds. For the plots in \Cref{fig:ucb-datacenter}, we merged all $90$ samples across the three seeds.

\begin{table}[t]
\centering
\small
\setlength{\tabcolsep}{2.7pt}
\begin{tabular}{@{}ccccc@{}}
\toprule
$K$ & \shortstack{UAV-S\\(98)} & \shortstack{UAV-M\\(358)} & \shortstack{UAV-L\\(1813)} & \shortstack{datacenter\\(330)} \\
\midrule
1  & 0.053 & 0.091 & 0.126 & 14.28 \\
2  & 0.032 & 0.019 & 0.017 & 3.86 \\
3  & 0.018 & 0.010 & 0.016 & 2.82 \\
5  & 0.002 & 0.003 & 0.015 & 2.79 \\
7  & 0.002 & 0.003 & 0.008 & 1.79 \\
10 & 0.002 & 0.003 & 0.003 & 0.80 \\
\midrule
Construction time & 43.8s & 384.1s & 2357.8s & 2732.7s \\
\bottomrule
\end{tabular}
\caption{Empirical robust regret, averaged over three seeds.
Parentheses give the number of states. Bottom row indicates wall-clock time.}
\label{tab:portfolio-results-summary}
\end{table}

\paragraph{RQ1: Portfolios reduce robust regret.}
\Cref{tab:portfolio-results-summary} shows empirical robust regret
decreasing with $K$ on every benchmark, with the largest single drop at $K=2$, and further gains as $K$ increases.
The $K=1$ row is a singleton rather than an adaptive portfolio.
Seed-level values and K-means inertia appear in \Cref{tab:merged-results} in \Cref{app:experiments}. For reference, the mini-max regret $\min_{\pi \in \Pi(\mathcal{C})} \max_{c \in \mathcal{C}} L_{\pi,c}$ is $0.039$, $0.118$, $0.146$, and $41.502$ by column.
Policy portfolios beat it for all $K$ on the three largest benchmarks.

\begin{figure}[t]
    \centering
    \begin{subfigure}[t]{0.48\columnwidth}
        \centering
        \includegraphics[width=\linewidth]{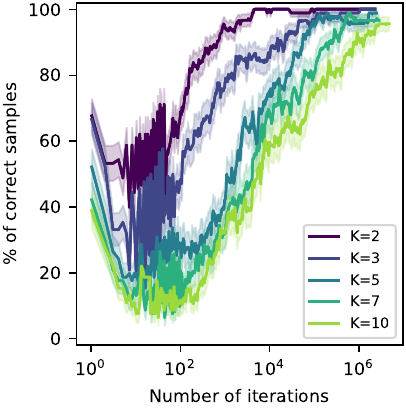}
        \caption{Runs with the correct arm recommended.}
        \label{fig:arm-ucb-recommended}
    \end{subfigure}
    \hfill
    \begin{subfigure}[t]{0.48\columnwidth}
        \centering
        \includegraphics[width=\linewidth]{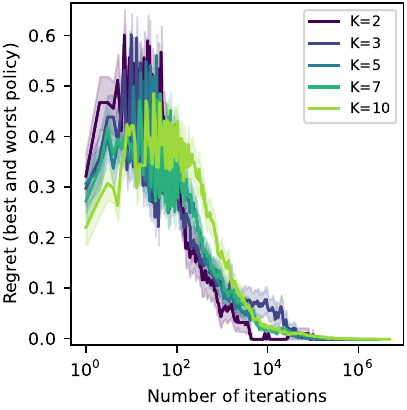}
        \caption{Regret of recommended arm.}
        \label{fig:regret-ucb-recommended}
    \end{subfigure}
    \caption{Reported UCB deployment results on the datacenter across all tested values of $K$.}
    \label{fig:ucb-datacenter}
\end{figure}

\paragraph{RQ2: Portfolio members can be identified online.}
At each iteration of the UCB algorithm, we measure which arm was pulled (maximal UCB) and which arm was recommended (maximal LCB). 
\Cref{fig:arm-ucb-recommended} gives the fraction of UCB runs recommending the best portfolio policy at each iteration.
Identification slows as the portfolio grows: after $10^4$ iterations, about half the runs recover the best member at $K=10$.
Since this ignores how well the remaining members perform, \Cref{fig:regret-ucb-recommended} instead reports
the regret of the recommended policy relative to the best portfolio policy, and there the value difference after $10^4$ iterations is minimal.

\section{Conclusion}
Adaptive policy portfolios sit between committing to a single robust policy and planning over a full belief state, offering finite and certifiable adaptation.
Certification is $\sUTR$-complete and bounded synthesis $\sEUTR$-complete under rational polytopic uncertainty; certification stays hard under rectangularity, since the pointwise maximum over members reintroduces the quantifier alternation that rectangularity eliminates.
Our prototype empirically shows portfolios whose regret falls sharply with the first few members, at an identification cost that grows with $K$.
Bounded synthesis under rectangular uncertainty, and the gap between $\sUTR$-hardness and $\sEUTR$-membership for single-policy minimal robust regret, remain open.

\clearpage
\bibliography{references}

\onecolumn
\appendix
\let\addcontentsline\appendixaddcontentsline
\renewcommand{\contentsname}{Appendix Contents}
\setcounter{tocdepth}{1}
\tableofcontents
\clearpage
\section{Conventions and primitives}
\label[appendix]{app:sqrs-preliminaries}
This appendix collects the conventions and building blocks that the later reductions share.
Three groups follow.
First, two numerical conventions: a discount $\gamma_0$ chosen so that the comparison-to-regret lift
composes exactly, and one inequality about it that the normalized square-root gadget needs.
Second, reduction primitives: how a state is given a prescribed value, how uncertainty enters a
single choice, the two-Dirac gadgets by which nature encodes a Boolean assignment and a verifier
reads it back, and ruinous-sink completion, which makes every choice we did not describe so costly
that no policy gains by taking it.
Third, encodings: a degree-four normal form that turns a polynomial constraint into small residuals,
and the Bellman formulas behind the membership proofs.
A reader may skip ahead and return here when a later construction cites one of these by name.

We first fix two conventions used throughout the appendix.
Set
\[
\beta=\frac{19}{20},
\qquad
\gamma_0=\beta^2=\frac{361}{400}.
\]
The identity $\beta^2=\gamma_0$ is needed by the exact comparison-to-regret lift, while
\[
\left(\frac{1+\gamma_0+\gamma_0^2}{1+\gamma_0}\right)^2>2
\]
is the numerical inequality used by the normalized square-root gadget.
All appendix constructions use $\gamma_0$ unless stated otherwise.
The bounded-synthesis reduction uses $\gamma=\tfrac12$ because it does not invoke the lift and hence
does not require $\beta^2=\gamma_0$.

\subsection{Reduction primitives}
\label{app:primitives}
A state is an \emph{absorbing final} when every described action has a Dirac self-loop.
A \emph{terminal with payoff $c$} is an absorbing final whose described action has reward $(1-\gamma)c$, so its value is exactly $c$.
Every action not separately described at such a terminal is assigned the same reward and Dirac self-loop.
Reaching such a terminal after $d$ transitions with probability $w$ contributes $w\gamma^d c$ to the value at the initial state.

For a choice $(s,a)$, call the distribution $\vec u(s,a,\cdot)$ its \emph{row}; the choice is uncertain
exactly when its row is not fixed.
A \emph{selector} is a state whose described action has a row carrying one uncertainty coordinate.
A \emph{splitter} is a state whose described action has a fixed row, independent of both the
realization and the policy; such a row is \emph{certain}.
A splitter is \emph{uniform} when its row is uniform over its successors.
We reserve $\bot$ for an absorbing zero-reward sink, that is, the terminal with payoff zero, and
$\bot_{\rm r}$ for a ruinous sink.

\begin{definition}[Local-bit verifier primitives]\label{def:local-bit-verifier-primitives}
For an occurrence $o$ and Boolean value $c$, a local-bit selector $q_{o,c}$ has outcomes
$q_{o,c}^0,q_{o,c}^1$ and row
\begin{equation}\label{eq:boolean-selector-choice}
\{(1-p)\delta_{q_{o,c}^0}+p\delta_{q_{o,c}^1}:p\in[0,1]\}.
\end{equation}
Its vertices encode the bit $b_{o,c}=p\in\{0,1\}$.
A pair test for an occurrence with designated satisfying value $\operatorname{val}(o)$ reads
$q_{o,\operatorname{val}(o)}$ and then $q_{o,1-\operatorname{val}(o)}$.
Outcomes $(1,0)$ certify local truth and $(0,1)$ certify local falsity.
An audit committed to value $c$ visits $q_{o,c}$ for the relevant occurrences in a fixed order,
advancing on outcome one and rejecting on outcome zero.
\end{definition}

\begin{lemma}[Acceptance-difference verifier]\label{lem:acceptance-difference-verifier}
Suppose reference and candidate verifier paths are padded to depth $H$, all nonaccepting payoffs are
zero, and their accepting terminals have payoffs $\gamma^{-H}$ and $-\gamma^{-H}$.
Then their value difference is the sum of their acceptance probabilities.
\end{lemma}
\begin{proof}
Every accepting run contributes $\gamma^H\gamma^{-H}=1$ to the corresponding signed value, while
every rejecting run contributes zero.
\end{proof}

A construction's \emph{ruinous-sink completion} adds a fresh absorbing state $\bot_{\rm r}$ and a rational $Z>0$.
Every action at $\bot_{\rm r}$ has reward $-(1-\gamma)Z$ and the singleton row $\{\delta_{\bot_{\rm r}}\}$, so $V(\bot_{\rm r})=-Z$.
Every otherwise undescribed choice at a nonterminal state has reward zero and the singleton row $\{\delta_{\bot_{\rm r}}\}$.
Write $V^{\rm bd}=\max_{s,a}|R(s,a)|/(1-\gamma)$ for the standard bound on values, the maximum ranging
over the choices described before completion.
Unless stated otherwise, $Z$ is fixed after those rewards, has polynomial encoding length, and
satisfies $\gamma Z>V^{\rm bd}+1$.
We invoke this convention by saying that the remaining choices use ruinous-sink completion.
The added state is an absorbing final and all added rows are singletons, so the completion preserves acyclicity, rectangularity, and every two-successor or two-Dirac restriction on uncertain choices.
Call a choice \emph{ruinous} when its row is $\{\delta_{\bot_{\rm r}}\}$ and its reward is zero,
whether the completion supplied it or a later transformation described it explicitly.
Call a policy \emph{compliant} when it plays a nonruinous action at every state, and write $\bar\pi$
for the policy obtained from $\pi$ by conditioning on the nonruinous actions at every state, using an
arbitrary nonruinous action where that conditioning has zero mass.

\begin{lemma}[Ruinous dominance]\label{lem:ruin-dominance}
Suppose $|V_u^{\bar\rho}(s)|\le M$ for every compliant policy $\bar\rho$, every state $s$ other than
$\bot_{\rm r}$, and every realization $u$, and suppose $\gamma Z>M+1$.
Then $V_u^\pi(s)\le V_u^{\bar\pi}(s)$ for every stationary randomized policy $\pi$, every state, and
every realization.
\end{lemma}
\begin{proof}
Both policies have value $-Z$ at $\bot_{\rm r}$, so fix any other state $s$.
The action value of a nonruinous choice $(s,a)$ under $V_u^{\bar\pi}$ is the value at $s$ of the
compliant policy that plays $a$ once and follows $\bar\pi$ afterwards, hence at least $-M$.
A ruinous choice has action value $\gamma V(\bot_{\rm r})=-\gamma Z<-M-1$, strictly smaller.
Writing $T_\pi$ for the Bellman operator of $\pi$ at $u$, evaluating both policies at
$V_u^{\bar\pi}$ therefore gives $T_\pi V_u^{\bar\pi}\le T_{\bar\pi}V_u^{\bar\pi}=V_u^{\bar\pi}$,
since $\bar\pi$ redistributes $\pi$'s ruinous mass onto strictly larger action values.
Monotonicity and contraction of $T_\pi$ give $V_u^\pi\le V_u^{\bar\pi}$.
\end{proof}

Under the default rule, $M=V^{\rm bd}$ by the standard bound on discounted values, so the hypothesis
holds.
A construction that fixes its own constant instead, such as $Z_{\rm r}$ in
\Cref{def:two-action-anchoring}, exhibits a bound on its own compliant values and checks the
inequality against that.
\Cref{def:eps-restricted-rmdp} is the exception: it chooses $Z_\varepsilon$ to meet a quantitative
requirement of its own and establishes dominance directly in the proof of
\Cref{lem:policy-restriction}.
Value identities below are stated for compliant policies and hold as upper bounds in general by
\Cref{lem:ruin-dominance}.

We use function-like notation for the four transformations that recur below.
For a 3-CNF formula $\varphi$ on $n$ variables, $\Cmp(\varphi)$ denotes the Boolean comparison RMDP of \Cref{def:boolean-comparison-rmdp}, together with its clause and consistency policies and threshold $2-\tfrac1{2n}$.
For a source RMDP $N$ and reference policy $\pi_0$, $\Lift(N,\pi_0)$ denotes the exact regret lift of \Cref{def:lifted-rmdp}, with its allowed-action family specified by context.
$\Restrict(N,A_{\mathrm{allow}},\varepsilon)$ denotes the policy-restriction transformation of \Cref{def:eps-restricted-rmdp}.
Finally, $\operatorname{Ruin}(N,Z)$ denotes ruinous-sink completion of $N$ with constant $Z$.

All constructions below run in polynomial time, and all the rational constants above have polynomial encoding length.
Ruinous-sink completion adds one state and $O(|S||A|)$ singleton rows, so it preserves polynomial size.
We mention size bounds only where they are not immediate from the construction.

\subsection{Degree-Four Residual Normal Form}
The bounded-domain equivalence of
\citet[Proposition~2.13]{DBLP:journals/mst/SchaeferS24} and the degree-four normal form quoted in their
proof of Lemma~2.8 allow bounded real sentences to use an explicitly represented rational polynomial
of degree at most four over $[0,1]^N$.
Write such a polynomial as
\[
F(w)=\sum_{\nu=1}^{N_{\rm mon}}c_\nu\prod_{j=1}^{d_\nu}w_{\nu,j},
\qquad d_\nu\le4,
\qquad
B=\max\left\{1,\sum_{\nu=1}^{N_{\rm mon}}|c_\nu|\right\}.
\]
Thus $|F|\le B$ on the box, and $B$ has polynomial encoding length.

\begin{definition}[Degree-four residual normal form]\label{def:policy-affine-normal-form}
For every occurrence $w_{\nu,j}$, introduce a copy $z_{\nu,j}\in[0,1]$ and residual
$h^{\rm cp}_{\nu,j}=z_{\nu,j}-w_{\nu,j}$.
Compute each monomial with auxiliary coordinates in $[0,1]$:
\begin{itemize}
\item for $d_\nu=0$, set $p_\nu=1$, and for $d_\nu=1$, set $p_\nu=z_{\nu,1}$;
\item for $d_\nu=2$, use $h_\nu=p_\nu-z_{\nu,1}z_{\nu,2}$;
\item for $d_\nu=3$, use
$t_\nu-z_{\nu,1}z_{\nu,2}$ and $p_\nu-t_\nu z_{\nu,3}$;
\item for $d_\nu=4$, use
$t_{\nu,12}-z_{\nu,1}z_{\nu,2}$,
$t_{\nu,34}-z_{\nu,3}z_{\nu,4}$, and
$p_\nu-t_{\nu,12}t_{\nu,34}$.
\end{itemize}
These copy and product residuals lie in $[-1,1]$, are affine or affine plus one product of distinct
variables, and have a common zero exactly at correct copies and products.
The decoded polynomial is $\widehat F=\sum_\nu c_\nu p_\nu$.
When an explicit output is needed, introduce $\bar o\in[0,1]$, put
$o=B(2\bar o-1)$, and append
$h^{\rm out}=(o-\widehat F)/(2B)$; this residual also lies in $[-1,1]$.
\end{definition}

\begin{lemma}[Degree-four residual error]\label{lem:policy-affine-error}
For any assignment to the inputs and auxiliary coordinates, let $\delta$ be the largest absolute
copy or product residual.
Then
\[
|p_\nu-\textstyle\prod_jw_{\nu,j}|\le7\delta
\quad\text{and}\quad
|\widehat F-F|\le7B\delta.
\]
If the output residual is present and $\delta'$ also includes $|h^{\rm out}|$, then
$|o-F|\le9B\delta'$.
\end{lemma}
\begin{proof}
For $a,b,a',b'\in[0,1]$,
$|ab-a'b'|\le |a-a'|+|b-b'|$, because
$ab-a'b'=b(a-a')+a'(b-b')$.
The errors in degrees zero and one are zero and at most $\delta$.
For degree two the error is at most $\delta+\delta+\delta=3\delta$.
For degree three it is at most $\delta+3\delta+\delta=5\delta$, and for degree four it is at most
$\delta+3\delta+3\delta=7\delta$.
Therefore
$|\widehat F-F|\le7\delta\sum_\nu|c_\nu|\le7B\delta$.
With the output residual,
$|o-\widehat F|=2B|h^{\rm out}|\le2B\delta'$, and the triangle inequality gives the last claim.
\end{proof}

\subsection{Encoding primitives}
For an uncertainty vector $u$, policy table $\sigma$, values $v_s$, and action values $q_{s,a}$, define
\[
\operatorname{Bell}(\sigma,u,v,q)=
\bigwedge_s\left(v_s=\sum_a\sigma_{s,a}q_{s,a}\right)
\wedge
\bigwedge_{s,a}\left(q_{s,a}=R(s,a)+\gamma\sum_{s'}u(s,a,s')v_{s'}\right).
\]
Let $\Phi_{\cU}(u)$ be the rational linear description of the uncertainty polytope, including nonnegativity and normalization, and let
\[
\operatorname{Policy}(\sigma)=
\bigwedge_s\left(\sum_a\sigma_{s,a}=1\right)
\wedge\bigwedge_{s,a}\sigma_{s,a}\ge0.
\]

\begin{lemma}[Universal Bellman encoding]\label{lem:universal-bellman-encoding}
For an explicit family $\Pi=\{\pi_1,\ldots,\pi_r\}$ and a universally quantified policy $\tau$, the assertion
\[
V_u^\tau(\sinit)-\max_{i\le r}V_u^{\pi_i}(\sinit)\le t
\quad\text{for every }u\in\cU
\]
has a polynomial-size universal formula over the reals.
The policy $\tau$ may instead be fixed, and prefixing an existential block for $r$ valid policy tables gives an existential-universal formula when $r$ is unary-bounded.
\end{lemma}
\begin{proof}
Use the universal implication
\[
\begin{split}
\forall\,\tau,u,v^\tau,q^\tau,(v^i,q^i)_{i=1}^r:\quad
&\Bigl(\operatorname{Policy}(\tau)\wedge\Phi_{\cU}(u)
\wedge\operatorname{Bell}(\tau,u,v^\tau,q^\tau)\wedge\bigwedge_{i=1}^r
\operatorname{Bell}(\pi_i,u,v^i,q^i)\Bigr)%
\Longrightarrow
\bigvee_{i=1}^r\left(v^\tau_{\sinit}-v^i_{\sinit}\le t\right).
\end{split}
\]
Discounting makes each valid Bellman system unique.
Invalid policy, realization, or Bellman assignments falsify the antecedent.
Fixing $\tau$ removes its policy variables, while existentially quantified family members require their simplex constraints outside the universal implication.
Compactness of the policy simplexes and uncertainty polytope ensures that the relevant extrema are attained.
\end{proof}

\section{On Policy and Portfolio Comparison}
\label[appendix]{app:policy-portfolio-comparison}
We collect the comparison problems and constructions used by the regret lower bounds.

\subsection{Comparison Preliminaries}
\label[appendix]{app:policy-comparison}

\begin{problem}[Robust policy comparison]\label{prob:policy-comparison}
Given an RMDP $M$, policies $\pi_1,\pi_2\in\polRand$, and a rational threshold $t$, decide whether
\[
\Delta_{\cU}(\pi_1,\pi_2)
=\sup_{\vec u\in\cU}
\bigl(V_{\vec u}^{\pi_1}(\sinit)-V_{\vec u}^{\pi_2}(\sinit)\bigr)
\le t.
\]
\end{problem}

A choice $(s,a)$ is \emph{used} by $\pi$ if $s$ is reachable with positive probability under some realization and $\pi(s,a)>0$.
It is \emph{shared} by $\pi_1$ and $\pi_2$ if both use it.
The set of shared choices is denoted $\mathrm{Sh}(\pi_1,\pi_2)$.
Call $s$ \emph{absorbing under $\pi$} when every choice used by $\pi$ at $s$ has row $\{\delta_s\}$.
Let $G_\pi$ contain $s\to s'$ when a choice used by $\pi$ at $s$ can reach $s'$ under some realization, except that the self-loop at a state absorbing under $\pi$ is omitted.
This exception matches the convention for RMDP acyclicity, which likewise permits self-loops only at absorbing final states.
Self-loops at other states remain one-edge cycles.
A used choice $(s,a)$ is \emph{cycle-free under $\pi$} if $s$ lies on no directed cycle of $G_\pi$.
We write $\mathrm{CF}(\pi)$ for these choices.
\begin{definition}[Cycle-free on shared choices]\label{def:cyclefree-shared}
The pair $(\pi_1,\pi_2)$ is \emph{cycle-free on shared choices} if $\mathrm{Sh}(\pi_1,\pi_2)\subseteq\mathrm{CF}(\pi_1)\cap\mathrm{CF}(\pi_2)$.
\end{definition}
For $\pi_1,\pi_2$ and $\vec u\in\cU$, write $D_{\pi_1,\pi_2}(\vec u):=V_{\vec u}^{\pi_1}(\sinit)-V_{\vec u}^{\pi_2}(\sinit)$, so that $\Delta_\cU(\pi_1,\pi_2)=\sup_{\vec u\in\cU}D_{\pi_1,\pi_2}(\vec u)$.
For a polytope $P$, write $\operatorname{vert}(P)$ for its set of vertices.
For $(s,a)$-rectangular $\cU$, write
$\operatorname{Vert}(\cU):=\prod_{s,a}\operatorname{vert}(\cU_{s,a})$ for the set of vertex tuples.

\paragraph{Separated choices.}
\label[appendix]{app:comparison-separation}
If $\pi_1$ and $\pi_2$ share no uncertain choice, rectangularity lets nature optimize their used rows independently, so
\[
\Delta_{\cU}(\pi_1,\pi_2)
=\sup_{u_1}V^{\pi_1}_{u_1}(\sinit)
-\inf_{u_2}V^{\pi_2}_{u_2}(\sinit).
\]
This is a direct consequence of the standard optimistic and robust fixed-policy linear programs and is computable in polynomial time \cite{DBLP:journals/mor/Iyengar05,DBLP:journals/corr/abs-2604-26748}.
\subsection{Combinatorial Hardness of Robust Policy Comparison}
\label[appendix]{app:shared-row-boolean}
The reduction makes one local copy of every variable in every clause.
One policy scans the clauses, and the other audits a uniformly selected variable against one valuation.

\begin{theorem}[Shared-choice Boolean hardness]\label{thm:comparison-conp}
Robust policy comparison is $\scoNP$-hard even for deterministic policies in acyclic $(s,a)$-rectangular RMDPs in which every uncertain choice is a two-Dirac segment and the only other stochastic row is a certain uniform splitter.
\end{theorem}

Fix a 3-CNF formula $\varphi=\bigwedge_{i=1}^m C_i$ over variables $x_1,\ldots,x_n$.
We may remove tautological clauses, repeated literals, and variables with no occurrence; the constant case is decided directly, so assume $n\ge1$.
For a literal occurrence $o$, let $\operatorname{var}(o)$ be its variable, let $\operatorname{val}(o)\in\{0,1\}$ be the value that makes its literal true, and order the occurrences $O_x$ of each variable $x$ by clause and then by position within the clause.

\begin{definition*}[Clause-local encoding]
For every occurrence $o$, introduce two local bits $b_{o,1}$ and $b_{o,0}$, indexed by the values they certify.
The equation $b_{o,c}=1$ says that occurrence $o$ is locally consistent with value $c$.
Occurrence $o$ is \emph{locally true} when $b_{o,\operatorname{val}(o)}=1$ and $b_{o,1-\operatorname{val}(o)}=0$.
Write $\Phi_C$ for the clause condition
\begin{equation}\label{eq:clause-local-satisfaction}
\Phi_C:=\bigwedge_{i=1}^m
\left(
 \bigvee_{o\in C_i}
 \bigl(
 b_{o,\operatorname{val}(o)}
 \wedge
 \lnot b_{o,1-\operatorname{val}(o)}
 \bigr)
\right).
\end{equation}
The local bits of different clauses are disjoint.

For every original variable $x$, introduce a global bit $X_x$.
For each variable write
\begin{equation}\label{eq:clause-local-consistency}
\Phi_K^x:=\bigwedge_{o\in O_x}b_{o,X_x},
\qquad
\Phi_K:=\bigwedge_x\Phi_K^x.
\end{equation}
\end{definition*}

\begin{lemma*}[Clause-local equivalence]
Formula $\varphi$ is satisfiable iff its global and local bits have an assignment satisfying $\Phi_C\wedge\Phi_K$.
\end{lemma*}
\begin{proof}
If a valuation $v$ satisfies $\varphi$, set $X_x=v(x)$ and set $b_{o,c}=1$ exactly when $c=v(\operatorname{var}(o))$.
Every $\Phi_K^x$ holds.
Every clause has an occurrence $o$ with $\operatorname{val}(o)=v(\operatorname{var}(o))$, so that occurrence has the locally true pattern and $\Phi_C$ holds.

Conversely, suppose $\Phi_C\wedge\Phi_K$ holds and define $v(x):=X_x$.
Each clause has a locally true occurrence $o$ with $b_{o,1-\operatorname{val}(o)}=0$.
Since $\Phi_K^{\operatorname{var}(o)}$ forces $b_{o,X_{\operatorname{var}(o)}}=1$, we have $X_{\operatorname{var}(o)}=\operatorname{val}(o)$.
Thus the literal at $o$ is true under $v$, and every clause is satisfied.
\end{proof}
$\Phi_K$ does not force an occurrence's two bits to be complementary.
The pair test is part of $\Phi_C$.

\begin{example*}[A clause-local encoding]
Consider
\[
\varphi=(x_1\vee\lnot x_2\vee x_3)
       \wedge(\lnot x_1\vee x_2\vee x_3).
\]
Take the valuation: \[ x_1=x_2=1, x_3=0 ,\] so \[ X_{x_1}=X_{x_2}=1\text{ and }X_{x_3}=0.\]
Following the encoding in the proof above, set $b_{o,c}=1$ exactly when $c=X_{\operatorname{var}(o)}$, for every occurrence $o$.
\begin{center}
\small
\begin{tabular}{@{}ccccccc@{}}
\toprule
Clause & literal & $\operatorname{var}(o)$ & $\operatorname{val}(o)$ & $X_{\operatorname{var}(o)}$ & $(b_{o,1},b_{o,0})$ & locally true? \\
\midrule
$C_1$ & $x_1$       & $x_1$ & $1$ & $1$ & $(1,0)$ & yes \\
      & $\lnot x_2$ & $x_2$ & $0$ & $1$ & $(1,0)$ & no  \\
      & $x_3$       & $x_3$ & $1$ & $0$ & $(0,1)$ & no  \\
\addlinespace
$C_2$ & $\lnot x_1$ & $x_1$ & $0$ & $1$ & $(1,0)$ & no  \\
      & $x_2$       & $x_2$ & $1$ & $1$ & $(1,0)$ & yes \\
      & $x_3$       & $x_3$ & $1$ & $0$ & $(0,1)$ & no  \\
\bottomrule
\end{tabular}
\end{center}
A literal is locally true exactly when $b_{o,\operatorname{val}(o)}=1$, i.e., when $\operatorname{val}(o)=X_{\operatorname{var}(o)}$: this holds for $x_1$ in $C_1$ and $x_2$ in $C_2$, matching the fact that these are the literals true under the valuation.
Using the labels $o_1,\ldots,o_6$, the clause and consistency conditions instantiate to
\[
\Phi_C=
\bigl[(b_{o_1,1}\wedge\lnot b_{o_1,0})\vee(b_{o_2,0}\wedge\lnot b_{o_2,1})\vee(b_{o_3,1}\wedge\lnot b_{o_3,0})\bigr]
\wedge
\bigl[(b_{o_4,0}\wedge\lnot b_{o_4,1})\vee(b_{o_5,1}\wedge\lnot b_{o_5,0})\vee(b_{o_6,1}\wedge\lnot b_{o_6,0})\bigr],
\]
\[
\Phi_K^{x_1}=b_{o_1,1}\wedge b_{o_4,1},
\qquad
\Phi_K^{x_2}=b_{o_2,1}\wedge b_{o_5,1},
\qquad
\Phi_K^{x_3}=b_{o_3,0}\wedge b_{o_6,0},
\qquad
\Phi_K=\Phi_K^{x_1}\wedge\Phi_K^{x_2}\wedge\Phi_K^{x_3}.
\]
Substituting the values from the table: in $\Phi_C$'s first bracket, $o_1$ gives $(1\wedge1)=1$ while $o_2,o_3$ give $0$, so the bracket is $1$; in the second, $o_5$ gives $(1\wedge1)=1$ while $o_4,o_6$ give $0$, so that bracket is also $1$; hence $\Phi_C=1$.
Each $\Phi_K^x$ is a conjunction of two matching bits (e.g.\ $\Phi_K^{x_1}=1\wedge1=1$), so $\Phi_K=1$ as well.
\end{example*}

The reduction realizes $\Phi_C$ as a clause scan and each $\Phi_K^x$ as one audit branch in the same RMDP.
Nature selects the global and local bits.
The two paths share every local-bit choice, so they evaluate the same assignment.
\begin{definition}[Boolean comparison RMDP]
\label{def:boolean-comparison-rmdp}
For a 3-CNF formula $\varphi=\bigwedge_{i=1}^m C_i$ over variables $x_1,\ldots,x_n$ with $n\ge1$, no tautological clause, no repeated literal within a clause, and no variable absent from every clause, the Boolean comparison RMDP is $M_\varphi=\tup{S,A,\cU,R,\sinit,\gamma_0}$, with the following components.
\begin{itemize}
\item $S$ is the disjoint union of the following groups.
\begin{itemize}
\item The initial state $\sinit$.
\item Clause-scan states: a control state $f_{i,j}$ for each literal occurrence $o=(i,j)$, together with the terminals $\mathrm{acc}_C$ and $\mathrm{rej}_C$.
\item Audit states: global selectors $q_{X_x}$ with outcomes $q_{X_x}^0,q_{X_x}^1$, controls $k_{x,c,\ell}$ that inspect the $\ell$th occurrence of $x$ after committing to $c\in\{0,1\}$, together with the terminals $\mathrm{acc}_K$ and $\mathrm{rej}_K$.
\item Shared local-bit selectors: for each occurrence $o$ and $c\in\{0,1\}$, the selector of
\Cref{def:local-bit-verifier-primitives}, used by both policies.
\item Padding states $S_{\mathrm{pad}}$, reward-free, fixed once the transitions below are defined.
\item A ruinous sink $\bot_{\rm r}$.
\end{itemize}
\item At $\sinit$, action $a_C$ enters $f_{1,1}$, whereas action $a_K$ has the fixed uniform transition to $q_{X_x}$, $x\in\{x_1,\ldots,x_n\}$.
Every selector has one action.
Outcome $q_{X_x}^c$ leads to $k_{x,c,1}$, each $f_{i,j}$ leads to the first local selector for its occurrence, and each $k_{x,c,\ell}$ leads to the corresponding $q_{o,c}$.
From a shared outcome $q_{o,c}^b$, the clause-scan action and audit action have their respective deterministic continuations, ending in $\mathrm{acc}_C$/$\mathrm{rej}_C$ or $\mathrm{acc}_K$/$\mathrm{rej}_K$ as described below.
All other control states have the indicated unique deterministic continuation.
Fix $H$ at least as long as the longest of these paths.
$S_{\mathrm{pad}}$ consists of fresh states inserted along every shorter path so it also reaches its terminal after exactly $H$ transitions, each with a single reward-free action continuing toward that terminal.
This common depth is necessary because a shared accepting terminal can be reached at
realization-dependent depths, so one terminal payoff cannot otherwise cancel every discount factor.
\item $\cU$ is the product of the two-Dirac segments in
\Cref{eq:boolean-selector-choice}, one for every global bit $X_x$ and local bit $(o,c)$.
All remaining described choices are singletons, including the certain uniform splitter at $\sinit$.
Its vertices are exactly the Boolean realizations $p_z\in\{0,1\}$.
\item $\mathrm{acc}_C$ is the terminal with payoff $\gamma_0^{-H}$ and $\mathrm{acc}_K$ the terminal with payoff $-\gamma_0^{-H}$, while every other described reward, including at $\mathrm{rej}_C,\mathrm{rej}_K$ and at padding states, is zero.
\item All remaining choices use ruinous-sink completion.
\item The initial state is $\sinit$ and the discount is $\gamma_0$.
\end{itemize}
\end{definition}
Fix a realization $u\in\cU$ and a stationary policy $\pi$.
Since $\pi$ prescribes one action at every state and $u$ fixes one outcome at every two-Dirac selector, the only randomness left in $\pi$'s run under $u$ comes from any fixed-probability transitions $\pi$ itself uses, such as the certain uniform splitter.
Because the graph is acyclic and $S$ is finite, this run reaches a terminal with probability one.
Say $\pi$ \emph{accepts} under $u$ with probability equal to the chance of reaching an accepting terminal; if $\pi$ never uses such a transition -- as is the case for $\pi_C$ -- this probability is always $0$ or $1$, and we speak of $\pi$ accepting or rejecting outright.

\paragraph{The clause policy.}
Policy $\pi_C$ processes clauses and their literals in order.
At occurrence $o=(i,j)$, it uses the pair test of
\Cref{def:local-bit-verifier-primitives}, accepting the locally true outcome pair.
A true literal advances to the next clause, a false literal advances to the next literal, and an entirely false clause rejects.
Satisfying the last clause accepts.
Thus $\pi_C$ checks exactly $\Phi_C$.

\paragraph{The consistency policy.}
Policy $\pi_K$ first takes the certain uniform splitter, which selects a variable $x$.
It reads $X_x$ and thereby commits to $c\in\{0,1\}$.
It then uses the audit chain of \Cref{def:local-bit-verifier-primitives} over $O_x$ in clause-major
order.
Conditioned on the certain uniform splitter selecting $x$, $\pi_K$ therefore checks exactly $\Phi_K^x$.

\begin{lemma*}[Unambiguous prescriptions]
Both policies above are stationary deterministic policies.
\end{lemma*}
\begin{proof}
The clause scan encounters each $f_{i,j}$ once and reaches the two selectors of an occurrence in the fixed order prescribed above.
An audit branch commits at $q_{X_x}^c$ before reaching any $k_{x,c,\ell}$.
Thus each policy prescribes a single continuation at every outcome state it can reach.
At a shared outcome $q_{o,c}^b$, the two roles use distinct actions, so their continuations need not agree.
Complete each policy at its unreachable states with a fixed default action, using ruinous-sink completion when no role action was described there.
\end{proof}
By \Cref{lem:acceptance-difference-verifier}, their value difference is the sum of their acceptance
probabilities.

The shared fragment for a positive occurrence is shown in \Cref{fig:comparison-boolean-rmdp}.
The audit path shown is the branch that committed to $X_x=1$.

\begin{figure}[t]
\centering
\resizebox{0.98\columnwidth}{!}{%
\begin{tikzpicture}[
  >=Latex,
  state/.style={draw,rounded corners,minimum height=5.2mm,align=center,inner sep=2pt},
  box/.style={draw,rounded corners,minimum height=5.2mm,align=center,inner sep=2pt},
  font=\scriptsize,
  x=1cm,y=1cm
]
\node[box] (split) at (-4.0,-1.45) {certain uniform splitter};
\node[state] (qx) at (-1.80,-1.45) {$q_{X_x}$};
\node[state] (qx1) at (-0.50,-1.45) {$q_{X_x}^1$};
\node[box] (kin) at (0.7,-1.45) {$k_{x,1,\ell}$};
\node[box] (cin) at (-1.0,0.75) {$f_{i,j}$: test $x$};
\node[state] (q1) at (2.35,0) {$q_{o,1}$};
\node[state] (q10) at (4.05,0.8) {$q_{o,1}^0$};
\node[state] (q11) at (4.05,-0.8) {$q_{o,1}^1$};
\node[box] (fail) at (6.2,1.15) {$\pi_C$: next literal\\$\pi_K$: reject};
\node[state] (q0) at (6.0,-0.35) {$q_{o,0}$};
\node[box] (passk) at (6.35,-1.45) {$\pi_K$: next copy\\or accept};
\node[state] (q00) at (7.8,0.15) {$q_{o,0}^0$};
\node[state] (q01) at (7.8,-0.85) {$q_{o,0}^1$};
\node[box] (passc) at (9.8,0.15) {$\pi_C$: next clause\\or accept};
\node[box] (failc) at (9.8,-0.85) {$\pi_C$: next literal\\or reject};

\draw[->] (split) -- node[below] {$1/n$} (qx);
\draw[->] (qx) -- node[below] {$p_{X_x}$} (qx1);
\draw[->] (qx1) -- (kin);
\draw[->] (cin) -- (q1);
\draw[->] (kin) -- (q1);
\draw[->] (q1) -- node[above left] {$1-p_{o,1}$} (q10);
\draw[->] (q1) -- node[below left] {$p_{o,1}$} (q11);
\draw[->] (q10) -- (fail);
\draw[->] (q11) -- node[above] {$\pi_C$} (q0);
\draw[->] (q11) -- node[below] {$\pi_K$} (passk);
\draw[->] (q0) -- node[above left] {$1-p_{o,0}$} (q00);
\draw[->] (q0) -- node[below left] {$p_{o,0}$} (q01);
\draw[->] (q00) -- (passc);
\draw[->] (q01) -- (failc);
\end{tikzpicture}%
}
\caption{The shared fragment for a positive occurrence $o=(i,j)$ and the audit branch $X_x=1$.
The clause scan tests the pair $(b_{o,1},b_{o,0})$, whereas this audit branch needs only $b_{o,1}$.
The branch $X_x=0$ analogously audits $q_{o,0}$.
Selector arrows are uncertain; arrows from outcome states are the certain, role-specific actions of the two policies.}
\label{fig:comparison-boolean-rmdp}
\end{figure}
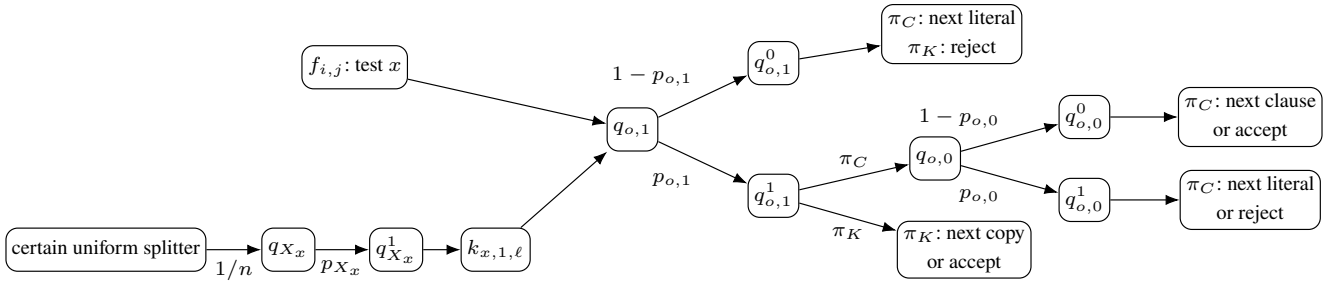
\begin{lemma*}[Policy semantics]
At a vertex $\vec u$ of the uncertainty polytope, $\pi_C$ accepts under $\vec u$ exactly when $\Phi_C$ holds under $\vec u$.
Conditioned on the certain uniform splitter selecting $x$, $\pi_K$ accepts under $\vec u$ exactly when $\Phi_K^x$ holds under $\vec u$; unconditionally, this makes $\pi_K$'s probability of accepting under $\vec u$ equal to $\frac1n|\{x:\Phi_K^x\text{ holds under }\vec u\}|$, the fraction of variables whose audit would pass.
Moreover,
\[
D_{\pi_C,\pi_K}(\vec u)
=\mathbf{1}_{\{\pi_C\text{ accepts under }\vec u\}}
 +\Pr[\pi_K\text{ accepts under }\vec u].
\]
\end{lemma*}
\begin{proof}
At a vertex, each transition in \eqref{eq:boolean-selector-choice} selects one Boolean value.
The clause path implements \eqref{eq:clause-local-satisfaction}, and the audit branch for $x$ implements its conjunct in \eqref{eq:clause-local-consistency}.
Padding places every final state at depth $H$, so the accepting rewards contribute $1$ and $-1$ to the two policy values.
The certain uniform splitter averages the audit contribution over the $n$ variables, and rejecting paths contribute zero.
\end{proof}

\paragraph{Example.}
For the satisfying assignment in the encoding table, $\pi_C$ confirms $x_1$ in $C_1$ and $x_2$ in $C_2$, and every variable audit passes, giving difference $2$.
Now take the unsatisfying valuation $x_1=0,x_2=1,x_3=0$ and its canonical local pairs.
The clause scan rejects $C_1$, although every audit passes.
If only the copy of $x_1$ in $C_1$ is changed from $(0,1)$ to $(1,0)$, the clause scan accepts, but the $x_1$ audit, committed to $0$, reads $b_{o,0}=0$ and rejects.
The malformed pairs $(0,0)$ and $(1,1)$ cannot make a literal locally true because the clause scan tests both bits.

\begin{lemma}\label{lem:common-label-monolithic}
There is a vertex under which $\pi_C$ accepts and every variable audit passes if and only if $\varphi$ is satisfiable.
\end{lemma}
\begin{proof}
This is the clause-local equivalence, together with the policy semantics established above.
\end{proof}
The uncertain choices are independent line segments between two Dirac distributions.
Hence the RMDP is $(s,a)$-rectangular and every uncertain choice is two-Dirac.
The only other stochastic row is the certain uniform splitter.
After deleting terminal self-loops, the full transition graph is a DAG:
after the initial state, order the global-selector blocks first, then the occurrence blocks in clause-major order, placing within each occurrence the selector for $\operatorname{val}(o)$ before the other selector, and finally the padding and terminal states.
Every control and outcome state can be placed immediately before or after its associated selector.
In particular, no run under either policy uses the same uncertain choice twice, and every terminal that both policies reach is absorbing under each of them, so every shared choice is cycle-free under both policies in the sense of \Cref{def:cyclefree-shared}.
By \Cref{lem:cyclefree-vertex}, $\pi_C$ and $\pi_K$'s value difference attains its maximum at a tuple of choice vertices.
By this and the two lemmas above,
\[
\begin{aligned}
\varphi\in\mathrm{SAT}
&\Longrightarrow \Delta_{\cU}(\pi_C,\pi_K)=2,\\
\varphi\in\mathrm{UNSAT}
&\Longrightarrow \Delta_{\cU}(\pi_C,\pi_K)\le 2-\frac1n.
\end{aligned}
\]
In the first case, nature chooses the canonical vertex of a satisfying valuation.
In the second, a vertex satisfying $\Phi_C$ has at least one failing variable audit by \Cref{lem:common-label-monolithic}, and therefore has difference at most $1+(n-1)/n$; a vertex violating $\Phi_C$ has difference at most $1$.
Vertex sufficiency shows that an interior realization cannot do better.

\begin{proof}[Proof of \Cref{thm:comparison-conp}]
Use threshold $2-\frac1{2n}$.
The construction satisfies
\[
\varphi\in\UNSAT
\iff
\Delta_{\cU}(\pi_C,\pi_K)\le 2-\frac1{2n}.
\]
Indeed,
\[
2-\frac1n < 2-\frac1{2n} < 2.
\]
The construction has polynomial size, $\gamma_0^{-H}=(400/361)^H$ has $O(H)$ bits, and all uncertain choices are independent two-Dirac choices.
This is a polynomial reduction from UNSAT.
\end{proof}

\subsection{Vertex-Extremal Membership}
\label[appendix]{app:comparison-membership}
\begin{theorem}[Vertex-extremal membership]\label{thm:comparison-conp-membership}
Robust policy comparison is in $\scoNP$ for $(s,a)$-rectangular polytopes when both policies are cycle-free on shared choices (\Cref{def:cyclefree-shared}).
\end{theorem}
\begin{definition}[Vertex-extremal pair]\label{def:vertexextremal-pair}
The pair $(\pi_1,\pi_2)$ is \emph{vertex-extremal} if
\[
\sup_{\vec u\in\cU}D_{\pi_1,\pi_2}(\vec u)=\max_{\vec v\in\operatorname{Vert}(\cU)}D_{\pi_1,\pi_2}(\vec v).
\]
\end{definition}

\begin{lemma}[Cycle-freeness on shared choices implies vertex-extremality]\label{lem:cyclefree-vertex}
If $(\pi_1,\pi_2)$ is cycle-free on shared choices (\Cref{def:cyclefree-shared}), then it is vertex-extremal (\Cref{def:vertexextremal-pair}).
\end{lemma}
The main-body example \Cref{ex:comparison-interior} shows why the hypothesis is necessary: one shared choice on a policy cycle can create an irrational interior maximum, so endpoint evaluation is no longer sound.
We use two standard facts.
A separately affine function on a product of polytopes is extremized at a tuple of vertices by optimizing one block at a time.
A finite linear-fractional function on a polytope is extremized at a vertex because its value on a segment lies between its endpoint values.
\begin{proof}[Proof of \Cref{lem:cyclefree-vertex}]
Fix a single block $\cU_{s_0,a_0}$, hold every other block fixed, and consider two cases according to whether $(s_0,a_0)\in\mathrm{Sh}(\pi_1,\pi_2)$.
Then iterate over the blocks.
\begin{itemize}
    \item \emph{$(s_0,a_0)\notin\mathrm{Sh}(\pi_1,\pi_2)$}, say only $\pi_1$ uses it.
    Then $V_{\vec u}^{\pi_2}(\sinit)$ is constant in this block, and $D_{\pi_1,\pi_2}$ differs from $V_{\vec u}^{\pi_1}(\sinit)$ by a constant.
    By Cramer's rule applied to the one equation of $\pi_1$'s Bellman system that this block enters,
    $V_{\vec u}^{\pi_1}$ is linear-fractional in this block, a ratio of two functions each affine in
    it regardless of cycles elsewhere in the graph, so the standard fact above makes it, and hence
    $D_{\pi_1,\pi_2}$, extremal at a vertex.
\item \emph{$(s_0,a_0)\in\mathrm{Sh}(\pi_1,\pi_2)$} (hence, by hypothesis, cycle-free under both $\pi_1$ and $\pi_2$).
If $s_0$ is absorbing under one of the two policies, then that policy uses $(s_0,a_0)$ and the choice has row $\{\delta_{s_0}\}$, so the block is the singleton $\{\delta_{s_0}\}$ and $D_{\pi_1,\pi_2}$ is constant in it.
Otherwise no self-loop was omitted at $s_0$, so cycle-freeness is the literal graph condition.
Under either policy, a run then uses the choice $(s_0,a_0)$ at most once.
The probability of reaching $s_0$ is independent of its outgoing row, and the continuation value after leaving $s_0$ cannot depend on that row because the run never returns.
Conditioning on reaching and selecting $(s_0,a_0)$ therefore makes each policy value affine in the free block.
Their difference is affine as well and is extremized at a vertex.
\end{itemize}
Every block falls under one of the two cases, so $D_{\pi_1,\pi_2}$ is extremal at a vertex in every block, giving vertex-tuple extremality.
\end{proof}
\begin{proof}[Proof of \Cref{thm:comparison-conp-membership}]
By \Cref{lem:cyclefree-vertex}, the pair is vertex-extremal.
For the strict complement, guess a vertex tuple $\vec v$ with $D_{\pi_1,\pi_2}(\vec v)>t$.
Such a tuple has polynomial encoding length because each component vertex solves a full-rank
subsystem of tight rational constraints.
The verifier checks $\vec v\in\cU$, solves the two rational Bellman systems, and compares their
initial values in polynomial time.
Thus the complement is in $\sNP$.
\end{proof}
\subsection{Shared Cycles and Square-Root-Sum Hardness}
\label[appendix]{app:shared-cycle-sqrs}
\begin{theorem}[Shared-cycle algebraic hardness]\label{thm:comparison-cosqrs}
Robust policy comparison is $\scoSRS$-hard for deterministic policies on $(s,a)$-rectangular RMDPs in which every uncertain choice is two-Dirac and the only other stochastic rows are certain uniform splitters.
\end{theorem}
We construct a gadget with a shared choice on a cycle, which creates one normalized square-root term.
An exact derivative calculation locates its interior maximum, and a certain splitter combines
independently chosen local maxima into the target sum of square roots.
The delicate point is sign control at the critical point.
All choices stay $(s,a)$-rectangular and use the appendix-wide discount $\gamma_0$ required by the exact lift.

Fix
\[
b\ge2,
\qquad m=\lceil\sqrt b\rceil.
\]
We round $\sqrt b$ up to the nearest integer $m$, so that $b\le m^2$ (needed below), while $m$ stays close enough to $\sqrt b$ for the final identity to come out exactly right.
The local gadget in \Cref{fig:shared-cycle-root-gadget} has entry state $s$, intermediate states
$s_1,s_2$, and a zero-reward sink.
Policy $\pi_u$ chooses $u$ at $s$ and moves to $s_1$ with reward $r_u$.
Policy $\pi_v$ chooses $v$ and moves to $s_2$ directly, with reward zero.
The forced edge $s_1\to s_2$ has reward zero: its only purpose is to give $\pi_u$ one extra step before reaching $s_2$, so that its eventual return to $s$ is discounted one extra factor of $\gamma_0$ compared to $\pi_v$'s.
This asymmetry in cycle length is essential: without it, $Q_u-Q_v$ below would be a ratio of two
affine functions, which is monotone and has no interior maximum.
The same mechanism appears in \Cref{ex:comparison-interior}, where one policy closes the cycle and the other exits immediately.
At $s_2$, the unique shared uncertain choice returns to $s$ with probability $p$ and enters the sink with probability $1-p$, paying reward $c_v$ either way.
\begin{figure}[t]
\centering
\begin{tikzpicture}[node distance=12mm and 17mm,
 state/.style={circle,draw,minimum size=8mm,inner sep=1pt}]
\node[state] (s) {$s$};
\node[state,right=of s,yshift=8mm] (s1) {$s_1$};
\node[state,right=of s,yshift=-8mm] (s2) {$s_2$};
\node[state,right=21mm of s2] (bot) {$\bot$};
\draw[->] ($(s.west)+(-6mm,0)$) -- (s.west);
\draw[->] (s) -- node[above,sloped] {$u,r_u$} (s1);
\draw[->] (s) -- node[below,sloped] {$v,0$} (s2);
\draw[->] (s1) -- node[right] {$0$} (s2);
\draw[->] (s2) to[bend left=30] node[below] {$p,c_v$} (s);
\draw[->] (s2) -- node[below] {$1-p,c_v$} (bot);
\path[->] (bot) edge[loop right] node[right] {$0$} (bot);
\end{tikzpicture}
\caption{The shared-cycle gadget realizing one normalized square-root term.}
\label{fig:shared-cycle-root-gadget}
\end{figure}
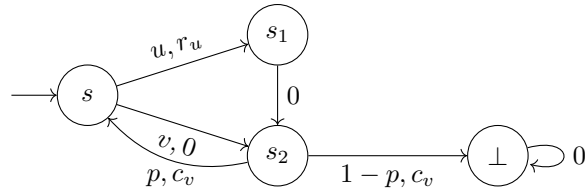
Our goal is to choose $U,V,c_v,r_u$ so that $\sup_{p\in[0,1]}\bigl(Q_u(p)-Q_v(p)\bigr)$ takes the form $C_b-\sqrt b$, for a constant $C_b$ depending only on $b$.
\Cref{lem:normalized-root} establishes this.
Set
\[
U=\frac{m(1-\gamma_0)}{2\gamma_0},
\]
\[
V=\frac{b(1-\gamma_0)}{2m},
\]
\[
c_v=\frac{V}{\gamma_0},
\]
\[
r_u=U-\gamma_0V
=\frac{(1-\gamma_0)(m^2-\gamma_0^2b)}{2\gamma_0m}.
\]
All constants are nonnegative: $U,V,c_v>0$, while $m^2\ge b>\gamma_0^2b$ gives $r_u>0$.
The Bellman equations give, writing $Q_i(p):=V_{\vec u}^{\pi_i}(s)$ for the value of $\pi_i$ at the entry state as a function of the shared return probability $p$,
\[
Q_u(p)=r_u+\gamma_0^2c_v+\gamma_0^3pQ_u(p)
=\frac{U}{1-\gamma_0^3p},
\]
\[
Q_v(p)=\gamma_0c_v+\gamma_0^2pQ_v(p)
=\frac{V}{1-\gamma_0^2p}.
\]

\begin{lemma}[Rational-fraction maximum]\label{lem:normalized-root}
With the constants above,
\[
\sup_{p\in[0,1]}\bigl(Q_u(p)-Q_v(p)\bigr)
=C_b-\sqrt b,
\qquad
C_b=\frac{m}{2\gamma_0}+\frac{\gamma_0b}{2m}.
\]
\end{lemma}
\begin{proof}
Put $x=1-\gamma_0^2p$.
Then $x\in[1-\gamma_0^2,1]$ and
\[
f(x)=\frac{U}{1-\gamma_0+\gamma_0x}-\frac{V}{x}.
\]
After multiplying by the positive denominators, the sign of $f'(x)$ is the sign of
\[
\sqrt V(1-\gamma_0+\gamma_0x)-\sqrt{\gamma_0U}\,x.
\]
Since $V\le\gamma_0U$, the affine sign factor has negative slope.
It is positive before its unique zero and negative after it, so the zero is a maximum.
It lies in the interval provided
\[
V\le\gamma_0U
\quad\text{and}\quad
\gamma_0U\le h_{\gamma_0}^2V,
\qquad
h_{\gamma_0}=\frac{1-\gamma_0^3}{1-\gamma_0^2}
=\frac{434721}{304400}.
\]
The first inequality is $b\le m^2$.
The second follows from $m^2/b\le2$ and the exact rational comparison $h_{\gamma_0}^2>2$.
At the critical point, direct substitution yields
\[
\max f=\frac{(\sqrt U-\sqrt{\gamma_0V})^2}{1-\gamma_0}
=\frac{m}{2\gamma_0}+\frac{\gamma_0b}{2m}-\sqrt b.
\]
This value is nonnegative by the arithmetic-geometric mean inequality applied to $m/\gamma_0$ and $\gamma_0b/m$.
\end{proof}
\paragraph{Combining independent copies.}
For inputs $b_1,\ldots,b_n$, take disjoint copies of the gadget.
A fresh certain uniform splitter enters each copy with probability $1/n$.
Scaling all rewards in every copy by $n/\gamma_0$ cancels the splitter probability and first discount.
After this scaling, every remaining choice uses ruinous-sink completion.
Rectangularity makes the local parameters independent, so
\[
\Delta_{\cU}(\pi_u,\pi_v)
=\sum_{i=1}^n C_{b_i}-\sum_{i=1}^n\sqrt{b_i}.
\]
\begin{proof}[Proof of \Cref{thm:comparison-cosqrs}]
Reduce from the $\scoSRS$ variant asking whether $\sum_i\sqrt{b_i}\ge k$.
Set
\[
t=\sum_iC_{b_i}-k.
\]
Then $\Delta_{\cU}(\pi_u,\pi_v)\le t$ exactly when the $\scoSRS$ instance is positive.
Every uncertain choice is two-Dirac.
\end{proof}
\subsection{Real-Hierarchy Completeness of Policy Comparison}
\label[appendix]{app:comparison-real}
\begin{theorem}[General policy comparison]\label{thm:comparison-forallr}
Robust policy comparison is $\sUTR$-complete for deterministic policies under general rational polytopic uncertainty.
\end{theorem}

\begin{lemma}[$\sUTR$ membership for policy comparison]\label{lem:comparison-forallr-membership}
Robust policy comparison is in $\sUTR$ for deterministic policies under general rational polytopic uncertainty.
\end{lemma}
\begin{proof}
Apply \Cref{lem:universal-bellman-encoding} with $\tau$ fixed to $\pi_1$ and the singleton family $\{\pi_2\}$.
\end{proof}

\begin{definition}[Polynomial-evaluation RMDP]
\label{def:polynomial-evaluator}
Fix an explicit sparse polynomial
\[
f(p)=\sum_{\ell=1}^N c_\ell\prod_{j=1}^{d_\ell}p_{i_{\ell,j}},
\qquad p\in[0,1]^m,
\]
and a rational discount $\gamma\in(0,1)$.
The component $\operatorname{Poly}_\gamma(f)$ is defined as follows.
\begin{itemize}
\item A reward-zero certain uniform splitter at $s_f$ enters branch $\ell$ with probability $1/N$.
That branch has states $b_\ell^0,\ldots,b_\ell^{d_\ell}$ and a common zero-reward sink $\bot$.
\item At $b_\ell^{j-1}$, the unique action continues to $b_\ell^j$ with probability $u_{\ell,j}$ and enters $\bot$ otherwise.
State $b_\ell^{d_\ell}$ is absorbing with value $r_\ell=N\gamma^{-(d_\ell+1)}c_\ell$.
\item Each $u_{\ell,j}$ is a coordinate of a two-successor row in $[0,1]$.
Occurrences representing the same logical coordinate may be tied by rational linear equalities.
\item At each state the action described above is intended, every other described reward is zero, and
$R(b_\ell^{d_\ell})=(1-\gamma)r_\ell$.
\item All remaining choices use ruinous-sink completion.
\end{itemize}
For policy comparison, choose a representative $u_i^{\rm rep}$ for every parameter and impose
$u_{\ell,j}=u_i^{\rm rep}$ whenever $i_{\ell,j}=i$; occurrences of $1-p_i$, when present, instead
satisfy $u_{\ell,j}=1-u_i^{\rm rep}$.
Together with stochasticity and box constraints, these equalities define $\cU$ and the resulting RMDP
$M_f=\tup{S,A,\cU,R,s_f,\gamma}$.
Repeated indices in a monomial represent powers and are realized by distinct rows tied to the same representative.
\end{definition}

\Cref{fig:comparison-polynomial-rmdp} shows one monomial branch of this construction.
\begin{figure}[t]
\centering
\begin{tikzpicture}[node distance=14mm and 18mm,
 state/.style={circle,draw,minimum size=8mm,inner sep=1pt}]
\node[state] (b0) {$b_\ell^0$};
\node[state,right=of b0] (b1) {$b_\ell^1$};
\node[right=of b1] (dots) {$\cdots$};
\node[state,right=of dots] (bd) {$b_\ell^{d_\ell}$};
\node[state,below=16mm of dots] (bot) {$\bot$};
\draw[->] (b0) -- node[above] {$u_{\ell,1}$} (b1);
\draw[->] (b1) -- node[above] {$u_{\ell,2}$} (dots);
\draw[->] (dots) -- node[above] {$u_{\ell,d_\ell}$} (bd);
\draw[->] (b0) -- node[below,sloped] {$1-u_{\ell,1}$} (bot);
\draw[->] (b1) -- node[below,sloped] {$1-u_{\ell,2}$} (bot);
\draw[->] (dots) -- node[below,sloped] {$1-u_{\ell,d_\ell}$} (bot);
\path[->] (bd) edge[loop above] node[above] {$(1-\gamma)r_\ell$} (bd);
\path[->] (bot) edge[loop right] node[right] {$0$} (bot);
\end{tikzpicture}
\caption{One branch of the polynomial-evaluation RMDP.
Equalities in $\cU$ tie repeated parameter occurrences across branches.}
\label{fig:comparison-polynomial-rmdp}
\end{figure}
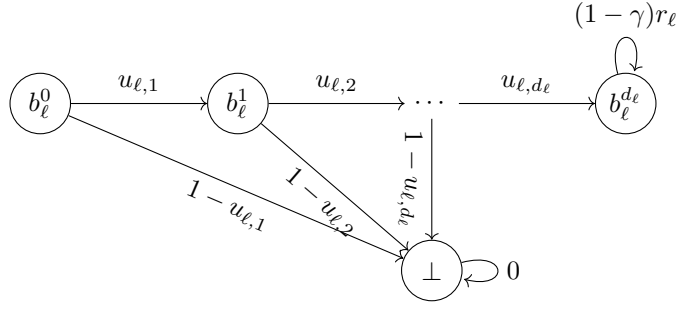

\begin{lemma}[Polynomial evaluation]\label{lem:polynomial-evaluator}
For every realization of the occurrence rows, the compliant value of
$\operatorname{Poly}_\gamma(f)$ is
$\sum_\ell c_\ell\prod_j u_{\ell,j}$.
In particular, when the representative coordinates equal $p$, the unique compliant policy of $M_f$
has value $f(p)$ at $s_f$.
The component is acyclic apart from absorbing terminals, every uncertain row has two successors, and
its only other stochastic row is a certain uniform splitter.
It has size $O(N\max_\ell d_\ell)$ and rewards of polynomial encoding length.
Moreover, $\cU$ is one rational polytope of polynomial description size.
\end{lemma}
\begin{proof}
Branch $\ell$ is reached with probability $1/N$, survives with probability $\prod_ju_{\ell,j}$, and
reaches its terminal after $d_\ell+1$ transitions.
Its discounted contribution is therefore
\[
\frac1N\gamma^{d_\ell+1}
\left(N\gamma^{-(d_\ell+1)}c_\ell\right)
\prod_ju_{\ell,j},
\]
and summing gives the first identity.
The representative equalities then give $f(p)$.
Each occurrence contributes one row.
Every branch is a directed chain, and the payoff exponent is at most the polynomial input size.
The tying, complement, stochasticity, and interval constraints are rational linear constraints, so their intersection is a polynomial-size rational polytope.
\end{proof}

\begin{lemma}[$\sUTR$ hardness for policy comparison]\label{lem:comparison-forallr-hardness}
Robust policy comparison is $\sUTR$-hard for deterministic policies under general rational polytopic uncertainty.
\end{lemma}
\begin{proof}
By the strict degree-six normal form of \citet[Lemma~2.8]{DBLP:journals/mst/SchaeferS24} and their bounded-open equivalence~\citep[Proposition~2.12]{DBLP:journals/mst/SchaeferS24}, $\exists p\in[0,1]^m:f(p)>0$ is $\sETR$-complete for an explicitly represented degree-six polynomial.
Indeed, the affine image of the source box is $(0,1)^m$, and strict positivity there is equivalent to strict positivity on its closure by continuity.
Its complement is $\forall p:f(p)\le0$.
Construct $M_f$ by \Cref{def:polynomial-evaluator}, add a fresh initial choice between $s_f$ and a
zero-reward sink, and let deterministic policies $\pi_f,\pi_\bot$ choose the two actions.
By \Cref{lem:polynomial-evaluator},
\[
\Delta_\cU(\pi_f,\pi_\bot)\le0
\iff \forall p\in[0,1]^m:\ f(p)\le0.
\]
This proves hardness.
\end{proof}
\begin{proof}[Proof of \Cref{thm:comparison-forallr}]
Immediate from \Cref{lem:comparison-forallr-membership,lem:comparison-forallr-hardness}.
\end{proof}
\subsection{Portfolio Comparison and the Shared-Selector Evaluator}
\label{app:portfoliocomparison}
\begin{problem}[Portfolio comparison]\label{prob:portfolio-comparison}
Given an RMDP $M$, a policy $\pi_0\in\polRand$, an input portfolio $\Pi\subset\polRand$, and rational $t$, decide whether
\[
\Delta_{\cU}(\pi_0,\Pi)
=\sup_{\vec u\in\cU}\left(
V_{\vec u}^{\pi_0}(\sinit)-\max_{\pi\in\Pi}V_{\vec u}^{\pi}(\sinit)
\right)\le t.
\]
\end{problem}

\begin{theorem}\label{thm:portfolio-comparison-atr}
Portfolio comparison is $\sUTR$-complete, with portfolio size part of the input.
Hardness already holds for deterministic policies and acyclic $(s,a)$-rectangular RMDPs with two-successor uncertain choices.
\end{theorem}
\begin{lemma}[Membership]\label{lem:portfolio-comparison-membership}
Portfolio comparison is in $\sUTR$.
\end{lemma}
\begin{proof}
Apply \Cref{lem:universal-bellman-encoding} with $\tau$ fixed to $\pi_0$ and the explicit family $\Pi$.
\end{proof}

\begin{problem}[Strict elementary feasibility]\label{prob:strict-elementary}
Given rational expressions $g_1,\ldots,g_r$ over variables $x\in[0,1]^n$, each affine or affine plus one bilinear term, with no variable occurring twice in the same expression, decide whether
\[
g_1(x)>0,\ldots,g_r(x)>0
\]
has a solution $x\in[0,1]^n$.
\end{problem}
\begin{lemma}\label{lem:strict-elementary}
\Cref{prob:strict-elementary} is $\sETR$-complete.
\end{lemma}
\begin{proof}
Membership is immediate.
For hardness, use the bounded degree-four equality normal form underlying
\Cref{def:policy-affine-normal-form}.
It is $\sETR$-complete to decide
\[
\exists y\in[-1,1]^n:\quad f(y)=0,
\]
and substituting $y_i=2x_i-1$ changes the domain to $[0,1]^n$.
Apply \Cref{def:policy-affine-normal-form} without its optional output coordinate and append the affine
residual $\widehat f/B$, which sets the decoded polynomial to zero.
Denote all residuals by $h_1,\ldots,h_t$.
If they have no common zero on the compact box, then
\[
\eta=\min_x\max_j|h_j(x)|>0.
\]
Clear denominators and put $F=\sum_jh_j^2$.
This is a nonnegative integer polynomial of degree at most four, with polynomial coefficient bit length.
If there is no common zero, apply the effective \L{}ojasiewicz bound of \citet[Theorem~2.2]{DBLP:journals/mst/SchaeferS24} on the compact box to $F$ and the constant polynomial one.
It gives a lower bound $2^{-\ell D^{cN^2}}$ in terms of the number $N$ of variables, degree $D\le4$, and coefficient length $\ell$.
Consequently an explicit polynomial $p$ in the input length satisfies $\eta>2^{-2^p}$.
Introduce positive variables $a_0,b_0,\ldots,a_p,b_p$ with
\[
0<a_0<\frac12,\qquad 0<b_0<\frac12,
\]
\[
0<a_{j+1}<\frac{a_jb_j}{4},
\qquad
0<b_{j+1}<\frac{a_jb_j}{4}
\quad(j<p).
\]
The exponent recurrence gives $a_p<2^{-2^p}$.
Replace $h_j=0$ by
\[
a_p+h_j>0,
\qquad
a_p-h_j>0.
\]
An exact normal-form solution satisfies the strict system.
Conversely, a strict solution would have $|h_j|<a_p<\eta$ for every $j$, contradicting the definition
of $\eta$ when the residual system is infeasible.
Each constraint is affine or affine plus one product of distinct variables.
Input copies ensure that an input variable occurs only in affine copy residuals.
\end{proof}

\begin{lemma}[Hardness]\label{lem:portfolio-comparison-hardness}
Portfolio comparison is $\sUTR$-hard already for deterministic policies and acyclic $(s,a)$-rectangular RMDPs with two-successor uncertain choices.
\end{lemma}
\begin{definition}[Shared-selector evaluator]\label{def:shared-selector-evaluator}
Fix a strict elementary feasibility instance $g_1,\ldots,g_r$ over $x\in[0,1]^n$ (\Cref{prob:strict-elementary}) and a discount $\gamma$.
For each $i$, treat $-g_i$ as its constant term plus a list of monomials, each either linear or bilinear in two distinct variables, and order the variables of every monomial by increasing index.
The \emph{shared-selector evaluator} is the RMDP $M=\tup{S,A,\cU,R,\sinit,\gamma}$, where
\begin{itemize}
\item $S=\{s_0,\bot\}\cup\{q_j,q_j^0,q_j^1 : 1\le j\le n\}\cup\{\tau_{i,\ell}\}$: the initial state $s_0$; a fresh absorbing zero-reward state $\bot$; one \emph{selector} $q_j$ and its two successors $q_j^0,q_j^1$ per variable $x_j$, shared by every branch that mentions $x_j$; and one \emph{terminal} $\tau_{i,\ell}$ per branch $\ell$ of $-g_i$ (its constant term counts as one branch);
\item $A=\{a_0,a_1,\ldots,a_r\}\cup\{d_{i,j,b}\}$: at $s_0$, action $a_0$ leads deterministically to $\bot$, and action $a_i$ enters a rational certain splitter with one branch per term of $-g_i$; at $q_j^b$, the \emph{decoder} action $d_{i,j,b}$ is enabled for every policy $i$ whose current branch visits $x_j$; it leads to $\bot$ when $b=0$, and otherwise to the next selector or to the branch's terminal if $x_j$ was its last variable;
\item $\cU$ is $(s,a)$-rectangular: the certain splitter at $a_i$ and every decoder are deterministic;
the one uncertain choice at $q_j$ has distribution
$(1-x_j)\delta_{q_j^0}+x_j\delta_{q_j^1}$, the same distribution regardless of which branch or which
$g_i$ is passing through;
\item branch $\ell$, selected with splitter probability $w>0$ and reaching its terminal
$\tau_{i,\ell}$ after $d$ transitions, has terminal payoff $w^{-1}\gamma^{-d}$ times its coefficient
(including its sign); every other described choice pays reward $0$, and no common branch depth is
needed because each payoff cancels its own depth;
\item all remaining choices use ruinous-sink completion;
\item $\sinit=s_0$.
\end{itemize}
\end{definition}

\begin{lemma}[Evaluator policies]\label{lem:shared-selector-policies}
In the shared-selector evaluator of \Cref{def:shared-selector-evaluator}, the following are well-defined deterministic (stationary) policies.
\begin{itemize}
\item $\pi_0$ chooses $a_0$ at $s_0$ and a fixed ruinous-completed default choice at every other nonterminal state, which is unreachable under $\pi_0$.
\item For each $i\in\{1,\dots,r\}$, $\pi_i$ chooses $a_i$ at $s_0$; and for every variable $x_j$ occurring in $g_i$ and every outcome $b\in\{0,1\}$, $\pi_i$ chooses $d_{i,j,b}$ at $q_j^b$.
\end{itemize}
This assignment is well-defined because a stationary policy fixes a single action per state and no variable occurs twice in one expression $g_i$.
\end{lemma}
\begin{proof}[Proof of \Cref{lem:shared-selector-policies}]
No prescription is repeated within one policy because no variable occurs twice in one expression.
Prescriptions made by different policies need not agree.
\end{proof}
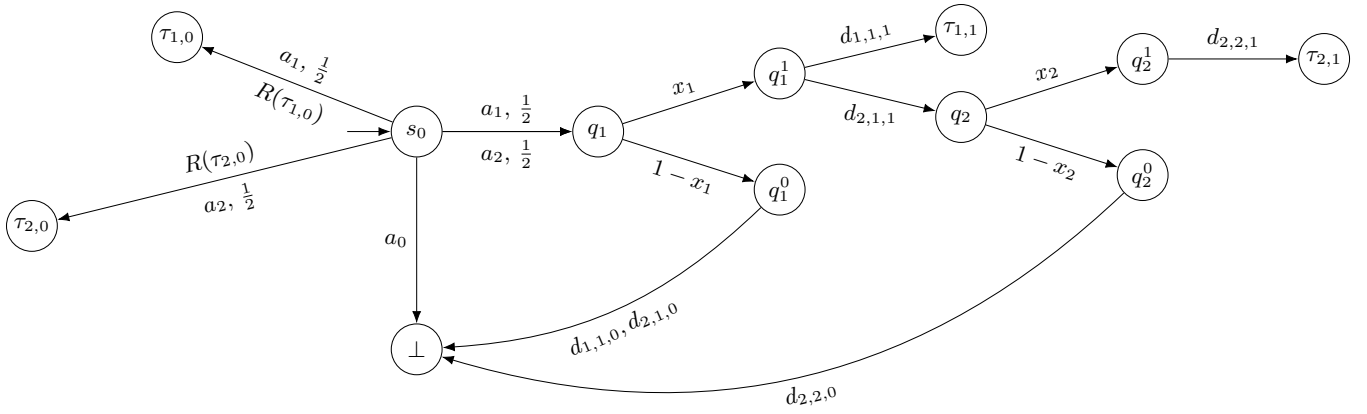
\begin{figure}[t]
\centering
\resizebox{\textwidth}{!}{%
\begin{tikzpicture}[
  >=Latex,
  state/.style={circle,draw,minimum size=7mm,inner sep=1pt},
  every node/.style={font=\small}
]
\node[state] (s0) at (0,0) {$s_0$};
\node[state] (bot) at (0,-3) {$\bot$};
\node[state] (t10) at (-3.3,1.3) {$\tau_{1,0}$};
\node[state] (t20) at (-5.3,-1.3) {$\tau_{2,0}$};
\node[state] (q1) at (2.5,0) {$q_1$};
\node[state] (q1a) at (5,0.8) {$q_1^1$};
\node[state] (q1b) at (5,-0.8) {$q_1^0$};
\node[state] (t11) at (7.5,1.4) {$\tau_{1,1}$};
\node[state] (q2) at (7.5,0.2) {$q_2$};
\node[state] (q2a) at (10,1) {$q_2^1$};
\node[state] (q2b) at (10,-0.6) {$q_2^0$};
\node[state] (t21) at (12.5,1) {$\tau_{2,1}$};
\draw[->] ($(s0.west)+(-6mm,0)$) -- (s0.west);
\draw[->] (s0) -- node[left] {$a_0$} (bot);
\draw[->] (s0) -- node[above,sloped] {$a_1,\,\tfrac12$} node[below,sloped] {$R(\tau_{1,0})$} (t10);
\draw[->] (s0) -- node[below,sloped] {$a_2,\,\tfrac12$} node[above,sloped] {$R(\tau_{2,0})$} (t20);
\draw[->] (s0) -- node[above] {$a_1,\,\tfrac12$} (q1);
\draw[->] (s0) -- node[below] {$a_2,\,\tfrac12$} (q1);
\draw[->] (q1) -- node[above,sloped] {$x_1$} (q1a);
\draw[->] (q1) -- node[below,sloped] {$1-x_1$} (q1b);
\draw[->] (q1b) to[bend left=20] node[below,sloped] {$d_{1,1,0},d_{2,1,0}$} (bot);
\draw[->] (q1a) -- node[above,sloped] {$d_{1,1,1}$} (t11);
\draw[->] (q1a) -- node[below] {$d_{2,1,1}$} (q2);
\draw[->] (q2) -- node[above,sloped] {$x_2$} (q2a);
\draw[->] (q2) -- node[below,sloped] {$1-x_2$} (q2b);
\draw[->] (q2b) to[bend left=30] node[below,sloped] {$d_{2,2,0}$} (bot);
\draw[->] (q2a) -- node[above,sloped] {$d_{2,2,1}$} (t21);
\end{tikzpicture}}
\caption{The shared-selector evaluator of \Cref{ex:shared-selector-evaluator}.
Each branch payoff cancels that branch's own probability and depth.
Each absorbing terminal uses the payoff convention of \Cref{app:sqrs-preliminaries}.}
\label{fig:shared-selector-evaluator}
\end{figure}

\begin{example}[A worked evaluator]\label{ex:shared-selector-evaluator}
Take $n=2$, $r=2$, $\gamma=\tfrac12$, $g_1(x)=x_1-\tfrac12$, and $g_2(x)=x_1x_2-\tfrac14$, so that $-g_1(x)=\tfrac12-x_1$ and $-g_2(x)=\tfrac14-x_1x_2$.
\Cref{fig:shared-selector-evaluator} shows the resulting RMDP.
By \Cref{lem:shared-selector-policies}, $\pi_0$ (always $a_0$), $\pi_1$ (choosing $a_1$ at $s_0$ and $d_{1,1,1}$ at $q_1^1$), and $\pi_2$ (choosing $a_2$ at $s_0$, $d_{2,1,1}$ at $q_1^1$, and $d_{2,2,1}$ at $q_2^1$) are all well-defined.

Each constant branch reaches $\tau_{i,0}$ in one step with probability $1/2$.
The monomial branch of $\pi_1$ follows $s_0\to q_1\to q_1^1\to\tau_{1,1}$, with total branch probability $x_1/2$ and depth three.
That of $\pi_2$ follows $s_0\to q_1\to q_1^1\to q_2\to q_2^1\to\tau_{2,1}$, with probability $x_1x_2/2$ and depth five.
The policies share the selector $q_1$ and diverge only through their decoder actions at $q_1^1$.

The monomial branches reach their terminals after depths three and five, while each constant branch
has depth one.
With $w=\tfrac12$ and $\gamma=\tfrac12$, the rule
$R=w^{-1}\gamma^{-d}\times(\text{coefficient})$ applies at each branch's own depth and gives
\[
R(\tau_{1,0})=2,\qquad R(\tau_{1,1})=-16,\qquad
R(\tau_{2,0})=1,\qquad R(\tau_{2,1})=-64.
\]
The first policy consequently has value
\[
V_x^{\pi_1}(s_0)=\gamma\cdot\tfrac12\cdot2
+\gamma^3\cdot\tfrac12 x_1(-16)=\tfrac12-x_1,
\]
and similarly
\[
V_x^{\pi_2}(s_0)=\gamma\cdot\tfrac12\cdot1
+\gamma^5\cdot\tfrac12 x_1x_2(-64)=\tfrac14-x_1x_2,
\]
as required by $-g_1$ and $-g_2$.
Since $\pi_0$ always reaches $\bot$, $V_x^{\pi_0}(s_0)=0$ for every $x$.
Hence
\[
\Delta_\cU(\pi_0,\{\pi_1,\pi_2\})=\sup_{x\in[0,1]^2}\min\bigl(g_1(x),g_2(x)\bigr),
\]
which is maximized at $x=(1,1)$, giving $\min(\tfrac12,\tfrac34)=\tfrac12$.
\end{example}

\begin{proof}[Proof of \Cref{lem:portfolio-comparison-hardness}]
We reduce from strict elementary feasibility (\Cref{lem:strict-elementary}) using the shared-selector evaluator of \Cref{def:shared-selector-evaluator} at discount $\gamma=\gamma_0$.
Its policies $\pi_0,\pi_1,\ldots,\pi_r$ are well-defined by \Cref{lem:shared-selector-policies}.
Every uncertain choice is at a selector $q_j$ and has two successors, and the selectors' choices are independent, so the construction is $(s,a)$-rectangular with two-successor choices.
Ordering every monomial's selectors by increasing variable index makes the full transition graph acyclic apart from its absorbing finals.
There are polynomially many selectors, decoders, and terminals, and every branch depth is polynomial,
so the terminal payoffs have polynomial bit length.
By construction, $\pi_0$ reaches only $\bot$, so $V_x^{\pi_0}(\sinit)=0$.
Policy $\pi_i$ realizes exactly the discounted sum of $-g_i$'s terms, so $V_x^{\pi_i}(\sinit)=-g_i(x)$.
Hence, with $\Pi=\{\pi_1,\ldots,\pi_r\}$,
\[
\Delta_{\cU}(\pi_0,\Pi)
=\sup_{x\in[0,1]^n}\Bigl(0-\max_i\bigl(-g_i(x)\bigr)\Bigr)
=\sup_{x\in[0,1]^n}\min_i g_i(x).
\]
The right side is positive exactly when the strict system is feasible and is at most zero otherwise.
Hence its upper-threshold language at threshold zero is the complement of an $\sETR$-complete problem, proving hardness.
\end{proof}

\begin{proof}[Proof of \Cref{thm:portfolio-comparison-atr}]
Immediate from \Cref{lem:portfolio-comparison-membership,lem:portfolio-comparison-hardness}.
\end{proof}

\section{Proofs of \Cref{thm:portfolio-regret-membership,thm:min-regret-portfolio-eatr-membership}: Membership for Certification and Synthesis}
\label[appendix]{app:portfolio-regret-membership}
Bellman equations encode explicit or existentially quantified portfolio members together with a
universally quantified policy.

\portfolioregretmembership*

\begin{proof}
Apply \Cref{lem:universal-bellman-encoding} to the explicit input portfolio.
Quantifying all memoryless randomized policies is sound because every realized discounted MDP has an optimal deterministic policy, so the encoded assertion is exactly $\Rreg(\Pi)\le t$.
The given-policy problem is the case $r=1$.
\end{proof}

\minregretportfoliomembership*

\begin{proof}
Existentially quantify the $k$ policy tables before the universal encoding of
\Cref{lem:universal-bellman-encoding}.
Their policy-simplex constraints remain outside the universal implication, and unary $k$ keeps the
formula polynomial.
The single-policy problem is the case $k=1$.

The resulting formula asserts that some portfolio of at most $k$ members has regret at most $t$,
whereas $\rho_k$ is an infimum, so the two agree only when that infimum is attained.
It is: the input uncertainty set is a rational polytope, hence compact, and
\Cref{lem:portfolio-attainment} applies.
\end{proof}

\begin{lemma}[Portfolio attainment]\label{lem:portfolio-attainment}
In a finite discounted RMDP with compact $\cU$, the infimum
\[
\rho_r=\inf_{|\Pi|\le r}\Rreg(\Pi)
\]
is attained.
\end{lemma}
\begin{proof}
Represent a portfolio by an ordered $r$-tuple, repeating members when the set is smaller.
This does not change its pointwise maximum, so the two infima agree.
Three facts give the claim.
The space of ordered $r$-tuples of memoryless randomized policies is a finite product of simplexes,
hence compact, and $\cU$ is compact by hypothesis.
For a fixed discount, the Bellman system of a policy has a unique solution that depends continuously
on the policy probabilities and the realization jointly, and $V_u^*(\sinit)$ is the maximum of
finitely many such solutions, one per memoryless deterministic policy, hence also jointly continuous.
Therefore
\[
f(\boldsymbol\pi,u)
=V_u^*(\sinit)-\max_iV_u^{\pi_i}(\sinit)
\]
is jointly continuous, $\boldsymbol\pi\mapsto\max_{u\in\cU}f(\boldsymbol\pi,u)$ is continuous, and it
attains its minimum on the compact policy-tuple space.
\end{proof}
No acyclicity is needed, so the lemma applies to every input of
\Cref{prob:min-portfolio-regret} and to the synthesis RMDP of \Cref{app:portfolio-eatr} alike.

\section{Exact Transfer from Policy Comparison to Robust Regret}
\label[appendix]{app:exact-lift}
Regret uses a realization-dependent best response, whereas comparison fixes both policies.
The lift below makes one fixed reference optimal at every realization: it factors each source choice
through a separate routing state and gives only the reference an additional bonus.
We state it for an allowed-action family $A_{\rm allow}$.
For a finite portfolio, take the union of its supports.
\begin{definition}[Lifted RMDP]\label{def:lifted-rmdp}
Fix an $(s,a)$-rectangular source RMDP $M=\tup{S,A,\cU,R,\sinit,\gamma_0}$, a deterministic reference policy $\pi_0$, and a nonempty allowed-action family $A_{\rm allow}(s)\subseteq A$ at each $s\in S$, not necessarily containing $\pi_0(s)$.
Let $F\subseteq S$ be its absorbing final states, and let $c_f$ be the payoff of $f\in F$.
Fix
\[
V_{\max}=\frac{\max_{s,a\in A_{\rm allow}(s)\cup\{\pi_0(s)\}}|R(s,a)|}{1-\gamma_0},
\qquad
K=2V_{\max}+1,
\qquad
\Lambda=\frac{K}{1-\gamma_0}.
\]
Choose a rational $Z$ of polynomial encoding length such that $\beta Z>\Lambda+V_{\max}$.
The \emph{lifted RMDP} is the tuple $\widehat M=\tup{\widehat S,\widehat A,\widehat\cU,\widehat R,\widehat\sinit,\beta}$, where
\begin{itemize}
\item $\widehat S=\{x_s : s\in S\}\cup\{y_{s,a} : s\in S\setminus F,\ a\in A_{\rm allow}(s)\cup\{\pi_0(s)\}\}\cup\{\bot_{\rm r}\}$.
There is one \emph{tag state} $x_s$ per source state, one \emph{selector state} $y_{s,a}$ for every allowed or reference action at a nonfinal source state, and one ruinous sink.
\item $\widehat A=A\cup\{\mathrm{bonus},\mathrm{route}\}$ adds two fresh action symbols to the source action set.
\item $\widehat\cU$ is $(s,a)$-rectangular.
At a nonfinal tag state, an allowed action $a$ sends $x_s$ to $y_{s,a}$ and the bonus action sends it to $y_{s,\pi_0(s)}$.
The route action at $y_{s,a}$ carries exactly the original uncertainty set $\cU_{(s,a)}$, redirected from successors $s'$ to tag states $x_{s'}$.
At a final tag $x_f$, every allowed action self-loops, as does the bonus action.
\item At a nonfinal tag, $\widehat R(x_s,a)=R(s,a)$ for $a\in A_{\rm allow}(s)$ and $\widehat R(x_s,\mathrm{bonus})=R(s,\pi_0(s))+K$.
Every route action has reward zero.
At a final tag $x_f$, every allowed action has reward $(1-\beta)c_f$, while the bonus action has reward $(1-\beta)(c_f+\Lambda)$.
\item Every action at $\bot_{\rm r}$ has reward $-(1-\beta)Z$ and self-loops.
Every action at $x_f$ not already described also self-loops with reward $-(1-\beta)Z$.
All remaining choices use ruinous-sink completion with this $Z$.
\item $\widehat\sinit=x_{\sinit}$.
\end{itemize}
Write $\widehat\pi_0$ for the policy always choosing the bonus action.
The lift $\widehat\pi$ of any memoryless randomized policy $\pi$ supported on $A_{\rm allow}$ uses the same action distribution as $\pi$ at each tag state and the route action at every selector state.
For a finite portfolio $\Pi=\{\pi_1,\ldots,\pi_r\}$, write $\widehat\Pi=\{\widehat\pi_1,\ldots,\widehat\pi_r\}$.
\end{definition}
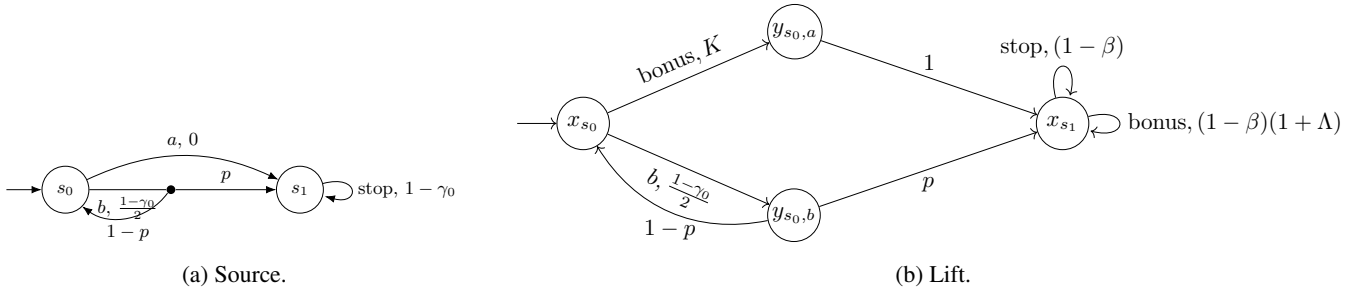
\begin{figure}[t]
\centering
\begin{subfigure}[t]{.34\linewidth}
\centering
\resizebox{\linewidth}{!}{%
\begin{tikzpicture}[
  >=Latex,
  state/.style={circle,draw,minimum size=8mm,inner sep=1pt},
  distribution/.style={circle,fill,inner sep=1.5pt},
  every node/.style={font=\small}]
\node[state] (s0) at (0,0) {$s_0$};
\node[distribution] (d) at (1.8,0) {};
\node[state] (s1) at (4.0,0) {$s_1$};
\draw[->] ($(s0.west)+(-6mm,0)$) -- (s0.west);
\draw[->] (s0) to[bend left=25] node[above] {$a,\,0$} (s1);
\draw[-] (s0) -- node[below] {$b,\,\frac{1-\gamma_0}{2}$} (d);
\draw[->] (d) -- node[above] {$p$} (s1);
\draw[->] (d) to[bend left=45] node[below] {$1-p$} (s0);
\draw[->] (s1) edge[loop right] node[right] {$\mathrm{stop},\,1-\gamma_0$} (s1);
\end{tikzpicture}}
\caption{Source.}\label{fig:source-rmdp-example}
\end{subfigure}\hfill
\begin{subfigure}[t]{.62\linewidth}
\centering
\resizebox{\linewidth}{!}{%
\begin{tikzpicture}[
 state/.style={circle,draw,minimum size=8mm,inner sep=1pt}]
\node[state] (x0) {$x_{s_0}$};
\node[state,right=24mm of x0,yshift=14mm] (ya) {$y_{s_0,a}$};
\node[state,right=24mm of x0,yshift=-14mm] (yb) {$y_{s_0,b}$};
\node[state,right=65mm of x0] (x1) {$x_{s_1}$};
\draw[->] ($(x0.west)+(-6mm,0)$) -- (x0.west);
\draw[->] (x0) -- node[above,sloped] {$\mathrm{bonus},K$} (ya);
\draw[->] (x0) -- node[below,sloped] {$b,\frac{1-\gamma_0}{2}$} (yb);
\draw[->] (ya) -- node[above] {$1$} (x1);
\draw[->] (yb) -- node[below] {$p$} (x1);
\draw[->] (yb) to[bend left=35] node[below] {$1-p$} (x0);
\path[->] (x1) edge[loop above] node[above] {$\mathrm{stop},(1-\beta)$} (x1);
\path[->] (x1) edge[loop right] node[right] {$\mathrm{bonus},(1-\beta)(1+\Lambda)$} (x1);
\end{tikzpicture}}
\caption{Lift.}\label{fig:lifted-rmdp-example}
\end{subfigure}
\caption{The source and its lifted RMDP in \Cref{ex:lifted-rmdp}.}
\end{figure}

\begin{example}[A worked lift]\label{ex:lifted-rmdp}
Take a source RMDP with states $S=\{s_0,s_1\}$, $\sinit=s_0$, actions $A=\{a,b,\mathrm{stop}\}$, and discount $\gamma_0$.
At $s_0$, action $a$ moves deterministically to $s_1$ with reward $R(s_0,a)=0$.
Action $b$ is uncertain, moving to $s_1$ with probability $p$ and back to $s_0$ with probability $1-p$ for any $p\in[0,1]$, with reward $R(s_0,b)=\tfrac{1-\gamma_0}{2}$ regardless of outcome.
At $s_1$, the only action $\mathrm{stop}$ self-loops deterministically with reward $R(s_1,\mathrm{stop})=1-\gamma_0$, so $V(s_1)=1$ (\Cref{fig:source-rmdp-example}).
Fix reference $\pi_0(s_0)=a,\pi_0(s_1)=\mathrm{stop}$ and a single candidate $\pi_1(s_0)=b,\pi_1(s_1)=\mathrm{stop}$, so $\Pi=\{\pi_1\}$ and $A_{\rm allow}(s_0)=\{b\}$, $A_{\rm allow}(s_1)=\{\mathrm{stop}\}$.
Note $A_{\rm allow}(s_0)$ does not contain $\pi_0(s_0)=a$, while $A_{\rm allow}(s_1)$ happens to coincide with $\pi_0(s_1)$.
The Bellman equations give
\[
V_u^{\pi_0}(s_0)=\gamma_0,\qquad
V_u^{\pi_1}(s_0)=\frac{\frac{1-\gamma_0}{2}+\gamma_0p}{(1-\gamma_0)+\gamma_0p},
\]
and the second value is increasing in $p$.
Hence
\[
\Delta_\cU(\pi_0,\{\pi_1\})=\gamma_0-\tfrac12=\tfrac{161}{400}.
\]
\Cref{def:lifted-rmdp} gives tag states $x_{s_0},x_{s_1}$, selectors $y_{s_0,a},y_{s_0,b}$, a ruinous sink, and $\beta=19/20$ (\Cref{fig:lifted-rmdp-example}).
Because $s_1$ is an absorbing final of payoff one, the allowed and bonus actions at $x_{s_1}$ self-loop with rewards $1-\beta$ and $(1-\beta)(1+\Lambda)$.
The selector $y_{s_0,b}$ inherits the source interval for $p$.
The relevant rewards are $R(s_0,a)=0,R(s_0,b)=\tfrac{1-\gamma_0}{2},R(s_1,\mathrm{stop})=1-\gamma_0$, so \Cref{def:lifted-rmdp} gives
\[
V_{\max}=\frac{1-\gamma_0}{1-\gamma_0}=1,
\qquad
K=2V_{\max}+1=3,
\qquad
\Lambda=\frac{K}{1-\gamma_0}=\frac{400}{13},
\]
and $\Rreg(\{\widehat\pi_1\}) =\Lambda+\Delta_\cU(\pi_0,\{\pi_1\}) =\tfrac{400}{13}+\tfrac{161}{400}$.
\end{example}
The margin in $K=2V_{\max}+1$ is what makes the bonus action uniquely optimal at every tag state.
\Cref{lem:exact-regret-lift} formalizes this point.

\begin{lemma}[Exact comparison-to-regret lift]\label{lem:exact-regret-lift}
Consider the construction of \Cref{def:lifted-rmdp} for a deterministic reference $\pi_0$ and an allowed-action family $A_{\rm allow}$.
For every memoryless randomized policy $\pi$ supported on this family and every realization $u$, its lift satisfies
\[
V_{\widehat u}^{\widehat\pi}(x_s)=V_u^\pi(s)
\qquad(s\in S).
\]
This construction preserves $(s,a)$-rectangularity.
It preserves determinism when the lifted policy is deterministic.
It preserves acyclicity of the RMDP and of policy graphs, with every source absorbing final remaining an absorbing final.
Each non-singleton source choice is copied unchanged to one selector state.
In particular, two-successor and two-Dirac restrictions on uncertain choices are preserved.
The bonus policy $\widehat\pi_0$ is pointwise optimal in every realization, and the bonus action is the unique optimal action at every tag state.
Consequently, for every finite portfolio $\Pi$ of policies supported on $A_{\rm allow}$,
\[
\Rreg(\widehat\Pi)=\Lambda+\Delta_{\cU}(\pi_0,\Pi).
\]
\end{lemma}
\begin{proof}[Proof of \Cref{lem:exact-regret-lift}]
We first show that the bonus policy is optimal in the target RMDP at every realization, together with its value.
Fix a realization $u$ of the source RMDP, inducing a realization $\widehat u$ of the target RMDP via the redirected transitions of \Cref{def:lifted-rmdp}.
Define $W(s):=\Lambda+V_u^{\pi_0}(s)$ for every source state $s$, and read it as a candidate value for the tag state $x_s$.
For a nonfinal source state, $\beta^2=\gamma_0$ means that going from $x_s$ through its selector to the next tag state contributes the source discount overall.
We may therefore compare $W$ against this two-step Bellman backup directly.
For the bonus action at $x_s$, the backup is
\[
R(s,\pi_0(s))+K+\gamma_0\sum_{s'}u(s,\pi_0(s),s')W(s').
\]
Substituting $\pi_0$'s own Bellman equation $V_u^{\pi_0}(s)=R(s,\pi_0(s))+\gamma_0\sum_{s'}u(s,\pi_0(s),s')V_u^{\pi_0}(s')$ for $R(s,\pi_0(s))$, this simplifies to
\[
V_u^{\pi_0}(s)+K+\gamma_0\Lambda,
\]
which equals $W(s)=\Lambda+V_u^{\pi_0}(s)$ exactly, since $K=\Lambda(1-\gamma_0)$ by definition of $\Lambda$.
For any allowed action $a\in A_{\rm allow}(s)$, the same two-step backup is
\[
R(s,a)+\gamma_0\sum_{s'}u(s,a,s')W(s')
=\underbrace{R(s,a)+\gamma_0\sum_{s'}u(s,a,s')V_u^{\pi_0}(s')-V_u^{\pi_0}(s)}_{\text{at most }2V_{\max}\text{ in absolute value}}+V_u^{\pi_0}(s)+\gamma_0\Lambda.
\]
The bracketed term is at most $2V_{\max}$ because the reward is at most $(1-\gamma_0)V_{\max}$ and both value terms are at most $V_{\max}$ in absolute value.
Combined with $(1-\gamma_0)\Lambda=K=2V_{\max}+1$, the whole backup is at most $W(s)-1 < W(s)$.
At an absorbing final $f$ of payoff $c_f$, the bonus backup is
\[
(1-\beta)(c_f+\Lambda)+\beta W(f)=W(f),
\]
while every allowed-action backup is
\[
(1-\beta)c_f+\beta W(f)=c_f+\beta\Lambda
=W(f)-(1-\beta)\Lambda<W(f).
\]
Every other action at $x_f$ has value below $W(f)$ because its self-loop payoff is $-Z$, and every undescribed choice at a nonfinal state has backup $-\beta Z<W(s)$ by the choice of $Z$.
At selector states, the route action is likewise better than entering the ruinous sink.
Thus $W$ satisfies the Bellman optimality equation at every tag state, with the bonus action as the unique maximizer.
Since the discounted Bellman optimality equation has a unique solution, $W$ is exactly the optimal value function on tag states and the bonus policy is pointwise optimal.
In particular,
\[
V_{\widehat u}^{\widehat\pi_0}(x_s)=\Lambda+V_u^{\pi_0}(s)
\]
for every source state $s$.

For any policy $\pi$ supported on $A_{\rm allow}$, its lift $\widehat\pi$ replays $\pi$'s Bellman equation over two lifted steps at nonfinal states and uses the payoff-preserving self-loop at final states.
Direct substitution using $\beta^2=\gamma_0$ therefore gives
\[
V_{\widehat u}^{\widehat\pi}(x_s)=V_u^{\pi}(s)
\]
for every $s$.
No optimality claim is needed here, only that $\widehat\pi$'s value matches $\pi$'s under the same discount-matching substitution used above.

Combining the two identities at $\sinit$,
\[
\begin{aligned}
\Rreg(\widehat\Pi)
&=\sup_{\widehat u}\Bigl(V_{\widehat u}^*(\widehat\sinit)-\max_{1\le i\le r}V_{\widehat u}^{\widehat\pi_i}(\widehat\sinit)\Bigr)\\
&=\sup_u\Bigl(\Lambda+V_u^{\pi_0}(\sinit)-\max_{1\le i\le r}V_u^{\pi_i}(\sinit)\Bigr)\\
&=\Lambda+\Delta_{\cU}(\pi_0,\Pi),
\end{aligned}
\]
using pointwise optimality of the bonus policy, so $V_{\widehat u}^*(\widehat\sinit)=V_{\widehat u}^{\widehat\pi_0}(\widehat\sinit)$.
Every uncertain source choice occurs at one selector state, so shared choices are not copied.
Products of source choice uncertainty sets remain products, which preserves $(s,a)$-rectangularity.
Replacing each nonfinal source edge by its two-edge tag-selector path preserves a topological order, while each final tag and the ruinous sink remains absorbing.
For policy graphs the reachability qualifier in the definition of a used choice is load-bearing rather than incidental: route actions are installed at every selector state, including selectors for source actions the lifted policy never chooses, and those states are unreachable under it.
Reading $G_{\widehat\pi}$ as containing every transition the policy gives positive probability, instead of only its used choices, would therefore add edges that no run traverses and break the claim.
This proves the stated acyclicity preservation for the full RMDP and for its policy graphs.
\end{proof}
\section{Proofs of \Cref{thm:fixed-regret-hardness,thm:portfolio-regret-atr,thm:fixed-regret-forallr}: Transferring Comparison Bounds to Certification}
\subsection{Rectangular Given-Policy Hardness}
\label[appendix]{app:fixed-regret-hardness}
The exact lift transfers the Boolean and square-root-sum comparison bounds.

\fixedregrethardness*

\begin{proof}
For the $\scoNP$ lower bound, start with $\Cmp(\varphi)$ and apply $\Lift(\Cmp(\varphi),\pi_C)$ with singleton portfolio $\{\pi_K\}$.
The source consists of deterministic policies $\pi_C,\pi_K$ in an acyclic $(s,a)$-rectangular RMDP with independent two-Dirac uncertain choices, one certain uniform splitter, and discount $\gamma_0$, and it satisfies
\[
\varphi\in\UNSAT
\quad\Longleftrightarrow\quad
\Delta_{\cU}(\pi_C,\pi_K)\le 2-\frac1{2n}.
\]
For the lifted candidate $\widehat\pi_K$ the lemma gives
\[
\Rreg(\widehat\pi_K)
=\Lambda+\Delta_{\cU}(\pi_C,\pi_K).
\]
Thus, with the rational target threshold $\widehat t=\Lambda+2-\frac1{2n}$, the lifted regret instance is a yes-instance exactly when $\varphi\in\UNSAT$.
The structural restrictions follow from \Cref{lem:exact-regret-lift}, proving $\scoNP$-hardness.

For the $\scoSRS$ lower bound, use the deterministic policies $\pi_u,\pi_v$ and threshold $t$ produced by \Cref{thm:comparison-cosqrs}.
That construction uses the same discount $\gamma_0$, is $(s,a)$-rectangular, and has two-Dirac uncertain choices.
Apply $\Lift(N,\pi_u)$ with singleton candidate $\{\pi_v\}$.
Then
\[
\Delta_{\cU}(\pi_u,\pi_v)\le t
\quad\Longleftrightarrow\quad
\Rreg(\widehat\pi_v)\le\Lambda+t.
\]
Hence given-policy robust regret is $\scoSRS$-hard.
Together the two reductions prove the theorem.
\end{proof}

\subsection{Rectangular Portfolio Certification}
\label[appendix]{app:portfolio-atr}
The family form of the exact lift transfers the rectangular portfolio-comparison bound.

\portfolioregretatr*

\begin{proof}
The evaluator above was instantiated at $\gamma_0$, so apply \Cref{lem:exact-regret-lift} to its reference policy and explicit portfolio.
The exact identity
\[
\Rreg(\widehat\Pi)=\Lambda+\Delta_{\cU}(\pi_0,\Pi)
\]
translates the zero comparison threshold to the rational regret threshold $\Lambda$.
The structural conclusions are those of \Cref{lem:exact-regret-lift}.
Therefore portfolio-regret certification is $\sUTR$-hard under all restrictions stated in the theorem.
\end{proof}

\subsection{General-Polytope Regret Certification}
\label[appendix]{app:regret-certification-real}
Membership is the singleton case of \Cref{thm:portfolio-regret-membership}.
For hardness, a choice between the polynomial evaluator and zero makes regret compute the polynomial's positive part.

\fixedregretforallr*

The membership half of \Cref{thm:fixed-regret-forallr} is the singleton-family case of \Cref{lem:universal-bellman-encoding}.
\paragraph{Polynomial-evaluation hardness.}
\begin{definition}[Positive-part gadget]\label{def:positive-part-gadget}
Given $M_f$ from \Cref{def:polynomial-evaluator}, add a fresh initial state $s^+$ with two reward-zero
actions: $a_f$ enters $M_f$ at $s_f$, and $a_\bot$ enters a zero-reward sink.
Let $\pi_\bot$ choose $a_\bot$ at $s^+$.
All choices left undescribed after enlarging the global action set use the same ruinous-sink
completion as $M_f$.
\end{definition}

\begin{lemma}[Positive-part regret]\label{lem:positive-part-gadget}
The fixed policy in \Cref{def:positive-part-gadget} satisfies
\[
\Rreg(\pi_\bot)=\sup_{p\in[0,1]^m}\max\{\gamma_0 f(p),0\}.
\]
Consequently $\Rreg(\pi_\bot)\le0$ exactly when $\forall p\in[0,1]^m:f(p)\le0$.
\end{lemma}
\begin{proof}
At realization $p$, the two initial actions have values $\gamma_0f(p)$ and zero by \Cref{lem:polynomial-evaluator}.
Their positive part is precisely the shortfall of $\pi_\bot$.
Taking the supremum proves the identity.
Regret is nonnegative, so threshold zero is tight.
\end{proof}

\Cref{ex:positive-part-gadget} illustrates the zero-threshold case.
\begin{example}\label{ex:positive-part-gadget}
For $f(p)=p-p^2$, the unique compliant policy inside $M_f$ has value $p-p^2$ and
\[
\Rreg(\pi_\bot)=\gamma_0\max_{p\in[0,1]}(p-p^2)=\gamma_0/4>0.
\]
Thus this instance is a no-instance at threshold zero.
\end{example}

\begin{proof}[Proof of \Cref{thm:fixed-regret-forallr}]
Apply \Cref{lem:positive-part-gadget} to the bounded universal-polynomial instances used in \Cref{lem:comparison-forallr-hardness}.
Together with the membership encoding above, this proves the theorem.
\end{proof}

\section{Proof of \Cref{thm:min-regret-combinatorial}: Combinatorial Minimal-Regret Hardness}
\label[appendix]{app:min-regret-boolean}

We prove the Boolean lower bounds for minimal robust regret.

\minregretcombinatorial*

Both reductions first constrain the candidate to allowed actions.
\Cref{lem:policy-restriction} makes disallowed actions ruinous and thereby transfers this restricted minimum to unrestricted synthesis within an arbitrarily small $\varepsilon$.

\subsection{Policy Restriction}
Fix an RMDP $N$, nonempty allowed action sets, and rational $\varepsilon>0$.
The construction below turns restricted synthesis into ordinary unrestricted synthesis: every disallowed choice is handed to the adversary as an extra option that can collapse straight into a state so bad that no min-regret policy will ever choose to use it.
For $\mathcal P\in\{\polRand,\polDet\}$, write $\mathcal P_{\mathrm{allow}}$ for the policies in $\mathcal P$ supported on the allowed actions.

\begin{definition}[$\varepsilon$-restricted RMDP]\label{def:eps-restricted-rmdp}
Let $N=\tup{S,A,\cU,R,\sinit,\gamma}$ be an RMDP and let $A_{\mathrm{allow}}(s)\subseteq A$ be a nonempty set of allowed actions for every $s\in S$.
Suppose every disallowed choice has a fixed transition distribution, i.e., $\cU_{(s,a)}=\{p_{s,a}\}$ for every $s\in S$ and $a\in A\setminus A_{\mathrm{allow}}(s)$.
Choices supplied by ruinous-sink completion satisfy this hypothesis.
The construction proceeds in two steps.

First it \emph{unfolds} the absorbing finals of $N$, so that no state both self-loops and can leave.
Add a fresh zero-reward absorbing sink $\bot_0$, every action of which has reward zero and row $\{\delta_{\bot_0}\}$, and let every action be allowed there.
At every absorbing final $f$ of $N$, replace the reward $R(f,a)$ by $R(f,a)/(1-\gamma)$ and the row by $\{\delta_{\bot_0}\}$, keeping $A_{\mathrm{allow}}(f)$ as it is.
An action $a$ at $f$ had value $R(f,a)/(1-\gamma)$ before the step, since it self-looped forever, and it collects exactly that reward once and then nothing after the step.
The unfolding therefore leaves the value of every choice at $f$ unchanged.
Every state therefore keeps its value under every stationary randomized policy and every realization, so the optimal value and every robust regret are unchanged as well.
The unfolding adds no uncertain choice, preserves rectangularity and acyclicity, and leaves $\bot_0$ as the only absorbing final.
It raises the largest reward magnitude to at most the largest value magnitude of $N$, hence multiplies $V_{\max}$ below by at most $1/(1-\gamma)$ and keeps every constant of polynomial encoding length.
Write $N$ for the unfolded RMDP from here on.

Second, let $\varepsilon>0$ be rational, and set
\[
V_{\max}=\max\left\{1,\frac{\max_{s,a}|R(s,a)|}{1-\gamma}\right\},
\qquad
Z_\varepsilon=\frac{(2-\gamma)V_{\max}+4V_{\max}^2/\varepsilon}{\gamma}.
\]
The \emph{$\varepsilon$-restricted RMDP} is the tuple $N_\varepsilon=\tup{S_\varepsilon,A_\varepsilon,{\cU}_\varepsilon,R_\varepsilon,{\sinit}_\varepsilon,\gamma_\varepsilon}$, where:
\begin{itemize}
\item $S_\varepsilon=S\cup\{\bot_{\rm r}\}$, adding one fresh ruinous sink $\bot_{\rm r}$;
\item $A_\varepsilon=A$, $\gamma_\varepsilon=\gamma$, and ${\sinit}_\varepsilon=\sinit$;
\item $R_\varepsilon(s,a)=R(s,a)$ for every $s\in S$ and $a\in A$, while every action at $\bot_{\rm r}$ has reward $-(1-\gamma)Z_\varepsilon$;
\item $\cU_\varepsilon(s,a)=\cU(s,a)$ for every allowed choice, $a\in A_{\mathrm{allow}}(s)$;
\item $\cU_\varepsilon(s,a)=\bigl\{\theta\,p_{s,a}+(1-\theta)\delta_{\bot_{\rm r}}:\theta\in[0,1]\bigr\}$ for every disallowed choice, $a\in A\setminus A_{\mathrm{allow}}(s)$;
\item every action at $\bot_{\rm r}$ has row $\{\delta_{\bot_{\rm r}}\}$.
\end{itemize}
\end{definition}

The self-loop reward at $\bot_{\rm r}$ gives it value $-Z_\varepsilon$ under every realization:
$V(\bot_{\rm r})=-(1-\gamma)Z_\varepsilon+\gamma V(\bot_{\rm r})$ forces
$V(\bot_{\rm r})=-Z_\varepsilon$.
At the endpoint $\theta=1$, $\cU_\varepsilon(s,a)$ reduces to $\{p_{s,a}\}$, exactly $N$'s original choice.
At $\theta=0$, it collapses to $\{\delta_{\bot_{\rm r}}\}$, routing straight into the sink.
The unfolding is what makes the interpolation safe.
When $N$ is acyclic its only self-loops sit at absorbing finals, so after the unfolding the only self-looping state is $\bot_0$, where nothing is disallowed.
Every interpolated row therefore lies between two distributions that both leave their state, and no state acquires a self-loop it did not already have.

\begin{lemma}[Policy restriction]\label{lem:policy-restriction}
For either $\mathcal P=\polRand$ or $\mathcal P=\polDet$,
\[
\begin{aligned}
\inf_{\rho\in\mathcal P}\Rreg^{N_\varepsilon}(\rho)
&\le \inf_{\pi\in\mathcal P_{\mathrm{allow}}}\Rreg^N(\pi),\\
\inf_{\rho\in\mathcal P}\Rreg^{N_\varepsilon}(\rho)
&\ge \inf_{\pi\in\mathcal P_{\mathrm{allow}}}\Rreg^N(\pi)-\varepsilon.
\end{aligned}
\]
\end{lemma}
\Cref{ex:policy-restriction} illustrates the transformation on two states.
\begin{example}[A two-state restriction]\label{ex:policy-restriction}
Let the source RMDP have states $s$ and $g$, with $s$ initial and $g$ an absorbing zero-payoff terminal.
At $s$, both the allowed action $a$ and the disallowed action $b$ move deterministically to $g$, and $A_{\mathrm{allow}}(s)=\{a\}$.
In $\Restrict(N,A_{\mathrm{allow}},\varepsilon)$, the unfolding first sends $g$ to the fresh zero sink $\bot_0$ in one step, leaving its value at zero.
Action $a$ is then unchanged, while $b$ reaches $g$ with
probability $\theta$ and $\bot_{\rm r}$ with probability $1-\theta$, as shown in
\Cref{fig:policy-restriction-example}.
The allowed policy reproduces the restricted source optimum, and \Cref{lem:policy-restriction} places the unrestricted optimum in the interval from that value minus $\varepsilon$ to that value.
Thus the transformation introduces at most the promised $\varepsilon$ gap, which is zero in this symmetric example because choosing $a$ has regret zero.
\end{example}

\begin{figure}[t]
\centering
\begin{subfigure}[t]{.42\linewidth}
\centering
\begin{tikzpicture}[
  >=Latex,
  state/.style={circle,draw,minimum size=8mm,inner sep=1pt},
  every node/.style={font=\small}]
\node[state] (s) at (0,0) {$s$};
\node[state] (g) at (3.4,0) {$g$};
\draw[->] ($(s.west)+(-6mm,0)$) -- (s.west);
\draw[->] (s) to[bend left=22] node[above] {$a$ (allowed)} (g);
\draw[->] (s) to[bend right=22] node[below] {$b$ (disallowed)} (g);
\path[->] (g) edge[loop right] node[right] {$0$} (g);
\end{tikzpicture}
\caption{Before.}
\end{subfigure}\hfill
\begin{subfigure}[t]{.52\linewidth}
\centering
\begin{tikzpicture}[
  >=Latex,
  state/.style={circle,draw,minimum size=8mm,inner sep=1pt},
  distribution/.style={circle,fill,inner sep=1.5pt},
  every node/.style={font=\small}]
\node[state] (s) at (0,0) {$s$};
\node[distribution] (d) at (2.0,-1.0) {};
\node[state] (g) at (4.2,0.7) {$g$};
\node[state] (z) at (6.4,0.7) {$\bot_0$};
\node[state] (bot) at (4.2,-1.7) {$\bot_{\rm r}$};
\draw[->] ($(s.west)+(-6mm,0)$) -- (s.west);
\draw[->] (s) -- node[above,sloped] {$a$ (allowed)} (g);
\draw[-] (s) -- node[below,sloped] {$b$ (disallowed)} (d);
\draw[->] (d) -- node[above,sloped] {$\theta$} (g);
\draw[->] (d) -- node[below,sloped] {$1-\theta$} (bot);
\draw[->] (g) -- node[above] {$0$} (z);
\path[->] (z) edge[loop right] node[right] {$0$} (z);
\path[->] (bot) edge[loop right] node[right] {$-Z_\varepsilon$} (bot);
\end{tikzpicture}
\caption{After.}
\end{subfigure}
\caption{Policy restriction before and after the transformation.
The allowed action remains unchanged, while nature can route the disallowed action to the ruinous sink.}
\label{fig:policy-restriction-example}
\end{figure}
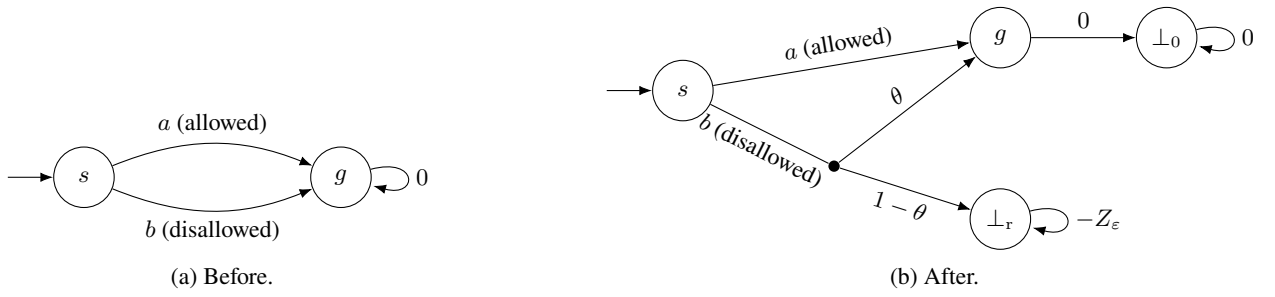

\begin{proof}[Proof of \Cref{lem:policy-restriction}]
The unfolding step changes no state's value under any stationary policy or realization, so it changes neither $\Rreg^N$ nor the set of allowed policies, and $N$ below denotes the unfolded RMDP.
Fix a source realization $u$ and let $W=V^*_{u,\theta\equiv1}$ in $N_\varepsilon$.
At the all-normal realization, the original states reproduce $N_u$,
$V(\bot_{\rm r})=-Z_\varepsilon$, and no sink-routed action improves on an original action.
Thus $W(s)=V_u^*(s)$ for $s\in S$ and $W(\bot_{\rm r})=-Z_\varepsilon$.
For a disallowed choice and arbitrary $\theta\in[0,1]$, its backup at $W$ is
\[
R(s,a)+\gamma\left(\theta\sum_{s'}p_{s,a}(s')W(s')+(1-\theta)(-Z_\varepsilon)\right)
\le
R(s,a)+\gamma\sum_{s'}p_{s,a}(s')W(s'),
\]
because $-Z_\varepsilon\le -V_{\max}\le\min_{s\in S}W(s)$.
Allowed backups are unchanged, so the optimal Bellman operator satisfies $T_{u,\theta}W\le W$.
Monotonicity and contraction give $V^*_{u,\theta}\le W$.
An allowed policy never uses a modified choice, hence its value is independent of $\theta$ and agrees with its value in $N_u$.
Its regret is therefore maximized at $\theta\equiv1$, and every allowed policy has the same robust regret in $N$ and $N_\varepsilon$.
This proves the upper inequality.

Fix an arbitrary policy $\rho$ and condition it on allowed actions at every state, using an arbitrary allowed action if the conditioning probability is zero.
Call the result $\bar\rho$.
The conditioning is total because $A_{\mathrm{allow}}(s)$ is nonempty at every state.
It preserves membership in $\polRand$, and it preserves determinism when $\mathcal P=\polDet$.

The weight below depends on the realization, so fix a source realization $u\in\cU$ and argue pointwise in it, taking the supremum over $u$ only at the end.
Couple the two trajectories by reusing the same transition draw and, while an allowed action is
selected, the same normalized action draw.
The two trajectories then agree until $\rho$ first chooses a nonallowed action at time $\tau$.
Put
\[
d=\mathbb E[\gamma^\tau\mathbf 1_{\{\tau<\infty\}}].
\]
At the joint realization $(u,\theta\equiv1)$, every action value, including its immediate reward and discounted continuation, lies in $[-V_{\max},V_{\max}]$.
Conditional on the first differing choice, $\rho$ therefore gains at most $2V_{\max}$ over $\bar\rho$, so
\[
V_{u,\theta\equiv1}^\rho(\sinit)\le V_u^{\bar\rho}(\sinit)+2V_{\max}d,
\qquad
\Rreg^{N_\varepsilon}(\rho)
\ge V_u^*(\sinit)-V_u^{\bar\rho}(\sinit)-2V_{\max}d.
\]
At the joint realization extending the same $u$ with $\theta=0$ at every nonallowed choice, the first such choice incurs loss at least
\[
\gamma Z_\varepsilon-(2-\gamma)V_{\max}
=\frac{4V_{\max}^2}{\varepsilon},
\]
and therefore
\[
\Rreg^{N_\varepsilon}(\rho)
\ge\frac{4V_{\max}^2}{\varepsilon}d.
\]
Both endpoints are valid joint realizations extending this same $u$, because the modified choices are independent.
If $2V_{\max}d\le\varepsilon$, the first bound loses at most $\varepsilon$.
Otherwise the second exceeds $2V_{\max}$ and hence dominates every possible regret in $N$.
Either way
\[
\Rreg^{N_\varepsilon}(\rho)
\ge V_u^*(\sinit)-V_u^{\bar\rho}(\sinit)-\varepsilon.
\]
The left side does not depend on $u$, so taking the supremum over $u\in\cU$ gives
\[
\Rreg^{N_\varepsilon}(\rho)
\ge \Rreg^N(\bar\rho)-\varepsilon
\ge \inf_{\pi\in\mathcal P_{\mathrm{allow}}}\Rreg^N(\pi)-\varepsilon.
\]
Taking the infimum over $\rho$ proves the lower inequality.
For $B$-bit rational inputs $V_{\max}$ and $\varepsilon$, the formula above gives $Z_\varepsilon$ polynomial bit length in $B$ and the model size.
Each modified choice is independent, so rectangularity is preserved, and a Dirac $p_{s,a}$ yields a two-Dirac segment.
Acyclicity is preserved as well: after the unfolding, $\bot_0$ is the only self-looping state of $N$ and nothing is disallowed there, so every interpolated row leaves its state, and the two added sinks are absorbing finals.
\end{proof}

\subsection{$\scoNP$-Hardness}\label[appendix]{app:min-regret-conp-part}
Apply the exact lift (\Cref{def:lifted-rmdp}) to the comparison instance of \Cref{thm:comparison-conp}, which builds from $\varphi$ a Boolean comparison RMDP with a clause policy $\pi_C$ and a consistency policy $\pi_K$ such that
\[
\Delta_\cU(\pi_C,\pi_K)\le 2-\frac1n \text{ if }\varphi\in\mathrm{UNSAT},
\qquad
\Delta_\cU(\pi_C,\pi_K)=2 \text{ if }\varphi\in\mathrm{SAT}.
\]
Take $\pi_0=\pi_C$ as the lift's reference policy.
No optimality property of $\pi_C$ is needed, only that it is the fixed policy the lift compares against.
Restrict the candidates admitted by the lift to $\pi_K$ alone: at every tag state $x_s$, set $A_{\mathrm{allow}}(x_s)=\{\pi_K(s)\}$, while every selector state has only its allowed $\mathrm{route}$ action.
Recall from \Cref{def:lifted-rmdp} that a tag state's available actions are $A_{\mathrm{allow}}(s)\cup\{\mathrm{bonus}\}$, never the reference's own action directly.
The action $\pi_C(s)$ is reachable only through $\mathrm{bonus}$ and $y_{s,\pi_C(s)}$.
Thus policy restriction keeps a candidate from taking the bonus route meant only for the reference.

Because $A_{\mathrm{allow}}$ is a singleton at every tag state and selector states offer no choice at all, $\polRand_{\mathrm{allow}}$ contains exactly one policy, the lift $\widehat\pi_K$ of $\pi_K$.
So there is no actual minimization left to do: writing $N$ for this lifted RMDP,
\[
\inf_{\pi\in\polRand_{\mathrm{allow}}}\Rreg^N(\pi)=\Rreg^N(\widehat\pi_K)
\overset{\Cref{lem:exact-regret-lift}}{=}\Lambda+\Delta_\cU(\pi_C,\pi_K),
\]
which is exactly \Cref{thm:comparison-conp}'s gap plus the lift's constant $\Lambda$.
The certification gap transfers verbatim, with no separate minimization argument needed.

The complete chained construction is
\[
N_\varepsilon=\Restrict\bigl(\Lift(\Cmp(\varphi),\pi_C),A_{\mathrm{allow}},\tfrac1{4n}\bigr).
\]
From here on write $\Rreg(\pi)$ for $\Rreg^{N_\varepsilon}(\pi)$, matching \Cref{thm:min-regret-combinatorial}'s notation.
\Cref{lem:policy-restriction} with $\mathcal P=\polRand$ then converts the two cases above into
\[
\varphi\in\mathrm{UNSAT}\Longrightarrow
\inf_\pi\Rreg(\pi)\le\inf_{\pi\in\polRand_{\mathrm{allow}}}\Rreg^N(\pi)\le\Lambda+2-\frac1n,
\]
\[
\varphi\in\mathrm{SAT}\Longrightarrow
\inf_\pi\Rreg(\pi)\ge\inf_{\pi\in\polRand_{\mathrm{allow}}}\Rreg^N(\pi)-\frac1{4n}
=\Lambda+2-\frac1{4n},
\]
using the lemma's upper inequality (no slack) in the first line and its lower inequality (losing $\varepsilon=\frac1{4n}$) in the second.
Since
\[
\Lambda+2-\frac1n
<
\Lambda+2-\frac1{2n}
<
\Lambda+2-\frac1{4n},
\]
the threshold $\Lambda+2-\frac1{2n}$ gives the $\scoNP$-hard part of \Cref{thm:min-regret-combinatorial}.

\subsection{$\sNP$-Hardness}\label[appendix]{app:min-regret-np-part}
The $\sNP$ reduction swaps the roles of the two policies.
The fixed reference now looks for a falsified clause, while the synthesized policy names a valuation whose local occurrences nature can audit.
\begin{definition}[Falsification and valuation verifier RMDP]
\label{def:falsification-valuation-verifiers}
For a 3-CNF formula $\varphi=\bigwedge_{i=1}^m C_i$ over variables $x_1,\ldots,x_n$, the \emph{verifier RMDP} is
\[
N_\varphi=\tup{S_\varphi,A_\varphi,\cU_\varphi,R_\varphi,\sinit,\gamma_0},
\]
with the following components.
\begin{itemize}
\item $S_\varphi=\{\sinit\}\cup S_{\mathrm{scan}}\cup S_{\mathrm{aud}}\cup S_{\mathrm{sel}}\cup\{\bot_{\rm r}\}$, consisting of four verifier groups and a ruinous sink:
\begin{itemize}
\item $S_{\mathrm{scan}}$, visited only by the falsification (scanner) policy: $S_F=\{f_{i,j}:i\in[m],\,j\in\{1,2,3\}\}$, for clause $C_i$ and literal $j$, together with $\mathrm{acc}_F$, $\mathrm{rej}_F$, and scanner-side padding states, fixed once the transitions below are defined.
\item $S_{\mathrm{aud}}$, visited only by the valuation (audit) policy: $\{v_x : x\text{ occurs in }\varphi\}$ and, for every $x$, $S_{x,T}=\{s_{x,T,\ell}\}_{\ell=1}^{|O_x|}$, $S_{x,F}=\{s_{x,F,\ell}\}_{\ell=1}^{|O_x|}$, where $x$ is assigned $b$ and occurrence $\ell$ of $O_x$ is audited (occurrences of $x$, fixed order), together with $\mathrm{acc}_K$, $\mathrm{rej}_K$, and audit-side padding states.
\item $S_{\mathrm{sel}}=\{q_{o,t},q_{o,t}^0,q_{o,t}^1 : o\text{ an occurrence},\,t\in\{T,F\}\}$,
the local-bit selectors of \Cref{def:local-bit-verifier-primitives}, with $T,F$ in place of $1,0$.
These are the only states either policy's transitions can lead into from the other's territory.
\end{itemize}
Let $\operatorname{val}(o)=T$ for a positive occurrence and $\operatorname{val}(o)=F$ for a negative occurrence, and write $\neg\operatorname{val}(o)$ for the other Boolean value.
\item $\sinit$ has actions $a_F^{\mathrm{init}}$, leading into $S_{\mathrm{scan}}$ at $f_{1,1}$, and $a_K^{\mathrm{init}}$, leading into $S_{\mathrm{aud}}$ at $v_{x_1}$.

Within $S_{\mathrm{scan}}$: every $q_{o,t}\in S_{\mathrm{sel}}$ has a single action, whose outcome $q_{o,t}^0$ or $q_{o,t}^1$ is governed by $\cU_\varphi$ below.
Each selector outcome offers a scanner action $a_F$ and an audit action $a_K$ with the role-specific continuations described next.
For $o=(i,j)$, action $a_F$ at $f_{i,j}$ uses the pair test of
\Cref{def:local-bit-verifier-primitives}, and its locally false outcome advances to $f_{i,j+1}$ or
$\mathrm{acc}_F$ when $j=3$.
All other outcomes abandon the current clause, continuing to $f_{i+1,1}$, or to $\mathrm{rej}_F$ if $i=m$.

Within $S_{\mathrm{aud}}$: every $v_x$ offers actions $T,F$ leading to $s_{x,T,1},s_{x,F,1}$.
The controls $s_{x,b,\ell}$ implement the audit chain of
\Cref{def:local-bit-verifier-primitives}; after its last occurrence, the chain continues to the next
variable or to $\mathrm{acc}_K$.

Fix $H$ at least as long as the longest path in either graph just described.
The scanner-side and audit-side padding states each consist of fresh states inserted along every shorter path in their own graph so it also reaches its terminal after exactly $H$ transitions, each with a single reward-free action continuing toward that terminal.
\item $\cU_\varphi$ is the product of the local-bit segments in
\Cref{eq:boolean-selector-choice}, with $T,F$ in place of $1,0$.
All remaining described choices are singletons.
\item $\mathrm{acc}_F$ is the terminal with payoff $\gamma_0^{-H}$ and $\mathrm{acc}_K$ the terminal with payoff $-\gamma_0^{-H}$, while every other described reward, including at $\mathrm{rej}_F$ and $\mathrm{rej}_K$, is zero.
\item All remaining choices use ruinous-sink completion.
\item The initial state is $\sinit$ and the discount is the appendix-wide $\gamma_0$.
\end{itemize}
Variables with no occurrence are removed.
\end{definition}
Since $\cU_\varphi$ is a product of two-Dirac segments, its vertices are exactly the Boolean assignments to the local bits $\{b_{o,t}\}$.
\Cref{fig:falsification-valuation-example,ex:falsification-valuation-verifiers} illustrate the verifier RMDP and the two policy paths through its shared selectors.

\begin{definition}[Falsification and valuation policies]
\label{def:falsification-valuation-policies}
For the verifier RMDP of \Cref{def:falsification-valuation-verifiers}, define
\[
\begin{aligned}
\pi_F(\sinit)&=a_F^{\mathrm{init}},
&\pi_F(s)&=a_F(s) &&(s\in S_F),\\
\pi_F(q_{o,t}^b)&=a_F(q_{o,t}^b),\\
K_\sigma(\sinit)&=a_K^{\mathrm{init}},
&K_\sigma(v_x)(b)&=\sigma_x(b),\\
K_\sigma(s)&=a_{x,b}(s) &&(s\in S_{x,b}),\\
K_\sigma(q_{o,t}^b)&=a_K(q_{o,t}^b).
\end{aligned}
\]
On decision states outside their reachable verifier graphs, fix default actions whose otherwise undescribed choices use ruinous-sink completion.
For a valuation $\alpha$, the deterministic policy $K_\alpha$ is obtained from $\sigma_x(T)=\alpha(x)$ and $\sigma_x(F)=1-\alpha(x)$.
At a Boolean realization, a vertex of $\cU_\varphi$, every selector sends its unique action to a fixed outcome, so a deterministic policy follows a single path.
We say the policy \emph{accepts} that realization when this path reaches its accepting terminal, $\mathrm{acc}_F$ for $\pi_F$ or $\mathrm{acc}_K$ for $K_\alpha$.
Thus $\pi_F$ accepts exactly when some clause is locally false, while $K_\sigma$ audits the value selected at each $v_x$.
\end{definition}
The common depth and terminal payoffs satisfy \Cref{lem:acceptance-difference-verifier}, so the
comparison value is the sum of the two acceptance probabilities.

\begin{figure*}[t]
\centering
\begin{tikzpicture}[
  x=11mm,y=13mm,
  state/.style={circle,draw,minimum size=8mm,inner sep=0pt,font=\small},
  sel/.style={draw,rounded corners,minimum width=16mm,minimum height=8mm,inner sep=1pt,font=\footnotesize},
  shared/.style={sel,very thick},
  terminal/.style={rectangle,draw,align=center,minimum height=7mm,inner sep=2pt,font=\footnotesize},
  reads/.style={->,densely dotted,gray},
  font=\small
]
\node at (-1.2,1) {$S_{\mathrm{scan}}$};
\node at (-1.2,0) {$S_{\mathrm{sel}}$};
\node at (-1.2,-1) {$S_{\mathrm{aud}}$};
\node[state] (sinit) at (0,0) {$\sinit$};
\draw[->] (-0.5,0) -- (sinit.west);
\node[state] (f11) at (1.5,1) {$f_{1,1}$};
\node[state] (f12) at (5.0,1) {$f_{1,2}$};
\node[state] (f13) at (8.5,1) {$f_{1,3}$};
\node[terminal] (accF) at (12.2,1) {$\mathrm{acc}_F$\\$\gamma_0^{-H}$};
\node[state] (v1) at (1.5,-1) {$v_{x_1}$};
\node[state] (v2) at (5.0,-1) {$v_{x_2}$};
\node[state] (v3) at (8.5,-1) {$v_{x_3}$};
\node[terminal] (accK) at (12.2,-1) {$\mathrm{acc}_K$\\$-\gamma_0^{-H}$};
\node[sel] (q11t) at (1.5,0) {$q_{(1,1),T}$};
\node[shared] (q11f) at (3.2,0) {$q_{(1,1),F}$};
\node[sel] (q12f) at (5.0,0) {$q_{(1,2),F}$};
\node[sel] (q12t) at (6.7,0) {$q_{(1,2),T}$};
\node[sel] (q13t) at (8.5,0) {$q_{(1,3),T}$};
\node[sel] (q13f) at (10.2,0) {$q_{(1,3),F}$};
\draw[->] (sinit) -- node[above left,font=\footnotesize,inner sep=1pt] {$a_F^{\mathrm{init}}$} (f11);
\draw[->] (sinit) -- node[below left,font=\footnotesize,inner sep=1pt] {$a_K^{\mathrm{init}}$} (v1);
\draw[->] (f11) -- (q11t);
\draw[->] (q11t) -- node[above,font=\footnotesize] {$0$} (q11f);
\draw[->] (q11f) to[bend right=18] node[above,font=\footnotesize] {$1$} (f12);
\draw[->] (f12) -- (q12f);
\draw[->] (q12f) -- node[above,font=\footnotesize] {$0$} (q12t);
\draw[->] (q12t) to[bend right=18] node[above,font=\footnotesize] {$1$} (f13);
\draw[->] (f13) -- (q13t);
\draw[->] (q13t) -- node[above,font=\footnotesize] {$0$} (q13f);
\draw[->] (q13f) to[bend right=18] node[above,font=\footnotesize] {$1$} (accF);
\draw[->] (v1) -- (v2);
\draw[->] (v2) -- (v3);
\draw[->] (v3) -- (accK);
\draw[reads] (v1) -- node[below right,font=\footnotesize,pos=.45,inner sep=1.5pt] {$x_1{=}F$} (q11f);
\end{tikzpicture}
\caption{The verifier RMDP of \Cref{def:falsification-valuation-verifiers} for
$\varphi=(x_1\vee\lnot x_2\vee x_3)\wedge(\lnot x_1\vee x_2\vee x_3)$, showing the four state groups.
For each literal of $C_1$, the scanner first reads the value that would make it true; only outcome $0$ followed by outcome $1$ at the opposite-value selector confirms that it is locally false.
The displayed $0,1$ paths advance through the three literals and then accept; every other outcome abandons $C_1$.
The auditor policy $K_\alpha$ walks $v_{x_1},v_{x_2},v_{x_3}$ and, having assigned a value to each variable, checks its occurrences.
Dotted edges summarize audit reads of the shared selectors: the scanner and the auditor of $x_1=F$ read the same $q_{(1,1),F}$, coupling their acceptances.
The remaining occurrences, both clauses, and the depth-$H$ padding follow the same pattern.}
\label{fig:falsification-valuation-example}
\end{figure*}
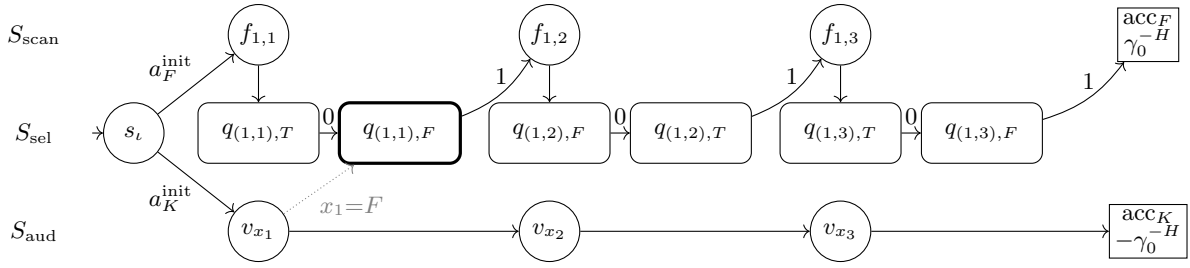
\begin{example}\label{ex:falsification-valuation-verifiers}
For $\varphi=(x_1\vee\lnot x_2\vee x_3)\wedge (\lnot x_1\vee x_2\vee x_3)$, the valuation $(x_1,x_2,x_3)=(1,1,0)$ satisfies $\varphi$, so an accepting valuation audit precludes falsification.
By contrast, $(0,1,0)$ falsifies the first clause, and its canonical local pairs make both verifiers accept.
\end{example}

\begin{lemma}[Verifier separation]\label{lem:falsification-valuation-separation}
Consider the verifier RMDP and policies of \Cref{def:falsification-valuation-verifiers,def:falsification-valuation-policies}.
At every Boolean realization at which $\pi_F$ and $K_\alpha$ both accept, $\alpha$ falsifies a clause of $\varphi$.
Conversely, if $\alpha$ falsifies a clause of $\varphi$, then there exists a Boolean realization at which both accept.
In particular, if $\alpha$ satisfies $\varphi$, the two acceptance events are disjoint at every realization.
\end{lemma}
\begin{proof}
Fix a Boolean realization and suppose $K_\alpha$ accepts at it.
Then every variable set true by $\alpha$ has $b_{o,T}=1$ at each occurrence, and every variable set false has $b_{o,F}=1$.
In every clause satisfied by $\alpha$, a true positive literal therefore cannot have the false pair $(0,1)$, and a true negative literal cannot have $(1,0)$.
Hence no clause satisfied by $\alpha$ is locally false in all three positions.
This also covers malformed pairs $(0,0)$ and $(1,1)$.
Since $\pi_F$ accepts only by finding a clause that is locally false in all three positions, its acceptance at the same realization exhibits a clause that $\alpha$ falsifies.
The same argument read contrapositively gives the disjointness claim for satisfying $\alpha$.

For the converse, let $\alpha$ falsify a clause and let nature use the canonical pair $(1,0)$ for true variables and $(0,1)$ for false variables.
At that realization $K_\alpha$ accepts all of its audits and every literal of the falsified clause is locally false, so $\pi_F$ accepts as well.
Only existence is claimed here, since another realization may make some audit fail.
\end{proof}

For a randomized valuation policy, put
\[
G_\varphi(\sigma)=
\sup_u\bigl(V_u^{\pi_F}(\sinit)-V_u^{K_\sigma}(\sinit)\bigr).
\]
At a vertex, acceptance by $\pi_F$ contributes one and acceptance by $K_\sigma$ contributes another one to this difference.

\begin{lemma}[Verifier gap]\label{lem:falsification-valuation-gap}
For the verifier RMDP and policies of \Cref{def:falsification-valuation-verifiers,def:falsification-valuation-policies}, the following gap holds.
If $\varphi$ is satisfiable, some deterministic valuation policy has $G_\varphi\le1$.
If $\varphi$ is unsatisfiable, every randomized valuation policy has $G_\varphi\ge1+2^{-n}$.
\end{lemma}
\begin{proof}
For a satisfying $\alpha$, \Cref{lem:falsification-valuation-separation} bounds the difference by one at every vertex.
Both policy graphs are acyclic and visit each local selector at most once, so
\Cref{lem:cyclefree-vertex} extends the bound to the entire product polytope.

Now fix $\sigma$ for an unsatisfiable formula.
At each of the $n$ variable states, choose a most probable action.
These actions form a valuation $\alpha$.
Since every such state is visited exactly once, $K_\sigma$ follows that complete valuation with probability at least $\delta=2^{-n}$.
Use the canonical realization for $\alpha$.
The formula has a clause falsified by $\alpha$, so $\pi_F$ accepts with probability one and $K_\sigma$ with probability at least $\delta$.
Thus $G_\varphi(\sigma)\ge1+\delta$.
\end{proof}

We now transfer the verifier gap to minimal robust regret while preventing the synthesized policy from leaving the valuation graph.
Apply \Cref{lem:exact-regret-lift} with $\pi_F$ as reference.
At variable states allow $\{T,F\}$, at verifier states allow the unique action of $K_\sigma$, and at source states used only by $\pi_F$ allow their unique action.
The lifted allowed policies are precisely the valuation policies and
\[
\Rreg^N(K_\sigma)=\Lambda+G_\varphi(\sigma).
\]
Equivalently, the target is $\Restrict(\Lift(N_\varphi,\pi_F),A_{\mathrm{allow}},2^{-n}/4)$, where $N_\varphi$ is the verifier RMDP of \Cref{def:falsification-valuation-verifiers}.
The disallowed bonus choices are singletons, so \Cref{lem:policy-restriction} applies with $\mathcal P=\polRand$.
The gap of \Cref{lem:falsification-valuation-gap} becomes
\[
\varphi\in\mathrm{SAT}\Longrightarrow
\inf_\pi\Rreg(\pi)\le\Lambda+1,\qquad
\varphi\in\mathrm{UNSAT}\Longrightarrow
\inf_\pi\Rreg(\pi)\ge\Lambda+1+\frac{3}{4}2^{-n}.
\]
Threshold $\Lambda+1+2^{-n-1}$ proves $\sNP$-hardness.

\begin{proof}
The construction in the $\scoNP$ subsection is a polynomial reduction from UNSAT, and the construction in the $\sNP$ subsection is a polynomial reduction from SAT.
Their structural restrictions follow from the source verifiers, \Cref{lem:exact-regret-lift}, and \Cref{lem:policy-restriction}.
The rational gaps above leave a valid separating threshold after the restriction loss.
Minimal robust regret is therefore both $\sNP$-hard and $\scoNP$-hard under the restrictions stated in the theorem.
\end{proof}

\section{Deterministic Minimal-Regret Complexity}
\label[appendix]{app:min-regret-deterministic}
We first establish membership and then give the matching hardness construction.

\subsection{Acyclic Deterministic Membership}
\begin{lemma}\label{lem:min-regret-sigma2p-membership}
Minimal robust regret over memoryless deterministic policies is in $\sSigmaTwoP$ on acyclic $(s,a)$-rectangular polytopic RMDPs.
\end{lemma}
\begin{proof}
Fix an acyclic RMDP.
For any uncertainty realization, ordinary finite discounted-MDP optimality gives
\[
V_u^*(s)=\max_{\pi'\in\polDet}V_u^{\pi'}(s).
\]
Indeed, process decision states in reverse topological order and select a maximizing action.
No history-dependent or randomized policy can improve the resulting backward-induction value.

For every deterministic candidate $\pi$,
\[
\begin{aligned}
\Rreg(\pi)
&=\sup_u\max_{\pi'\in\polDet}
\bigl(V_u^{\pi'}(\sinit)-V_u^\pi(\sinit)\bigr)\\
&=\max_{\pi'\in\polDet}\sup_u
\bigl(V_u^{\pi'}(\sinit)-V_u^\pi(\sinit)\bigr).
\end{aligned}
\]
Acyclicity makes every policy pair cycle-free on its shared choices, the only self-loops being those at absorbing finals, which \Cref{def:cyclefree-shared} omits.
By \Cref{lem:cyclefree-vertex}, the remaining supremum is attained at one vertex of every choice uncertainty set.
Therefore
\[
\Rreg(\pi)\le t
\iff
\forall\pi'\in\polDet\
\forall v\in\prod_{s,a}\operatorname{vert}(\cU_{s,a}):
V_v^{\pi'}(\sinit)-V_v^\pi(\sinit)\le t.
\]

The existential certificate is the action table of $\pi$.
The universal certificate contains the action table of $\pi'$ and, for every choice uncertainty set, a vertex represented by linearly independent tight input facets plus its rational coordinates.
Cramer's rule gives polynomial bit length.
The predicate verifies the facets and evaluates both policies by exact backward induction in polynomial time.
Thus the decision problem has an existential polynomial certificate followed by a universal polynomial certificate with a polynomial-time predicate, placing it in $\sSigmaTwoP$.
\end{proof}

\subsection{$\sSigmaTwoP$-Hardness over Deterministic Policies}
\label{app:min-regret-sigma2p-hardness}
Reduce from the canonical $\sSigmaTwoP$-complete problem \citep{DBLP:books/daglib/0072413} of deciding whether
\[
\exists x\in\{0,1\}^n\ \forall y\in\{0,1\}^m:\quad \psi(x,y),
\]
where $\psi$ is a 3-DNF.

\begin{definition}[Universal falsification reference]
\label{def:universal-falsification-reference}
Write $\psi=\bigvee_{i=1}^r D_i$, where every $D_i$ is a conjunction of three literals.
Decide constant cases directly, remove universal variables with no occurrence, and order every variable's occurrences by term and then position.
If necessary, add one unused existential variable so that $n\ge1$.
For an occurrence $o$, let $\operatorname{val}(o)\in\{T,F\}$ be the value making its literal true.
The \emph{universal falsification RMDP} is
\[
N_\psi=\tup{S_\psi,A_\psi,\cU_\psi,R_\psi,\sinit,\gamma_0},
\]
with the following components.
\begin{itemize}
\item $S_\psi$ contains the initial state and one candidate branch-entry state per variable,
\[
\{r^K_1,\ldots,r^K_n\}\cup\{r^Y_1,\ldots,r^Y_m\}.
\]
It contains a global selector $q_{Y_j}$ with outcomes $q_{Y_j}^T,q_{Y_j}^F$ for every universal
variable, and the shared local-bit selector of \Cref{def:local-bit-verifier-primitives} for every
occurrence $o$ and $c\in\{T,F\}$.
\item The reference-side controls are the term-scan states $f_{i,h}$.
The candidate-side controls are the decision states $v_{x_i}$, the existential audit states $k^x_{i,c,\ell}$, and the universal audit states $k^y_{j,c,\ell}$.
There are also accepting, rejecting, padding, and zero-reward terminal states, together with the ruinous sink supplied below.
\item At $\sinit$, action $a_F^{\mathrm{init}}$ leads deterministically to the scanner entry $f_{1,1}$, and action $a_K^{\mathrm{init}}$ has the certain uniform distribution over the $n+m$ branch-entry states.
\item State $r^Y_j$ leads to $q_{Y_j}$.
Outcome $q_{Y_j}^c$ leads to $k^y_{j,c,1}$, which implements the audit chain of
\Cref{def:local-bit-verifier-primitives} for the occurrences of $y_j$ in term-major order.
A universal branch therefore offers the candidate no choice, the committed value being nature's.
\item State $r^K_i$ leads to $v_{x_i}$.
At $v_{x_i}$, actions $T$ and $F$ enter the corresponding chain $k^x_{i,c,1}$.
That control implements the same audit chain for the occurrences of $x_i$ in term-major order.
If $x_i$ has no occurrence, both actions at $v_{x_i}$ accept immediately.
\item At occurrence $o=(i,h)$, the term scanner uses the pair test of
\Cref{def:local-bit-verifier-primitives}; its locally false outcome advances to $f_{i+1,1}$, or
accepts if $i=r$.
Every other outcome advances to $f_{i,h+1}$, or rejects if $h=3$.
\item Every selector has one described action.
At a shared local outcome, one described action takes the scanner continuation and another takes the audit continuation.
Since the two audit families lie behind $a_K^{\mathrm{init}}$ and the scanner behind $a_F^{\mathrm{init}}$, a policy that plays one of those initial actions needs only one of the two continuations, and each continuation is determined by the outcome state itself.
All paths are padded so that acceptance or rejection occurs after a common depth $H$.
\item $\cU_\psi$ is the product of the local-bit segments in
\Cref{eq:boolean-selector-choice} and analogous segments for the global selectors $q_{Y_j}$.
The initial splitter of $a_K^{\mathrm{init}}$ is certain and uniform, and all remaining described choices are
singletons.
\item Before ruinous-sink completion, $R_\psi$ is zero except at padded acceptance terminals.
Every reference-side accepting terminal has payoff $\gamma_0^{-H}$, every candidate-side accepting terminal has payoff $-\gamma_0^{-H}$, and every rejecting terminal has payoff zero.
\item All remaining choices use ruinous-sink completion.
\item The initial state is $\sinit$ and the discount is the appendix-wide $\gamma_0$.
\end{itemize}
\end{definition}

\begin{definition}[Universal falsification policies]
\label{def:universal-falsification-policies}
For the RMDP of \Cref{def:universal-falsification-reference}, the fixed deterministic reference
$\pi_F$ takes $a_F^{\mathrm{init}}$ and then uses the pair-test scanner action throughout.
The synthesized deterministic policy $K_\alpha$ takes $a_K^{\mathrm{init}}$, chooses $\alpha(x_i)$ at every $v_{x_i}$, and follows the audit action for the committed value at every audit state, that value being $\alpha(x_i)$ on an existential branch and nature's selection at $q_{Y_j}$ on a universal one.
Each policy therefore carries one role: the reference scans and the candidate audits.
At states outside a policy's reachable graph, fix default actions whose undescribed choices use ruinous-sink completion.
\end{definition}
By \Cref{lem:acceptance-difference-verifier}, the value difference is the sum of the reference and
candidate acceptance probabilities.

\begin{lemma*}[Unambiguous verifier policies]
The policies in \Cref{def:universal-falsification-policies} are stationary and deterministic, and no run of either policy visits a local selector twice.
\end{lemma*}
\begin{proof}
The candidate's certain uniform splitter chooses one branch.
An audit commits to one value before visiting one variable's occurrences in term-major order, while the term scanner visits the two selectors of each occurrence in their fixed pair-test order.
Every occurrence belongs to one variable, so the audit reaching a given selector is unique, and the scanner reaches each selector once.
Since every audit belongs to $K_\alpha$ and every scan to $\pi_F$, no state asks either policy to distinguish the branch by which it was reached, and one action per state suffices.
\end{proof}

Ignoring final self-loops, the full RMDP is acyclic.
A topological order places $\sinit$ first, then the branch-entry states, the global selectors $Y_1,\ldots,Y_m$, the decision states $v_{x_1},\ldots,v_{x_n}$, the occurrence selectors in term-major order with their associated controls and outcomes, and finally the padding, terminal, and ruinous-sink states.

For a valuation $\alpha$ of $x$, write
\[
G_\psi(\alpha)=
\sup_u\bigl(V_u^{\pi_F}(\sinit)-V_u^{K_\alpha}(\sinit)\bigr).
\]

\begin{lemma}[DNF dichotomy]\label{lem:sigma2p-dnf-dichotomy}
Let $M=n+m$.
For every $\alpha$, if $\forall y:\psi(\alpha,y)$ then $G_\psi(\alpha)\le2-\frac1M$, while if $\exists y:\lnot\psi(\alpha,y)$ then $G_\psi(\alpha)=2$.
\end{lemma}
\begin{proof}
At a Boolean vertex, let $A_j$ be the event that the $y_j$ audit passes, let $B$ be the event that the term scanner passes, and let $C_i$ be the event that the $x_i$ audit against $\alpha(x_i)$ passes.
The reference reaches its single scanner branch with probability one, while the candidate's splitter gives each of the $n+m$ audit branches weight $\frac1{n+m}$, so the padding and terminal rewards give
\[
D_{\pi_F,K_\alpha}
=\mathbf 1_{\{B\}}
+\frac{|\{j:A_j\}|+|\{i:C_i\}|}{n+m}.
\]
If $\exists y:\lnot\psi(\alpha,y)$, the canonical local pairs for $(\alpha,y)$ and the global choices $Y_j=y_j$ make every event hold.
This gives value two, which is also the largest possible value.

Now suppose $\forall y:\psi(\alpha,y)$.
All events cannot hold at once.
Indeed, if every audit passes, set $y_j=Y_j$.
Whenever the pair test declares an occurrence locally false, its value-certifying bit is zero, while the passing audit forces the bit indexed by the committed variable value to one.
The committed value therefore makes that literal false.
Event $B$ would then give a false literal in every DNF term, contradicting $\psi(\alpha,y)$.
At least one event fails, and its weight is at least
\[
\min\left\{1,\frac1{n+m}\right\}=\frac1M.
\]
Thus every vertex has value at most $2-\frac1M$.
No run repeats an uncertain choice, so \Cref{lem:cyclefree-vertex} extends the bound to the full product polytope.
\end{proof}

\begin{lemma}\label{lem:min-regret-sigma2p-hardness}
Minimal robust regret over $\polDet$ is $\sSigmaTwoP$-hard on acyclic $(s,a)$-rectangular RMDPs in which every uncertain choice is two-Dirac and the only other stochastic rows are certain uniform splitters.
\end{lemma}
\begin{proof}
Set $M=n+m$.
By \Cref{lem:sigma2p-dnf-dichotomy}, a yes-instance has some $K_\alpha$ with $G_\psi(\alpha)\le2-\frac1M$, while every $K_\alpha$ in a no-instance has $G_\psi(\alpha)=2$.
Apply the exact lift with reference $\pi_F$.
At existential variable states allow $\{T,F\}$.
At every other state allow the unique action used by the synthesized verifier, and at internal states used only by $\pi_F$ take $A_{\rm allow}(s)=\{\pi_F(s)\}$.
Thus every allowed set is nonempty.
At a shared local outcome this allows the audit action and not the scanner action, which costs the reference nothing: by \Cref{def:lifted-rmdp} the lift replays $\pi_F$ through the bonus action rather than through $A_{\rm allow}$.
The allowed deterministic policies therefore have minimal regret at most $\Lambda+2-\frac1M$ in yes-instances and exactly $\Lambda+2$ in no-instances.

The composed target is
\[
\Restrict\left(\Lift(N_\psi,\pi_F),A_{\mathrm{allow}},\frac1{4M}\right).
\]
Apply \Cref{lem:policy-restriction} with $\mathcal P=\polDet$.
The disallowed bonus choices are singletons.
In yes-instances the unrestricted minimum is at most $\Lambda+2-\frac1M$, while in no-instances it is at least $\Lambda+2-\frac1{4M}$.
Since
\[
\Lambda+2-\frac1M
<
\Lambda+2-\frac1{2M}
<
\Lambda+2-\frac1{4M},
\]
threshold $\Lambda+2-\frac1{2M}$ separates the cases.
The structural restrictions follow from \Cref{lem:exact-regret-lift,lem:policy-restriction}, whose unfolding step is what keeps the composed target acyclic.
\end{proof}

\begin{theorem}\label{thm:min-regret-sigma2p-acyclic-md}
Minimal robust regret over memoryless deterministic policies is $\sSigmaTwoP$-complete on acyclic $(s,a)$-rectangular polytopic RMDPs.
Hardness already holds when every uncertain choice is two-Dirac and the only other stochastic rows are certain uniform splitters.
\end{theorem}

\begin{proof}[Proof of \Cref{thm:min-regret-sigma2p-acyclic-md}]
Membership is \Cref{lem:min-regret-sigma2p-membership}.
Hardness is \Cref{lem:min-regret-sigma2p-hardness}.
\end{proof}

\section{Proof of \Cref{thm:min-regret-sqrs}: Signed Square-Root-Sum Hardness}
\label[appendix]{app:min-regret-sqrs}
The source problem $\sSignedSRS$ gives two lists of positive integers and asks, in its non-strict forward direction, whether
\[
\sum_i\sqrt{a_i}\leq\sum_j\sqrt{b_j}.
\]
The reduction represents each square root by a one-state gadget whose regret balances a decreasing term against an increasing term in the policy's mixing probability.
The optimum is therefore the solution of a quadratic equation and is generally irrational.

\minregretsqrs*

\subsection{One-State Balancing}
\label[appendix]{app:sqrs-balancing}
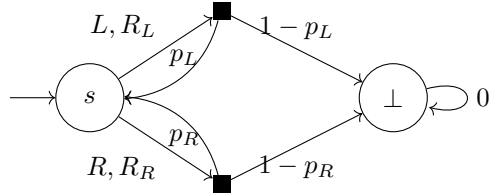
\begin{figure}[ht]
\centering
\begin{tikzpicture}[
    state/.style={circle,draw,minimum size=9mm},
    action/.style={rectangle,fill=black,minimum size=2.3mm,inner sep=0pt},
    boxstate/.style={rectangle,draw,rounded corners,minimum height=8mm,minimum width=13mm}
]
\node[state] (s) {$s$};
\node[action,above right=7mm and 13mm of s] (aL) {};
\node[action,below right=7mm and 13mm of s] (aR) {};
\node[state,right=31mm of s] (bot) {$\bot$};

\draw[->] ($(s.west)+(-6mm,0)$) -- (s.west);
\draw[->] (s) -- node[above left] {$L,R_L$} (aL);
\draw[->] (s) -- node[below left] {$R,R_R$} (aR);
\draw[->] (aL) to[bend left=35] node[above] {$p_L$} (s);
\draw[->] (aR) to[bend right=35] node[below] {$p_R$} (s);
\draw[->] (aL) -- node[above] {$1-p_L$} (bot);
\draw[->] (aR) -- node[below] {$1-p_R$} (bot);
\draw[->] (bot) edge[loop right] node {$0$} ();
\end{tikzpicture}

\caption{The one-state balancing gadget has independent choices $p_L,p_R\in[0,h]$ and a zero-reward sink.}
\label{fig:min-local-balancing-gadget}
\end{figure}

\begin{definition}[One-state balancing gadget]\label{def:one-state-balancing}
Fix $q\in(0,1)$ and $\gamma\in(1-q,1)$, and put $h=(1-q)/\gamma$.
The one-state balancing RMDP shown in \Cref{fig:min-local-balancing-gadget} is
\[
G=\tup{S,A,\cU,R,s,\gamma},
\]
with the following components.
\begin{itemize}
\item $S=\{s,\bot\}$, where $\bot$ is a zero-reward sink.
\item The global action set is $\{L,R\}$.
At $\bot$, both actions have a zero-reward Dirac self-loop.
\item $\cU$ is the product of the independent uncertainty sets
\[
\cU_{(s,L)}=\{p_L\delta_s+(1-p_L)\delta_\bot:p_L\in[0,h]\},\qquad
\cU_{(s,R)}=\{p_R\delta_s+(1-p_R)\delta_\bot:p_R\in[0,h]\}.
\]
\item $R$ assigns rewards $R_L$ and $R_R$ to actions $L$ and $R$ at $s$, respectively, and assigns zero reward at $\bot$.
\item The initial state is $s$ and the discount is the local parameter $\gamma$.
\end{itemize}
The condition $\gamma>1-q$ is exactly $h<1$, so these are valid transition probabilities.
Write $d_L=1-\gamma p_L$ and $d_R=1-\gamma p_R$.
Both range over $[q,1]$.
\end{definition}

For the gadget of \Cref{def:one-state-balancing}, the two parametrizations below choose $(R_L,R_R)$ so that its minimal robust regret realizes the signed square-root forms $m_+$ and $m_-$ summed by the $\sSignedSRS$ reduction.
For $A,B>0$, put $D=A+B$.
The \emph{positive-reward parametrization} is
\[
r_L=\frac{q(A+qB)}{1-q^2},
\qquad
r_R=\frac{q(B+qA)}{1-q^2},
\qquad
(R_L,R_R)=(r_L,r_R),
\]
and define
\[
m_+(A,B)=\frac{D-\sqrt{D^2-4(1-q^2)AB}}{2(1-q^2)}.
\]
The \emph{negative-reward parametrization} is
\[
\ell=\frac{B+qA}{1-q^2},
\qquad
r=\frac{A+qB}{1-q^2},
\qquad
(R_L,R_R)=(-\ell,-r),
\]
and define
\[
m_-(A,B)=\frac{\sqrt{q^2D^2+4(1-q^2)AB}-qD}{2(1-q^2)}.
\]

\begin{lemma}[One-state balancing gadgets]
\label{lem:min-local-balancing-gadgets}
The positive-reward parametrization has minimal robust regret $m_+(A,B)$.
If $q^2\ge1/2$, the negative-reward parametrization has minimal robust regret $m_-(A,B)$.
\end{lemma}

\begin{proof}
A memoryless randomized policy in this gadget is determined by one number $x\in[0,1]$, the probability of playing $L$ at $s$.

First, consider the positive-reward form.
For fixed $d_L,d_R\in[q,1]$,
\[
        V_L=\frac{r_L}{d_L},
        \qquad
        V_R=\frac{r_R}{d_R},
        \qquad
        V_x=\frac{x r_L+(1-x)r_R}{x d_L+(1-x)d_R}.
\]
A direct calculation gives
\[
        V_L-V_x
        =
        \frac{(1-x)(r_Ld_R-r_Rd_L)}
             {d_L(xd_L+(1-x)d_R)}
\]
and
\[
        V_R-V_x
        =
        \frac{x(r_Rd_L-r_Ld_R)}
             {d_R(xd_L+(1-x)d_R)}.
\]
On the region $V_L\ge V_R$, the first expression is decreasing in $d_L$ and increasing in $d_R$, so its maximum is attained at $d_L=q,d_R=1$.
Since
\[
        \frac{r_L}{q}-r_R=A,
\]
this maximum is
\[
        \frac{A(1-x)}{1-(1-q)x}.
\]
Similarly, on the region $V_R\ge V_L$, the second expression is maximized at $d_L=1,d_R=q$.
Since
\[
        \frac{r_R}{q}-r_L=B,
\]
this maximum is
\[
        \frac{Bx}{q+(1-q)x}.
\]
Thus
\[
        \Rreg(x)
        =
        \max\left\{
        \frac{A(1-x)}{1-(1-q)x},
        \frac{Bx}{q+(1-q)x}
        \right\}.
\]
The first term is strictly decreasing in $x$, and the second is strictly increasing in $x$.
Hence the minimum is attained when the two terms are equal.
If the common value is $y$, then
\[
        (1-q^2)y^2-Dy+AB=0,
        \qquad D=A+B.
\]
The smaller root gives the claimed value.

For the negative-reward form the same two-region split applies, now with $V_L=-\ell/d_L$ and $V_R=-r/d_R$.
The conditions $A,B>0$ are equivalent to
\[
        r-q\ell=A>0,
        \qquad
        \ell-qr=B>0,
\]
and imply $q<r/\ell<1/q$.
For $V_L-V_x$, set $z=d_R/d_L\in[q,1/q]$.
On $z\in[q,1]$, minimizing the feasible $d_L$ gives $\frac{(1-x)z(r-\ell z)}{q(x+(1-x)z)}$, whose derivative has numerator $x(r-2\ell z)-(1-x)\ell z^2\le0$ because $z\ge q$, $r/\ell<1/q$, and $q^2\ge1/2$.
On $z\in[1,1/q]$, the corresponding expression $\frac{(1-x)(r-\ell z)}{q(x+(1-x)z)}$ is strictly decreasing.
Thus the maximum occurs at $z=q$, namely $d_L=1,d_R=q$, and equals $A(1-x)/(q+(1-q)x)$.
By symmetry the other region contributes $Bx/(1-(1-q)x)$.
Hence
\[
\Rreg(x)=\max\left\{
\frac{A(1-x)}{q+(1-q)x},
\frac{Bx}{1-(1-q)x}
\right\}.
\]
Balancing the decreasing and increasing terms gives $(1-q^2)y^2+qDy-AB=0$.
Its positive root is $m_-(A,B)$.
\end{proof}

\subsection{Rational Square-Root Gadgets}
\label[appendix]{app:sqrs-gadgets}
\begin{lemma}[Local gadget for a negative square-root coefficient]
\label{lem:min-local-cosrs-gadget}
For every integer $b \geq 2$, one can compute in polynomial time rational numbers $q_b,D_b,A_b,B_b$ such that, for any $\gamma\in(1-q_b,1)$, the positive-reward balancing gadget of \Cref{lem:min-local-balancing-gadgets}, after a rational reward scaling, has minimal robust regret of
\[
        D_b-\sqrt b.
\]
\end{lemma}

\begin{proof}
Choose a rational number $\eta\in(\sqrt b-1,\sqrt{b-1})$.
The interval has length greater than $1/2$ for $b\ge2$, so scanning a constant-denominator dyadic grid finds such an $\eta$.
Both endpoint tests use exact rational square comparisons.
Thus $\eta$ has $O(\log b)$ bits and is found in polynomial time.
Define
\[
\begin{aligned}
        q_b=\frac{b-1-\eta^2}{2\eta}, \quad
        N_b=\frac{\eta+\frac{b-1}{\eta}}{2}, \quad
        D_b=\frac{N_b}{q_b}.
\end{aligned}
\]
The upper bound on $\eta$ gives $q_b>0$, while $\eta>\sqrt b-1$ is equivalent to $q_b<1$.
Direct expansion gives $N_b^2-q_b^2=b-1$.
Hence $N_b>q_b$, so $D_b>1$.
Hence, $q_b^2D_b^2+1-q_b^2=b$.

Now set
\[
        A_b=\frac{D_b+1}{2},
        \qquad
        B_b=\frac{D_b-1}{2}.
\]
Then $A_b,B_b>0$, $A_b+B_b=D_b$, and
\[
\begin{aligned}
        D_b^2-4(1-q_b^2)A_bB_b
        &=D_b^2-(1-q_b^2)(D_b^2-1)\\
        &=q_b^2D_b^2+1-q_b^2\\
        &=b .
\end{aligned}
\]
By \Cref{lem:min-local-balancing-gadgets}, the unscaled positive-reward gadget has minimal robust regret
\[
        \frac{D_b-\sqrt b}{2(1-q_b^2)}.
\]
Scaling all rewards by $2(1-q_b^2)$ gives the value $D_b-\sqrt b$.
Moreover $D_b^2=1+(b-1)/q_b^2>b$, so this local regret is positive.
All rational operations preserve polynomial bit length.
\end{proof}

\Cref{ex:min-local-cosrs-two} instantiates the construction at $b=2$.
\begin{example}[The gadget for $b=2$]\label{ex:min-local-cosrs-two}
Take $\eta=1/2$.
Then $q_2=3/4$, $N_2=5/4$, and $D_2=5/3$, with $N_2^2-q_2^2=1=b-1$.
After the prescribed reward scaling, the local minimal regret is $5/3-\sqrt2$.
\end{example}

\begin{lemma}[Local gadget for a positive square-root coefficient]
\label{lem:min-local-srs-gadget}
For every integer $b\geq 2$, one can compute in polynomial time a rational constant $C_b$ and a negative-reward $(s,a)$-rectangular gadget such that, for any $\gamma\in(\frac{1}{5},1)$, its minimal robust regret is
\[
        \sqrt b-C_b .
\]
\end{lemma}

\begin{proof}
Choose a rational number $m$ such that
\[
        \frac{2b}{3}<m^2<b .
\]
The interval $(\sqrt{2b/3},\sqrt b)$ has width at least
$\sqrt2(1-\sqrt{2/3})$.
Scanning dyadics of polynomial bit length with rational square comparisons therefore finds $m$ in polynomial time.
Define
\[
        D=\frac{m+\frac{b}{m}}{2},
        \qquad
        c=\frac{\frac{b}{m}-m}{2}.
\]
Then $D,c\in\bQ$, $c>0$, and $D^2-c^2=b$.
Moreover,
\[
        \frac{c}{D}
        =
        \frac{b-m^2}{b+m^2}
        <
        \frac15.
\]
Set
\[
\begin{aligned}
        q=\frac45, \quad
        E=\frac53 c, \quad
        A=\frac{D+E}{2}, \quad
        B=\frac{D-E}{2}.
\end{aligned}
\]
Since $\frac{E}{D}<\frac{1}{3}$, we have $A,B>0$.
Also $A+B=D$, and because $1-q^2=\frac{9}{25}$, we have
\[
\begin{aligned}
        q^2D^2+4(1-q^2)AB
        &=q^2D^2+(1-q^2)(D^2-E^2)\\
        &=D^2-(1-q^2)E^2\\
        &=D^2-c^2\\
        &=b.
\end{aligned}
\]
Finally, $\frac{c}{D}<\frac{1}{5}$ implies $D^2=b+c^2<\frac{25b}{24}$, and therefore $ q^2D^2< b$.
Thus $qD<\sqrt b$.

Now use the negative-reward form of \Cref{lem:min-local-balancing-gadgets} with these $q,A,B,D$.
Its unscaled minimal robust regret is
\[
        \frac{\sqrt{b}-qD}{2(1-q^2)}.
\]
Scaling all rewards by $2(1-q^2)$ gives minimal robust regret $\sqrt{b}-qD $.
Set $C_b=qD$.
The strict inequality $qD<\sqrt b$ proved above makes this local regret positive, and all parameters have polynomial binary length.
\end{proof}

\begin{figure}[ht]
\centering
\begin{tikzpicture}[
    state/.style={circle,draw,minimum size=9mm},
    action/.style={rectangle,fill=black,minimum size=2.3mm,inner sep=0pt},
    gadget/.style={rectangle,draw,rounded corners,minimum height=8mm,minimum width=13mm}
]
\node[state] (sin) {$s_{\mathrm{in}}$};
\node[action,below=5mm of sin] (act) {};
\node[gadget,below left=10mm and 19mm of act] (g1) {$H_1$};
\node[gadget,below=10mm of act] (g2) {$H_2$};
\node[below right=10mm and 14mm of act] (dots) {$\cdots$};
\node[gadget,right=14mm of dots] (gn) {$H_n$};

\draw[->] ($(sin.west)+(-6mm,0)$) -- (sin.west);
\draw[->] (sin) -- (act);
\draw[->] (act) -- node[above left] {$1/n$} (g1);
\draw[->] (act) -- node[right] {$1/n$} (g2);
\draw[->] (act) -- node[below] {$1/n$} (gn);
\end{tikzpicture}

\caption{Additive splitter for minimal regret.}
\label{fig:min-regret-splitter}
\end{figure}
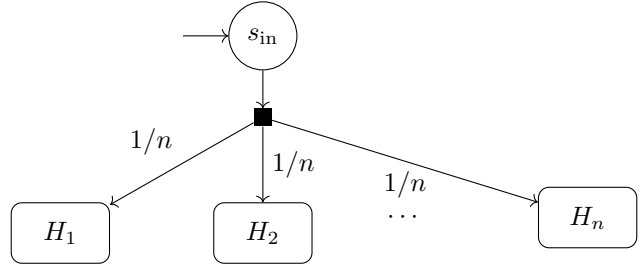

\subsection{Additive Composition}
\label[appendix]{app:sqrs-composition}
\begin{definition}[Additive splitter]\label{def:min-regret-additive-splitter}
Let $H_1,\ldots,H_n$ be disjoint $(s,a)$-rectangular gadgets with common discount $\gamma$, entry states $s_i$, local regrets $\Rreg_{H_i}$, and mutually independent uncertainty sets.
Add a zero-reward initial certain uniform splitter $s_{\mathrm{in}}$ whose only action enters each
$s_i$ with probability $1/n$, and scale every reward in every $H_i$ by $n/\gamma$
(\Cref{fig:min-regret-splitter}).
After scaling, every remaining choice uses ruinous-sink completion.
Denote the resulting RMDP, whose uncertainty set is the product of the local sets, by $H$.
\end{definition}

\begin{lemma}[Additive composition for minimal regret]
\label{lem:min-regret-additive-composition}
If $m_i=\inf_{\pi_i\in\polRand}\Rreg_{H_i}(\pi_i)$, then the additive splitter of \Cref{def:min-regret-additive-splitter} satisfies
\[
\inf_{\pi\in\polRand}\Rreg_H(\pi)=\sum_{i=1}^n m_i.
\]
\end{lemma}

\begin{proof}
For a global policy $\pi$, let $\pi_i$ be its restriction to $H_i$.
For any competing policy $\pi^*$ and any rectangular uncertainty realization
$\vec{u}=(\vec{u}_1,\ldots,\vec{u}_n)$, the Bellman equation at the certain uniform splitter has value
\[
\begin{aligned}
V^{\pi^*}_{\vec{u}}(s_{\mathrm{in}})-V^{\pi}_{\vec{u}}(s_{\mathrm{in}})
&=
\gamma\sum_{i=1}^n\frac1n\frac{n}{\gamma}
\left(
V^{\pi_i^*}_{\vec{u}_i}(s_i)-V^{\pi_i}_{\vec{u}_i}(s_i)
\right)=
\sum_{i=1}^n
\left(
V^{\pi_i^*}_{\vec{u}_i}(s_i)-V^{\pi_i}_{\vec{u}_i}(s_i)
\right).
\end{aligned}
\]
The local uncertainty sets are independent and the competing policy can be chosen independently inside the disjoint gadgets.
Hence
\[
        \Rreg_H(\pi)=\sum_{i=1}^n \Rreg_{H_i}(\pi_i).
\]
Taking the infimum over the product of the local policy simplexes proves the claim.
By \Cref{lem:ruin-dominance}, every ruinous choice is strictly worse than a described local one, so
adding the global action alphabet creates no further minimizer or comparator.
The construction preserves $(s,a)$-rectangularity.
\end{proof}

\subsection{Proof of \Cref{thm:min-regret-sqrs}}
\begin{proof}[Proof of \Cref{thm:min-regret-sqrs}]
Reduce from the non-strict $\leq$ direction of $\sSignedSRS$.
Given two lists $a_1,\ldots,a_m$ and $b_1,\ldots,b_n$, we construct in polynomial time an $(s,a)$-rectangular RMDP $M$ and a rational threshold $t$ such that
\[
\inf_{\pi\in\polRand}\Rreg_M(\pi)\leq t
\]
if and only if
\[
        \sum_{i=1}^m\sqrt{a_i}
        \le
        \sum_{j=1}^n\sqrt{b_j}.
\]
The reversed non-strict comparison uses the same construction after swapping the two lists.
Let $I=\{i \mid a_i\ge 2\}$, $J=\{j \mid b_j\geq 2\}$, and let $r_0=|\{i \mid a_i=1\}|-|\{j \mid b_j=1\}|$ be the rational contribution of unit roots.
If $I\cup J=\emptyset$, we output a single-state zero-reward RMDP with threshold $-r_0$, which is a yes-instance exactly when $r_0\leq 0$.

In the following, assume that at least one non-unit root is present.
Put $N=|I|+|J|$ for the number of local gadgets.
For every $i\in I$, apply \Cref{lem:min-local-srs-gadget}.
This gives a local gadget with minimal robust regret
\[
        \sqrt{a_i}-C_i^+
\]
for a rational constant $C_i^+$.
For every $j\in J$, apply \Cref{lem:min-local-cosrs-gadget} to obtain rational parameters $q_j,D_j^-,A_j^-,B_j^-$.
After the reward scaling from that lemma, gadget $j$ has local minimal robust regret
\[
        D_j^- - \sqrt{b_j}.
\]

Choose one common discount factor
\[
        q_{\min}=\min\left(\left\{\frac45\right\}\cup\{q_j \mid j\in J\}\right),
        \qquad
        \gamma=1-\frac{q_{\min}}{2}.
\]
Then $\gamma\in(\frac{1}{5},1)$, so the positive-coefficient gadgets are valid, and $\gamma>1-q_j$ for every $j\in J$, so the negative-coefficient gadgets are valid as well.
Compose all local gadgets using \Cref{lem:min-regret-additive-composition}.
The resulting RMDP has
\[
\begin{aligned}
        \inf_{\pi\in\polRand}\Rreg(\pi)
        &=
        \sum_{i\in I}(\sqrt{a_i}-C_i^+)
        +
        \sum_{j\in J}(D_j^- - \sqrt{b_j})\\
        &=
        C
        +
        \sum_{i\in I}\sqrt{a_i}
        -
        \sum_{j\in J}\sqrt{b_j},
\end{aligned}
\]
where
\[
        C=\sum_{j\in J}D_j^--\sum_{i\in I}C_i^+ .
\]
Set the threshold $t=C-r_0 $.
Then
\[
\begin{aligned}
        \inf_{\pi\in\polRand}\Rreg(\pi)\le t
        &\iff
        \sum_{i\in I}\sqrt{a_i}
        -
        \sum_{j\in J}\sqrt{b_j}
        \le -r_0\\
        &\iff
        \sum_{i=1}^m\sqrt{a_i}
        \le
        \sum_{j=1}^n\sqrt{b_j}.
\end{aligned}
\]
The uncertainty set is a product of choice uncertainty sets, so the reduction preserves $(s,a)$-rectangularity.
\end{proof}

\section{Proof of \Cref{thm:min-regret-forallr-hard}: General-Polytope Minimal-Regret Hardness}
\label[appendix]{app:regret-minimization-real}
We first prove the existential-universal upper bound, then use two-action anchoring to turn general policy comparison into a strictly monotone minimal-regret objective.

\minregretforallrhard*

\subsection{Existential-Universal Membership}
The membership claim is proved with its certification counterpart in
\Cref{app:portfolio-regret-membership}.

\subsection{Two-Action Anchoring}
Start with a deterministic comparison instance $(M,\pi_1,\pi_2,t)$ (\Cref{prob:policy-comparison}) and write
\[
\Delta=\sup_p\bigl(V_p^{\pi_1}(\sinit)-V_p^{\pi_2}(\sinit)\bigr),
\qquad
V_{\max}=\frac{\max_{s,a}|R(s,a)|}{1-\gamma}.
\]
Choose
\[
L=2\gamma V_{\max}+|\gamma t|+1,
\qquad
Z=L/\gamma+2V_{\max}+1.
\]
Every compliant value of the construction below is bounded in absolute value by
$M_{\widehat M}=L+\gamma(Z+V_{\max})+Z$, since the copies contribute at most $V_{\max}$, the anchors
$\pm Z$, and the initial action at most $L$ on top of those.
Choose a rational $Z_{\rm r}$ of polynomial encoding length such that
\[
\gamma Z_{\rm r}>M_{\widehat M}+1,
\]
which is the hypothesis of \Cref{lem:ruin-dominance}.
\begin{definition}[Two-action anchoring RMDP]
\label{def:two-action-anchoring}
Given $(M,\pi_1,\pi_2,t)$, the \emph{two-action anchoring RMDP} is
\[
\widehat M=\tup{\widehat S,\widehat A,\widehat\cU,\widehat R,\widehat s,\gamma},
\]
with the following components.
\begin{itemize}
\item $\widehat S$ contains a fresh initial state $\widehat s$, two disjoint forced copies of $M$, absorbing states $z_-$ and $z_+$ of values $-Z$ and $Z$, and a ruinous sink $\bot_{\rm r}$ of value $-Z_{\rm r}$.
\item At each state in copy $i$, the action prescribed by $\pi_i$ is the described action.
The initial state offers actions $\alpha$ and $\beta$.
Every action at $z_-$ and $z_+$ self-loops with the reward realizing its displayed payoff.
\item $\widehat\cU$ ties corresponding transition coordinates of the forced copies to one realization $p$ of $\cU$.
It also contains mutually independent coordinates $\lambda_\alpha,\lambda_\beta\in[0,1]$ that are unconstrained by the equalities tying the copies.
Action $\alpha$ enters the first copy with probability $\lambda_\alpha$ and $z_-$ otherwise.
Action $\beta$ enters the second copy with probability $\lambda_\beta$ and $z_+$ otherwise.
\item Rewards inside the forced copies agree with $M$.
Action $\alpha$ pays $L$, action $\beta$ pays zero, and the self-loop rewards at $z_-$ and $z_+$ are $-(1-\gamma)Z$ and $(1-\gamma)Z$, respectively.
Every action at $\bot_{\rm r}$ self-loops with reward $-(1-\gamma)Z_{\rm r}$.
All remaining choices use ruinous-sink completion with $Z_{\rm r}$.
\item The initial state is $\widehat s$ and the discount is $\gamma$.
\end{itemize}
\Cref{fig:two-action-anchoring} illustrates the construction.
\end{definition}

\begin{lemma}[Anchored action values]\label{lem:two-action-values}
At a realization $(p,\lambda_\alpha,\lambda_\beta)$ of \Cref{def:two-action-anchoring}, the values of $\alpha$ and $\beta$ at $\widehat s$ are
\[
A=L+\gamma\bigl(\lambda_\alpha V_p^{\pi_1}(\sinit)-(1-\lambda_\alpha)Z\bigr),
\qquad
B=\gamma\bigl(\lambda_\beta V_p^{\pi_2}(\sinit)+(1-\lambda_\beta)Z\bigr).
\]
\end{lemma}
\begin{proof}
The choice of $Z_{\rm r}$ makes every ruinous-completed action strictly worse than the prescribed continuation at a copied state, for both the comparator and the candidate.
Thus the two copies have values $V_p^{\pi_1}(\sinit)$ and $V_p^{\pi_2}(\sinit)$.
One-step conditioning on the two successors of each initial action gives the formulas.
\end{proof}

Set $D=L+\gamma\Delta$ and $E=2\gamma Z-L$.

\begin{figure}[t]
\centering
\begin{tikzpicture}[>=Latex,
 state/.style={circle,draw,minimum size=7mm,inner sep=1pt},
 box/.style={rectangle,draw,rounded corners,minimum width=19mm,minimum height=8mm},
 sink/.style={circle,draw,double,minimum size=7mm,inner sep=1pt},
 dot/.style={circle,fill,inner sep=1.5pt},every node/.style={font=\small}]
\node[state] (s) {$\widehat s$};
\node[dot,right=13mm of s,yshift=10mm] (da) {};
\node[dot,right=13mm of s,yshift=-10mm] (db) {};
\node[box,right=15mm of da,yshift=6mm] (m1) {forced $M,\pi_1$};
\node[sink,right=15mm of da,yshift=-6mm] (mz) {$-Z$};
\node[box,right=15mm of db,yshift=6mm] (m2) {forced $M,\pi_2$};
\node[sink,right=15mm of db,yshift=-6mm] (pz) {$Z$};
\draw[->] ($(s.west)+(-6mm,0)$) -- (s.west);
\draw[->] (s) -- node[above,sloped] {$\alpha;L$} (da);
\draw[->] (s) -- node[below,sloped] {$\beta;0$} (db);
\draw[->] (da) -- node[above,sloped] {$\lambda_\alpha$} (m1);
\draw[->] (da) -- node[below,sloped] {$1-\lambda_\alpha$} (mz);
\draw[->] (db) -- node[above,sloped] {$\lambda_\beta$} (m2);
\draw[->] (db) -- node[below,sloped] {$1-\lambda_\beta$} (pz);
\end{tikzpicture}
\caption{The two-action anchoring RMDP.
The two copies share the realization $p$, whereas $\lambda_\alpha$ and $\lambda_\beta$ are fresh coordinates.}
\label{fig:two-action-anchoring}
\end{figure}
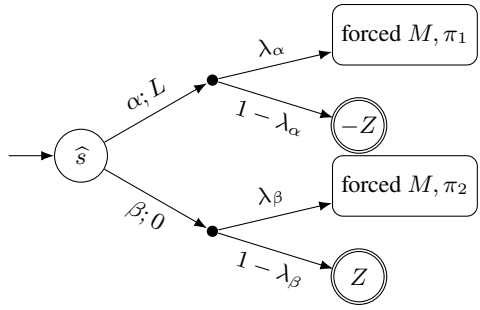
\begin{lemma}[Two-action anchoring]\label{lem:two-action-anchoring}
The constants $D$ and $E$ are positive.
A policy choosing $\alpha$ with probability $x$ has robust regret $\max\{(1-x)D,xE\}$.
\end{lemma}
\begin{proof}
By \Cref{lem:two-action-values}, the action-value difference is affine in each fresh coordinate, so its extrema occur at their four endpoint pairs.
For $A-B$ these values are
\[
\begin{array}{c|c}
(\lambda_\alpha,\lambda_\beta)&A-B\\\hline
(1,1)&L+\gamma(V^{\pi_1}_p-V^{\pi_2}_p)\\
(1,0)&L+\gamma(V^{\pi_1}_p-Z)\\
(0,1)&L-\gamma(Z+V^{\pi_2}_p)\\
(0,0)&L-2\gamma Z.
\end{array}
\]
The definition of $Z$ and $|V_p^{\pi_i}|\le V_{\max}$ make the first endpoint strictly largest after optimizing over $p$, so its supremum is $D$.
For $B-A$ the four values are respectively
\[
\gamma(V^{\pi_2}_p-V^{\pi_1}_p)-L,\quad
\gamma Z-L-\gamma V^{\pi_1}_p,\quad
\gamma(Z+V^{\pi_2}_p)-L,\quad
2\gamma Z-L.
\]
The last is strictly largest and equals $E>0$.
Moreover, $\gamma\Delta\ge-2\gamma V_{\max}$, and hence $D\ge |\gamma t|+1>0$.
At any realization the loss of the $x$-mixture is $(1-x)(A-B)$ when $A\ge B$ and $x(B-A)$ otherwise.
Taking the supremum, using $D,E>0$, gives the claimed maximum.
\end{proof}
\Cref{lem:two-action-anchoring} covers the policies that play $\alpha$ or $\beta$ and then the
prescribed continuation, whereas the infimum ranges over every stationary policy of the completed
RMDP, so the two must first be identified.
For an arbitrary policy $\pi$, its compliant projection $\bar\pi$ has value at least $\pi$'s at every
realization by \Cref{lem:ruin-dominance}, hence regret at most $\pi$'s.
Each copied state leaves $\bar\pi$ only the prescribed action, so $\bar\pi$ is determined by the
single probability $x=\bar\pi(\widehat s,\alpha)$, and the two infima coincide.
Balancing the two branches therefore gives minimum $DE/(D+E)$.
Let
\[
T=L+\gamma t,
\qquad
t'=\frac{ET}{E+T}.
\]
Since $z\mapsto Ez/(E+z)$ is strictly increasing for $z>0$,
\[
\inf_\pi\Rreg(\pi)\le t'
\iff \Delta\le t.
\]
\Cref{ex:two-action-anchoring} shows that the threshold remains rational for an irrational comparison gap.
\begin{example}[Anchoring an irrational comparison gap]
\label{ex:two-action-anchoring}
For the RMDP of \Cref{ex:comparison-interior}, $\gamma=1/2$, $V_{\max}=2$, and $\Delta=7-4\sqrt3$.
Taking $t=0$ gives $L=3$, $Z=11$, $D=13/2-2\sqrt3$, $E=8$, $T=3$, and $t'=24/11$.
Thus the anchoring constants and threshold remain rational even though the balanced regret $DE/(D+E)$ depends on the irrational source gap.
\end{example}
\begin{proof}[Proof of \Cref{thm:min-regret-forallr-hard}]
The equivalence above is a reduction from \Cref{thm:comparison-forallr}'s $\sUTR$-hard comparison problem.
\end{proof}

\section{Proof of \Cref{thm:min-portfolio-eatr}: Bounded Portfolio Synthesis}
\label[appendix]{app:portfolio-eatr}
The reduction reads the two quantifiers of a bounded real sentence as the two players in synthesis.
The portfolio plays the existential witness and nature plays the universal block.
Three obstacles have to be handled: a reward may not depend on a policy or a realization, a stationary
policy fixes one continuation per state, and different members must agree on the witness.
The construction below resolves all three by normalizing the matrix into tests that are affine in the
witness, and by making nature's choice of audit a coefficient inside one polynomial per member rather
than a branch in the model.

\minportfolioeatr*

\subsection{Existential-Universal Membership}
Membership is proved in \Cref{app:portfolio-regret-membership}.

\subsection{A Policy-Affine Max Normal Form}
Using the bounded degree-four normal form of \Cref{def:policy-affine-normal-form}, it is
$\sEUTR$-complete to decide
\[
\exists x\in[0,1]^n\ \forall y\in[0,1]^m:\quad F(x,y)\ge0
\]
after the affine substitution that moves the source box from $[-1,1]$ to $[0,1]$.
Write
\[
F(x,y)=\sum_{\nu=1}^{N_{\rm mon}}c_\nu\prod_{j=1}^{d_\nu}w_{\nu,j},
\qquad d_\nu\le4,
\]
where every $w_{\nu,j}$ is a variable of $x\cup y$, and put
\[
B=\max\left\{1,\sum_{\nu=1}^{N_{\rm mon}}|c_\nu|\right\}.
\]
Thus $|F|\le B$ on the box, and $B$ has polynomial encoding length.
Apply \Cref{def:policy-affine-normal-form} with its optional output coordinate to the input vector
$(x,y)$.
Let $\eta$ consist of $y$ and every copy, product, and output coordinate, and list all residuals as
$h_1,\ldots,h_s$.
Only copy residuals for occurrences of $x$ contain an existential coordinate; each does so affinely
and contains exactly one.

\begin{definition}[Elementary tests]\label{def:elementary-tests}
For every residual $h_j$ and every $\sigma\in\{-1,1\}$, put
\[
g_{j,\sigma}(x,\eta)=\frac{o+\sigma\,9B\,h_j}{10B},
\]
and enumerate the resulting $r=2s$ expressions as $g_1,\ldots,g_r$.
\end{definition}

\begin{lemma}[Elementary-test shape]\label{lem:elementary-test-shape}
Each test of \Cref{def:elementary-tests} satisfies:
\begin{enumerate}
\item $g_i(x,\eta)\in[-1,1]$ on the box;
\item $g_i(x,\eta)=c_i(\eta)+\sum_\ell d_{i,\ell}x_\ell$, where at most one
$d_{i,\ell}$ is nonzero and that coefficient is $\pm\tfrac9{10}$; and
\item $c_i$ has degree at most two in $\eta$.
\end{enumerate}
\end{lemma}
\begin{proof}
The first claim follows from $|o|\le B$ and $|h_j|\le1$.
Only copy residuals contain an existential variable, each contains exactly one with coefficient $-1$,
so its two tests have coefficient $\mp\tfrac9{10}$.
Every other residual and $o$ are free of $x$.
Copy and output residuals are affine, product residuals are quadratic in $\eta$, and $o$ is affine in
$\bar o$.
\end{proof}

\begin{lemma}[Policy-affine max form]\label{lem:policy-affine-maxform}
\[
\exists x\ \forall y:\ F(x,y)\ge0
\quad\Longleftrightarrow\quad
\exists x\ \forall\eta:\ \max_{i\in[r]}g_i(x,\eta)\ge0.
\]
\end{lemma}
\begin{proof}
For the forward direction, fix a witness $x$ and any $\eta$, retain its $y$-part, and let
$\delta=\max_j|h_j|$.
Choose $j$ with $|h_j|=\delta$ and $\sigma$ with $\sigma h_j=\delta$.
By \Cref{lem:policy-affine-error},
\[
o\ge F(x,y)-9B\delta\ge-9B\delta,
\]
and therefore $g_{j,\sigma}\ge0$.
This also covers $\delta=0$.

Conversely, if $x$ is not a witness, choose $y$ with $F(x,y)<0$.
Give every copy and product variable its correct value and set
\[
\bar o=\frac{F(x,y)+B}{2B}\in[0,1].
\]
Then $o=F(x,y)$ and every residual vanishes, so every test equals
$F(x,y)/(10B)<0$.
\end{proof}

\Cref{ex:policy-affine-normal-form} makes all constants in the normalization explicit.
\begin{example}[A policy-affine normal form]\label{ex:policy-affine-normal-form}
Let $n=m=1$ and $F(x,y)=x-\tfrac12y$, whose sentence is true exactly for $x\ge\tfrac12$.
Here $N_{\rm mon}=2$, $c_1=1$, $c_2=-\tfrac12$, and $B=\tfrac32$.
The copies are $z_1,z_2$, with $p_1=z_1$, $p_2=z_2$ and no product residuals:
\[
h_1=z_1-x,\qquad
h_2=z_2-y,\qquad
h_3=\frac13\left(o-z_1+\frac12z_2\right),
\qquad
o=3\bar o-\frac32.
\]
Thus $s=3$, $r=6$, $9B=\tfrac{27}{2}$, and $10B=15$.
Only $g_{1,\pm}$ mention $x$, with coefficients $\mp\tfrac9{10}$.
For $x=\tfrac13$, the reverse construction takes $y=1$, $z_1=\tfrac13$, $z_2=1$, and
$\bar o=\tfrac49$.
Every residual then vanishes and every test equals $(x-\tfrac12)/15<0$.
\end{example}

\subsection{Exact Evaluators}
Both components use the terminal-payoff convention of \Cref{app:sqrs-preliminaries}.
Throughout this section, $\gamma=\tfrac12$.
For a sparse polynomial $P$, write $\operatorname{Poly}(P)$ for the component of
\Cref{def:polynomial-evaluator} at this discount, with repeated logical coordinates tied later by
\Cref{def:portfolio-uncertainty}, and write $e_P$ for its entry.
\Cref{lem:polynomial-evaluator} gives its exact compliant value and its structural bounds.

\begin{definition}[Affine-policy evaluator]\label{def:affine-policy-evaluator}
For
\[
G(x,u)=P_0(u)+\sum_{\ell=1}^{n}x_\ell P_\ell(u),
\]
with every $P_\ell$ explicit, sparse, and of constant degree, the component
$\operatorname{Aff}(G)$ has:
\begin{itemize}
\item an entry $e_G$; coordinate states $s_1,\ldots,s_n$, each with actions
$\mathsf{on}$ and $\mathsf{off}$; a zero-payoff terminal $\mathrm{zr}$; and copies of
\[
\operatorname{Poly}\left(\frac{n+1}{\gamma}P_0\right)
\quad\text{and}\quad
\operatorname{Poly}\left(\frac{n+1}{\gamma^2}P_\ell\right)
\quad(\ell\in[n]);
\]
\item at $e_G$, one action with the certain uniform splitter over the $n+1$ successors
$e_{P_0},s_1,\ldots,s_n$; at $s_\ell$, a reward-zero singleton transition from
$\mathsf{on}$ to $e_{P_\ell}$ and from $\mathsf{off}$ to $\mathrm{zr}$;
\item the rewards and completion of its $\operatorname{Poly}$ copies; every other choice of this
component uses ruinous-sink completion.
\end{itemize}
\end{definition}

\begin{lemma}[Affine-policy evaluation]\label{lem:affine-policy-evaluator}
For every realization $u$ and every compliant $\pi\in\polRand$,
\[
V_u^\pi(e_G)=P_0(u)+\sum_{\ell=1}^{n}x_\ell P_\ell(u),
\qquad
x_\ell=\pi(s_\ell,\mathsf{on}).
\]
\end{lemma}
\begin{proof}
The constant branch contributes
\[
\frac1{n+1}\gamma\frac{n+1}{\gamma}P_0(u)=P_0(u).
\]
At $s_\ell$,
\[
V_u^\pi(s_\ell)
=x_\ell\gamma\frac{n+1}{\gamma^2}P_\ell(u)
+(1-x_\ell)\gamma\cdot0
=x_\ell\frac{n+1}{\gamma}P_\ell(u),
\]
so branch $\ell$ contributes
$\frac{\gamma}{n+1}x_\ell\frac{n+1}{\gamma}P_\ell(u)=x_\ell P_\ell(u)$.
Summing proves the identity.
\end{proof}

\begin{lemma}[One predecessor per coordinate]\label{lem:coordinate-single-predecessor}
In $\operatorname{Aff}(G)$ every coordinate state $s_\ell$ has exactly one predecessor, namely $e_G$,
and each of its two actions has one fixed continuation.
Thus a stationary policy's behavior at $s_\ell$ has a single interpretation, and no continuation
depends on the context from which $s_\ell$ was reached.
\end{lemma}
\begin{proof}
The certain uniform splitter at $e_G$ is the only transition into $s_\ell$, while
$\mathsf{on}$ and $\mathsf{off}$ always lead to $e_{P_\ell}$ and $\mathrm{zr}$, respectively.
\end{proof}

Every occurrence gets a fresh row, and the only policy choices are at the coordinate states.
Thus the evaluator has neither shared selector decoder actions nor context-dependent continuations.
\Cref{fig:portfolio-synthesis-rmdp} shows the nested evaluator structure.

\begin{figure}[t]
\centering
\resizebox{0.98\columnwidth}{!}{%
\begin{tikzpicture}[
  >=Latex,
  state/.style={draw,rounded corners,minimum height=5.5mm,align=center,inner sep=2pt},
  font=\scriptsize,
  x=1cm,y=1cm
]
\node[state] (init) at (-4.8,1.8) {$\sinit$};
\node[state] (eg) at (-2.7,1.8) {$e_G$};
\node[state] (sl) at (-0.5,1.8) {$s_\ell$};
\node[state] (zr) at (1.3,2.8) {$\mathrm{zr}$};
\node[state] (ep) at (1.3,1.2) {$e_{P_\ell}$};
\node[state] (c1) at (3.6,1.2) {$c_{\nu,1}$};
\node[state] (c2) at (5.2,1.2) {$c_{\nu,2}$};
\node[state] (tau) at (6.8,1.2) {$\tau_\nu$\\fixed payoff};
\node[state] (fail) at (5.2,-0.1) {$\mathrm{fl}$};
\draw[->] (init) -- node[above] {$a_i,b_h$} (eg);
\draw[->] (eg) -- node[above] {certain uniform} (sl);
\draw[->] (sl) -- node[above,sloped] {$\mathsf{off}$} (zr);
\draw[->] (sl) -- node[below,sloped] {$\mathsf{on}$} (ep);
\draw[->] (ep) -- node[above] {certain uniform} (c1);
\draw[->] (c1) -- node[above] {$u_1$} (c2);
\draw[->] (c2) -- node[above] {$u_2$} (tau);
\draw[->] (c1) -- node[below,sloped] {$1-u_1$} (fail);
\draw[->] (c2) -- node[right] {$1-u_2$} (fail);
\end{tikzpicture}}
\caption{The synthesis RMDP and its evaluator components.
The only policy choices are at $\sinit$ and the coordinate states; every reward shown is a fixed
rational.}
\label{fig:portfolio-synthesis-rmdp}
\end{figure}
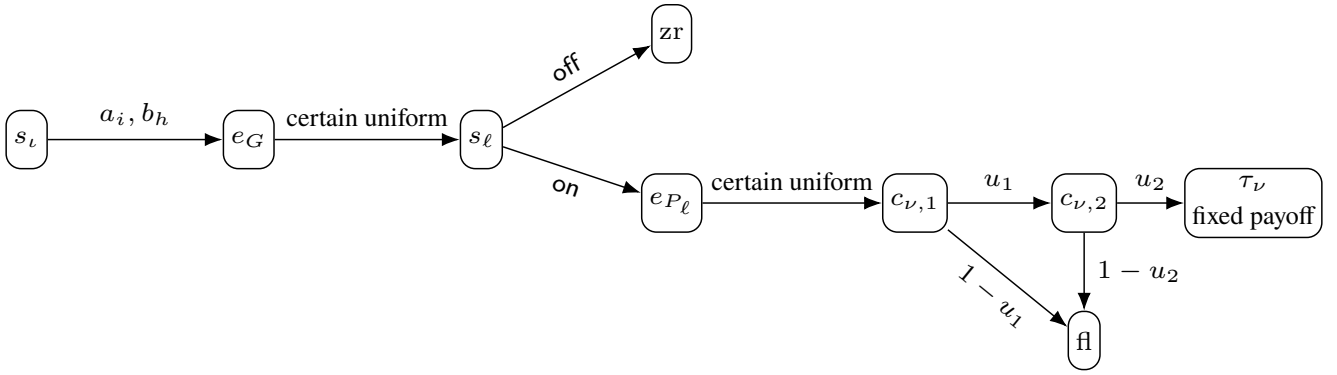

\subsection{Roles, Cases, and Scores}

\begin{definition}[Portfolio case scores]\label{def:portfolio-case-scores}
Let $r$ be the number of tests and also the portfolio budget.
Put
\[
\mathcal H
=\{\mathrm{ev}\}
\cup\{(i,\ell,\rightarrow),(i,\ell,\leftarrow):2\le i\le r,\ \ell\in[n]\}
\cup\{(i,t):i\in[r],\ t\in[3r+1]\}.
\]
Nature's coordinates are a case distribution $\lambda\in\cD(\mathcal H)$, the evaluation vector
$\eta$, and one coordinate $q_h\in[0,1]$ for every oriented equality case.
Write $\xi_h$ for the coordinates local to case $h$.
For role $i$, define the pure-case score $A_{h,i}(x_i,\xi_h)$ as follows.
\begin{itemize}
\item In the evaluation case, $A_{\mathrm{ev},i}(x_i,\eta)=g_i(x_i,\eta)$.
\item For $h=(i,\ell,\rightarrow)$,
\[
A_{h,1}=x_{1,\ell}-q_h,\qquad
A_{h,i}=q_h-x_{i,\ell},
\]
and $A_{h,j}=-1$ for $j\notin\{1,i\}$.
\item For $h=(i,\ell,\leftarrow)$,
\[
A_{h,1}=q_h-x_{1,\ell},\qquad
A_{h,i}=x_{i,\ell}-q_h,
\]
and $A_{h,j}=-1$ for $j\notin\{1,i\}$.
\item For $h=(i,t)$, set $A_{h,i}=0$ and $A_{h,j}=-1$ for $j\ne i$.
\end{itemize}
\end{definition}

\begin{lemma}[Pure-score shape]\label{lem:pure-score-shape}
Every pure-case score of \Cref{def:portfolio-case-scores} lies in $[-1,1]$ and has the form
\[
A_{h,i}(x_i,\xi_h)
=c_{h,i}(\xi_h)+\sum_\ell d_{h,i,\ell}x_{i,\ell},
\]
where $c_{h,i}$ has degree at most two in nature's coordinates and every $d_{h,i,\ell}$ is rational.
\end{lemma}
\begin{proof}
The evaluation scores have the claimed range and shape by \Cref{lem:elementary-test-shape}.
Equality scores are differences of two numbers in $[0,1]$, with
$c_{h,\cdot}=\mp q_h$ and $d_{h,\cdot}=\pm1$.
Forcing scores are constants.
\end{proof}

Define the mixed score of role $i$ by
\begin{equation}\label{eq:mixed-role-score}
W_i
=\sum_{h\in\mathcal H}\lambda_hA_{h,i}(x_i,\xi_h)
+1-\sum_{h\in\mathcal H}\lambda_h^2.
\end{equation}

\begin{lemma}[Mixed-score polynomial]\label{lem:mixed-score-polynomial}
The score $W_i$ has the form
\[
W_i=P_{i,0}(\lambda,\eta,q)
+\sum_\ell x_{i,\ell}P_{i,\ell}(\lambda),
\]
where
\[
P_{i,0}
=1-\sum_h\lambda_h^2+\sum_h\lambda_hc_{h,i}(\xi_h),
\qquad
P_{i,\ell}
=\sum_h\lambda_hd_{h,i,\ell}.
\]
The first polynomial has degree at most three, the second has degree one, and both have polynomially
many monomials.
Consequently, \Cref{lem:affine-policy-evaluator} applies to $W_i$.
\end{lemma}
\begin{proof}
Substitute \Cref{lem:pure-score-shape} into \eqref{eq:mixed-role-score} and collect coefficients.
The largest degree is $1+2$, from $\lambda_h$ multiplying a quadratic $c_{h,i}$.
The number of monomials is bounded by the sum of the explicit monomial counts of the pure scores plus
$1+|\mathcal H|$.
\end{proof}

The cases are coordinates of nature inside these polynomials, not branches in the model.
This lets one evaluator compute all of $W_i$ while every coordinate state retains the single
predecessor guaranteed by \Cref{lem:coordinate-single-predecessor}.

\subsection{The Synthesis RMDP}

\begin{definition}[Portfolio-synthesis RMDP]\label{def:portfolio-synthesis-rmdp}
Let $M_{\rm syn}=\tup{S,A,\cU,R,\sinit,\tfrac12}$ be defined as follows.
\begin{itemize}
\item Its states are $\sinit$; for every role $i\in[r]$, a copy $E_i$ of
$\operatorname{Aff}(\gamma^{-1}W_i)$ with coordinate states
$s_{i,1},\ldots,s_{i,n}$; for every $h\in\mathcal H$, a copy $U_h$ of
$\operatorname{Poly}(\gamma^{-1}(3\lambda_h-1))$; and a ruinous sink.
\item At $\sinit$, the role action $a_i$ has reward zero and its singleton row enters $E_i$.
The auxiliary action $b_h$ has reward zero and its singleton row enters $U_h$.
All other described transitions are those of the two component families.
\item The uncertainty set is given in \Cref{def:portfolio-uncertainty}.
\item Rewards are those of the components.
Every choice still undescribed after the global action set has been installed uses ruinous-sink
completion.
\item The initial state is $\sinit$ and $\gamma=\tfrac12$.
\end{itemize}
\end{definition}

\begin{definition}[Portfolio uncertainty]\label{def:portfolio-uncertainty}
Every occurrence of a logical nature coordinate inside a component is represented by a fresh
two-successor row.
For every $\lambda_h$, every coordinate of $\eta$, and every $q_h$, designate one occurrence as its
representative and impose a rational linear equality setting every other occurrence's success
probability equal to that representative.
On the $\lambda$ representatives add $\lambda_h\ge0$ and
$\sum_h\lambda_h=1$, together with the usual stochasticity constraints on all rows.
\end{definition}

\begin{lemma}[Uncertainty shape]\label{lem:portfolio-uncertainty-shape}
The set $\cU$ is a nonempty rational polytope.
Every uncertain row has two successors, and the only other stochastic rows are the certain uniform
splitters at component entries.
The uncertainty is nonrectangular only through the tying equalities and the simplex constraint.
\end{lemma}
\begin{proof}
Choose any $\lambda$ in the simplex and arbitrary values in $[0,1]$ for the other representatives,
then propagate them through the tying equalities, which witnesses nonemptiness.
Its half-space description consists of exactly four families, all listed in
\Cref{def:portfolio-uncertainty}: one equality per nonrepresentative occurrence tying it to its
representative, the bounds $\lambda_h\ge0$, the equality $\sum_h\lambda_h=1$, and nonnegativity with
normalization on every row.
Every coefficient is $0$ or $\pm1$, and the number of constraints is linear in the number of rows, so
the description has polynomial encoding length and the reduction can write it down.
The remaining claims follow from
\Cref{def:polynomial-evaluator,def:affine-policy-evaluator,def:portfolio-uncertainty}.
\end{proof}

\subsection{Action Values and the Optimal-Value Bound}

\begin{lemma}[Role-action values]\label{lem:portfolio-role-action-values}
In the RMDP of \Cref{def:portfolio-synthesis-rmdp}, for every realization and every compliant
$\pi\in\polRand$,
\[
Q_u^\pi(\sinit,a_i)=W_i(x_i,\lambda,\eta,q),
\qquad
x_{i,\ell}=\pi(s_{i,\ell},\mathsf{on}).
\]
\end{lemma}
\begin{proof}
The initial action makes one reward-zero transition into $E_i$, whose value is
$\gamma^{-1}W_i$ by
\Cref{lem:mixed-score-polynomial,lem:affine-policy-evaluator}.
\end{proof}

\begin{lemma}[Auxiliary-action values]\label{lem:portfolio-aux-action-values}
For every realization and every compliant policy,
\[
Q_u(\sinit,b_h)=3\lambda_h-1.
\]
At a pure case $\lambda=\delta_{h^*}$, this value is two for $h=h^*$ and $-1$ otherwise.
\end{lemma}
\begin{proof}
Apply \Cref{lem:polynomial-evaluator} to
$\gamma^{-1}(3\lambda_h-1)$ and include the initial discount.
\end{proof}

\begin{lemma}[Optimal-value bound]\label{lem:portfolio-optimal-bound}
For every realization, $V_u^*(\sinit)\le2$, with equality whenever
$\lambda=\delta_h$ is a pure case.
\end{lemma}
\begin{proof}
The only policy choices at described states are at $\sinit$ and at coordinate states.
The action values are those of
\Cref{lem:portfolio-role-action-values,lem:portfolio-aux-action-values}.
Since every pure score is at most one,
\[
W_i
\le\sum_h\lambda_h+1-\sum_h\lambda_h^2
=2-\sum_h\lambda_h^2
\le2,
\]
and $3\lambda_h-1\le2$.
By \Cref{lem:ruin-dominance}, these upper bounds extend from compliant policies to all policies.
A randomized action at $\sinit$ gives a convex combination of these values, so it cannot exceed two.
At $\lambda=\delta_h$, action $b_h$ attains two.
\end{proof}

\begin{lemma}[Combining pure cases]\label{lem:portfolio-combine}
For fixed coordinate vectors $x_1,\ldots,x_r$, the following are equivalent:
\begin{enumerate}
\item for every pure case $h$ and every valuation of $\xi_h$, some role has $A_{h,i}\ge0$;
\item for every $\lambda\in\cD(\mathcal H)$ and every valuation of all coordinates, some role has
$W_i\ge0$.
\end{enumerate}
\end{lemma}
\begin{proof}
The second claim implies the first by taking $\lambda=\delta_h$, when $W_i=A_{h,i}$.
For the converse, fix a realization and take
$h^*\in\arg\max_h\lambda_h$.
Apply the pure-case hypothesis at $h^*$ with the realization's own $\xi_{h^*}$ to obtain a role $i$
with $A_{h^*,i}\ge0$.
All other pure scores are at least $-1$, so
\[
\sum_h\lambda_hA_{h,i}\ge-(1-\lambda_{h^*}).
\]
Moreover,
\[
\sum_h\lambda_h^2
\le\lambda_{h^*}\sum_h\lambda_h
=\lambda_{h^*}.
\]
Substitution in \eqref{eq:mixed-role-score} gives
\[
W_i\ge-(1-\lambda_{h^*})+1-\lambda_{h^*}=0.
\]
\end{proof}

Each oriented equality case needs its own coordinate $q_h$: the role selected for $h^*$ depends on
$\xi_{h^*}$, which must remain free of every other case's coordinates.
\Cref{ex:portfolio-combine} shows that the correction bound can be tight.

\begin{example}[The mixing correction]\label{ex:portfolio-combine}
If $\lambda$ is uniform on two cases, then
$1-\sum_h\lambda_h^2=\tfrac12$.
If a role has score zero at one of those cases and $-1$ at the other, its mixed score is
\[
\frac12\cdot0+\frac12\cdot(-1)+\frac12=0.
\]
Thus the combining bound is tight.
\end{example}

\subsection{Correctness}

\begin{lemma}[Forward direction]\label{lem:portfolio-forward}
If the source sentence holds with witness $x$, then the portfolio
$\Pi=\{\pi_1,\ldots,\pi_r\}$ defined by
\[
\pi_i(\sinit)=a_i,
\qquad
\pi_i(s_{j,\ell},\mathsf{on})=x_\ell
\quad\text{for every }j,\ell
\]
and completed compliantly at every other state
satisfies $\Rreg(\Pi)\le2$.
\end{lemma}
\begin{proof}
Every $\pi_i$ is stationary and
$V_u^{\pi_i}(\sinit)=W_i(x)$ by \Cref{lem:portfolio-role-action-values}.
We verify pure-case coverage.
At $\mathrm{ev}$, \Cref{lem:policy-affine-maxform} gives
$\max_i g_i(x,\eta)\ge0$ for every $\eta$.
At an oriented equality case, the two nonconstant scores have maximum
\[
\max\{x_\ell-q_h,q_h-x_\ell\}=|x_\ell-q_h|\ge0.
\]
At a forcing case $(i,t)$, role $i$ scores zero.
By \Cref{lem:portfolio-combine}, every realization has some $W_i\ge0$.
Combining this with \Cref{lem:portfolio-optimal-bound} bounds regret by two.
\end{proof}

\Cref{lem:portfolio-attainment} is needed here because the reverse direction turns $\rho_r\le2$ into a
concrete portfolio.
Its hypotheses hold: $\cU$ is a compact rational polytope by
\Cref{lem:portfolio-uncertainty-shape}.
The other synthesis reductions compute their infimum in closed form.

\begin{lemma}[Role forcing]\label{lem:portfolio-role-forcing}
If $\Pi$ has at most $r$ members and $\Rreg(\Pi)\le2$, then it has exactly $r$ members, exactly one
member chooses each $a_i$ with probability one, and no member chooses an auxiliary action at
$\sinit$.
\end{lemma}
\begin{proof}
At a pure case, \Cref{lem:portfolio-optimal-bound} gives $V_u^*(\sinit)=2$.
Thus the regret bound implies
\[
\max_{\pi\in\Pi}V_u^\pi(\sinit)\ge0
\]
at every pure case and every valuation of its coordinates.

A member need not be compliant: besides its initial choice it may play a globally installed action
inside a component, where the action values below do not apply.
Replace each member $\pi$ by $\bar\pi$, which \Cref{lem:ruin-dominance} shows has value at least
$\pi$'s at every state and realization, so the displayed bound is only strengthened.
The projection keeps the initial distribution over $\{a_i,b_h\}$ and renormalizes at each coordinate
state, which leaves $\pi(s_{i,\ell},\mathsf{on})$ in $[0,1]$, so $\bar\pi$ still encodes a witness and
the count below is unaffected.

Fix a role $i$ and a forcing case $h=(i,t)$.
At $\lambda=\delta_h$, action $a_i$ has value zero, every other role action has value $-1$, action
$b_h$ has value two, every other auxiliary action has value $-1$, and every completed action has value
at most $-1$.
For a policy with initial distribution $p$, collect the mass outside $\{a_i,b_h\}$ to obtain
\[
V_u^\pi(\sinit)
\le p(a_i)+3p(b_h)-1.
\]
Suppose $p(a_i)<1$ and put $\delta_\pi=1-p(a_i)>0$.
Covering $h$ forces $p(b_h)\ge\delta_\pi/3$.
The actions $b_{(i,t)}$ are distinct and outside $a_i$, so
\[
\sum_{t=1}^{3r+1}p(b_{(i,t)})\le\delta_\pi.
\]
This policy can therefore cover at most three of the $3r+1$ forcing cases for role $i$.
If every portfolio member had $p(a_i)<1$, the at most $r$ members would cover at most $3r$ such cases,
a contradiction.
Thus some member has $p(a_i)=1$.

No policy can put probability one on both $a_i$ and $a_j$ for $i\ne j$.
The $r$ roles therefore require $r$ distinct members, exhausting the budget.
Every member is pure on its role action at $\sinit$, leaving none on an auxiliary action.
\end{proof}

\begin{lemma}[Equality cases]\label{lem:portfolio-equality-cases}
Under the hypotheses of \Cref{lem:portfolio-role-forcing}, let $x_i$ be the vector encoded by the
member assigned to role $i$.
Then $x_1=\cdots=x_r$.
\end{lemma}
\begin{proof}
Fix $i\ge2$ and $\ell$.
At pure case $(i,\ell,\rightarrow)$, coverage for a given $q\in[0,1]$ is
\[
\max\{x_{1,\ell}-q,\ q-x_{i,\ell},\ -1\}\ge0,
\]
equivalently
\[
\max\{x_{1,\ell}-q,\ q-x_{i,\ell}\}\ge0.
\]
If $x_{1,\ell}<x_{i,\ell}$, the legal midpoint
$q=\tfrac12(x_{1,\ell}+x_{i,\ell})$ makes both terms negative.
Conversely, if $x_{1,\ell}\ge x_{i,\ell}$, every $q$ satisfies
$q\le x_{1,\ell}$ or $q\ge x_{i,\ell}$, so one term is nonnegative.
Coverage for every $q$ is therefore equivalent to
$x_{1,\ell}\ge x_{i,\ell}$.
The reverse oriented case gives $x_{i,\ell}\ge x_{1,\ell}$.
\end{proof}

\begin{lemma}[Evaluation case]\label{lem:portfolio-evaluation-case}
Under the same hypotheses, with $x$ the common vector,
\[
\max_i g_i(x,\eta)\ge0
\qquad\text{for every }\eta.
\]
\end{lemma}
\begin{proof}
At the pure case $\mathrm{ev}$, role $i$ has value
$A_{\mathrm{ev},i}(x_i,\eta)=g_i(x_i,\eta)$.
Pure-case coverage gives $\max_i g_i(x_i,\eta)\ge0$.
Apply \Cref{lem:portfolio-equality-cases}.
\end{proof}

\begin{proof}[Proof of \Cref{thm:min-portfolio-eatr}]
Membership is the existential-universal encoding in the first subsection.
For hardness, \Cref{lem:portfolio-forward} gives $\rho_r\le2$ whenever the source sentence holds.
Conversely, if $\rho_r\le2$, then \Cref{lem:portfolio-attainment} supplies a portfolio of at most $r$
members with regret at most two.
\Cref{lem:portfolio-role-forcing,lem:portfolio-equality-cases,lem:portfolio-evaluation-case}, followed
by \Cref{lem:policy-affine-maxform}, gives the source sentence.
Thus
\[
\rho_r\le2
\quad\Longleftrightarrow\quad
\exists x\ \forall y:\ F(x,y)\ge0.
\]
The reduction is polynomial and satisfies the stated restrictions by
\Cref{lem:portfolio-synthesis-size,lem:portfolio-synthesis-structure}.
\end{proof}

\Cref{ex:portfolio-synthesis-rmdp} illustrates both equality auditing and role forcing.
\begin{example}[A two-role synthesis instance]\label{ex:portfolio-synthesis-rmdp}
Take $r=2$ and $n=1$.
Then $3r+1=7$ and
$|\mathcal H|=1+2+14=17$.
At the pure case $(2,1,\rightarrow)$ with $q=\tfrac25$, role one scores
$x_{1,1}-\tfrac25$ and role two scores $\tfrac25-x_{2,1}$.
For $x_{1,1}=\tfrac3{10}$ and $x_{2,1}=\tfrac12$, both are $-\tfrac1{10}$.
No member is nonnegative, and the regret is $2+\tfrac1{10}$, exposing
$x_{1,1}<x_{2,1}$.
If both coordinates are $\tfrac12$, role one scores $\tfrac1{10}$ and covers the case.

At forcing case $(1,3)$, a member with $p(a_1)=1$ scores zero.
A member with
$p(a_1)=p(b_{(1,3)})=\tfrac12$ scores
\[
3\cdot\frac12+\frac12-1=1
\]
and also covers that case, but its remaining mass $\tfrac12$ can place the required
$\tfrac16$ on at most three of role one's seven forcing actions.
\end{example}

\subsection{Size and Structural Restrictions}

\begin{lemma}[Synthesis size]\label{lem:portfolio-synthesis-size}
The reduction runs in polynomial time.
\end{lemma}
\begin{proof}
There is one copy variable per occurrence, at most three product variables per monomial, and one output
coordinate.
Thus $s$ and $r=2s$ are polynomial, and
\[
|\mathcal H|=1+2(r-1)n+r(3r+1).
\]
Every $W_i$ has degree at most three and polynomially many monomials.
Each $\operatorname{Poly}$ component has one branch per monomial and at most three factor rows per
branch.
The budget $r$ is written in unary and is polynomial.
Every constant comes from rational arithmetic on input coefficients, $B$, polynomially bounded
integers, and powers $2^{e_\nu+1}$, so every encoding length is polynomial.
\end{proof}

\begin{lemma}[Synthesis structure]\label{lem:portfolio-synthesis-structure}
The RMDP $M_{\rm syn}$ has fixed discount $\tfrac12$, threshold two, and is acyclic apart from absorbing
terminals.
Every uncertain row has two successors, and the only other stochastic rows are certain uniform
splitters.
Every reward is a fixed rational independent of the policy and realization.
Nonrectangularity arises only from the tying equalities and simplex constraint of
\Cref{def:portfolio-uncertainty}.
\end{lemma}
\begin{proof}
Every described path goes from $\sinit$ to a component entry, through a certain uniform splitter, an
optional coordinate state and a chain of factor selectors, and then to a terminal.
No transition returns to an earlier state, and the only self-loops are at terminal states and named
ruinous sinks.
The factor selectors are precisely the fresh two-successor uncertain rows.
Policy probabilities enter only through coordinate-state actions, and realization coordinates only
through factor rows; all terminal payoffs and rewards were fixed before the realization and policy are
chosen.
The remaining claims follow from \Cref{lem:portfolio-uncertainty-shape}.
\end{proof}

Folding the cases into the coefficients of one polynomial per role ensures that each role's entire
mixed score is computed by one evaluator and every coordinate state has one predecessor.

\section{Experimental setup}
\label[appendix]{app:experiments}

This appendix gives the details behind \Cref{sec:implementation}.
We first fix the class of models the benchmarks belong to and the uncertainty set each one induces, then describe the offline construction of a portfolio, and then the two evaluations answering the research questions:
whether portfolios reduce empirical robust regret as the budget grows \textbf{(RQ1)}, and whether the best member can be identified online in a fixed but unknown environment \textbf{(RQ2)}.
\Cref{alg:experimental-pipeline-full} summarizes the whole procedure.
Detailed descriptions of the two benchmarks close the appendix, and the tables and figures themselves appear in \Cref{app:experimental_data}.

\subsection{Parametric MDPs and their uncertainty sets}

A \emph{parametric MDP} (pMDP) is a tuple $\tup{S,A,P,R,\sinit,\gamma,X}$ in which $X$ is a finite set of parameters and every transition probability $P(s,a,s')$ is a polynomial in $\bQ[X]$.
A valuation $\theta\colon X\to\bR$ is \emph{well defined} when it turns every choice row into a probability distribution, and \emph{graph preserving} when no transition polynomial that is not identically zero evaluates to zero under it.
Both benchmarks below are \emph{affine} pMDPs: every transition polynomial has degree one, and the admissible valuations form a compact box $D\subseteq\bR^X$ on which well-definedness holds by construction, so no clamping or renormalization is ever needed.

Such a benchmark is an RMDP in the sense of \Cref{sec:preliminaries}: each $\theta\in D$ instantiates a transition vector $\vec u_\theta$, and
\[
\cU_D=\{\vec u_\theta : \theta\in D\}
\]
is the induced uncertainty set.
This set is \emph{not} $(s,a)$-rectangular.
A single parameter occurs in the rows of many different choices --- the wind intensity $p$ governs every grid cell of the UAV benchmark, and the cooling effectiveness $p_c$ every state of the datacenter benchmark --- so fixing nature's move at one choice fixes it everywhere else.

Robust policy evaluation, however, is performed by robust value iteration \citep{DBLP:journals/mor/Iyengar05}, which requires a rectangular uncertainty set.
We therefore relax each parameter region $c\subseteq D$ to the interval MDP whose per-choice uncertainty sets are the coordinatewise ranges of the transition polynomials over $c$.
Writing $\cU^{\square}_c$ for the resulting rectangular set, this relaxation only ever adds transition vectors,
\[
\cU_c=\{\vec u_\theta:\theta\in c\}\subseteq\cU^{\square}_c,
\]
because it drops the equalities tying repeated occurrences of a parameter to a common value.

\subsection{Offline portfolio construction}

\paragraph{Candidates.}
We split the parameter box $D$ into a uniform grid of cells by dividing each parameter's interval into $10$ equal-width bins, so a two-parameter benchmark yields $\lvert\mathcal C\rvert=10^2=100$ cells.
For every cell $c\in\mathcal C$ we instantiate the pMDP at the cell midpoint $\operatorname{mid}(c)$ and solve the resulting ordinary MDP by value iteration at tolerance $10^{-8}$, which gives an optimal policy $\pi^*_{\operatorname{mid}(c)}$ and its value $V^*_{\operatorname{mid}(c)}(\sinit)$.
The candidate set $\Pi(\mathcal C)$ collects one such memoryless deterministic policy per cell.

\paragraph{Scoring candidates.}
Every candidate is then evaluated robustly against every cell.
For a candidate $\pi$ and a cell $c'$, robust value iteration on the interval relaxation $\cU^{\square}_{c'}$ returns
\[
\underline V^\pi_{c'}
=\inf_{\vec u\in\cU^{\square}_{c'}}V^\pi_{\vec u}(\sinit),
\]
and the score recorded in the loss matrix is the midpoint-anchored difference
\begin{equation}\label{eq:midpoint-loss}
L_{\pi,c'}
=V^*_{\operatorname{mid}(c')}(\sinit)-\underline V^\pi_{c'}.
\end{equation}
Two approximations separate $L_{\pi,c'}$ from the exact cellwise robust regret $\sup_{\theta\in c'}\bigl(V^*_\theta(\sinit)-V^\pi_\theta(\sinit)\bigr)$.
First, the optimal value is taken at the single point $\operatorname{mid}(c')$ rather than at the maximizing $\theta$, which may err in either direction.
Second, by the inclusion above, the infimum over $\cU^{\square}_{c'}$ is at most the infimum over $\cU_{c'}$, so this term alone makes $L_{\pi,c'}$ an over-estimate of the loss it stands in for.
The matrix $L$ is a clustering score, not a certificate; the guarantees reported in \Cref{sec:implementation} come from the sampled evaluation of the next subsection, which uses no relaxation at all.

\paragraph{Selecting the portfolio.}
Each candidate $\pi$ thus carries a profile $L_\pi\in\bR^{\mathcal C}$, one entry per cell, describing where in the parameter space it does well.
For a budget $K$ we run K-means on these profiles and keep, from each of the $K$ clusters, the candidate whose profile is nearest that cluster's center, yielding a portfolio $\Pi_K\subseteq\Pi(\mathcal C)$ of $K$ memoryless deterministic policies.
Clustering is a cheap stand-in for the intractable synthesis problem of \Cref{prob:min-portfolio-regret}: it never leaves the candidate set, and it optimizes profile similarity rather than portfolio regret.

\paragraph{A single-policy reference point.}
The same loss matrix also yields an exact reference for $K=1$.
Minimizing each candidate's worst cell,
\[
\min_{\pi\in\Pi(\mathcal C)}\ \max_{c\in\mathcal C}\ L_{\pi,c},
\]
is the best worst-case score attainable by any single member of the candidate set, computed directly from $L$ without sampling or clustering.
It is the mini-max regret figure quoted in \Cref{sec:implementation}.

\paragraph{Implementation.}
The prototype is written in Python.
Models are built and instantiated with Stormvogel~\citep{VolkEtAl26StormTutorial} on top of Storm, which also provides the parameter regions and the conversion of a region into the interval MDP used for the relaxation above.
Value iteration, robust value iteration, and fixed-policy evaluation are our own NumPy implementations, in dense and sparse variants; the larger benchmarks use the sparse path throughout.
Clustering uses scikit-learn's K-means with automatic restart selection and the seed as its random state.
Each stage of the pipeline caches its output in HDF5, so a stage can be resumed or re-run without repeating its predecessors.

\subsection{RQ1: measuring portfolio regret}

The portfolios are scored against valuations drawn directly from $D$, with no cells and no rectangular relaxation involved.
For a seed we draw $\widehat D\subset D$ uniformly at random with $\lvert\widehat D\rvert=1000$, and at every $\theta\in\widehat D$ we instantiate the pMDP and compute two exact quantities: the optimal value $V^*_\theta(\sinit)$ by value iteration, and the value $V^\pi_\theta(\sinit)$ of each portfolio member by fixed-policy evaluation.
The reported estimate is
\[
\widehat\Rreg_{\widehat D}(\Pi)
=\max_{\theta\in\widehat D}\Bigl(V^*_\theta(\sinit)-\max_{\pi\in\Pi}V^\pi_\theta(\sinit)\Bigr).
\]
Each valuation is thus charged the shortfall of the best member \emph{at that valuation}, matching the offline coverage quantity of \eqref{eqn:rreg} and charging nothing for identifying that member online; RQ2 measures the identification cost separately.
Because the supremum in \eqref{eqn:rreg} is replaced by a maximum over finitely many samples, $\widehat\Rreg_{\widehat D}(\Pi)$ under-approximates $\Rreg(\Pi)$.
The whole construction-and-evaluation procedure is repeated for the three seeds $0,1,2$, which reseed both the K-means initialization and the sample $\widehat D$; \Cref{tab:merged-results} reports the individual seeds and \Cref{tab:portfolio-results-summary} their mean.

\subsection{RQ2: online portfolio selection}

\paragraph{Setting.}
A fixed but hidden valuation $\theta$ defines a generative model $G_\theta$ from which trajectories of the instantiated MDP can be sampled starting at $\sinit$.
The portfolio members are treated as the arms of a bandit: pulling arm $\pi$ rolls out one trajectory of horizon $H=100$ under $\pi$ and returns its discounted return.
Nothing but these sampled returns is observed, neither $\theta$ nor any value function, so a run must discriminate between members from rollouts alone.
Returns lie in $[-B_H,B_H]$ for
\[
B_H=R_{\max}\frac{1-\gamma^H}{1-\gamma},
\]
with $R_{\max}$ the benchmark's largest reward magnitude.

\paragraph{The selection rule.}
We use UCB \citep{pred-learn-games} in its best-arm-identification form.
After one warm-up pull of every arm, round $t$ computes for each arm $i$ still in contention a confidence radius: the empirical-Bernstein radius~\citep{10.1145/1390156.1390241}, whose leading term shrinks with the empirical variance rather than with the worst-case range; on these benchmarks the returns are far less variable than $B_H$ suggests, so it converges considerably faster than a Hoeffding radius.
Writing $i_{\mathrm{safe}}$ for the arm of greatest lower confidence bound, the run stops as soon as
\[
\mathrm{lcb}(i_{\mathrm{safe}})\ \ge\ \max_{j\neq i_{\mathrm{safe}}}\mathrm{ucb}(j)-\varepsilon_{\mathrm{abs}},
\]
that is, as soon as no other arm can be more than $\varepsilon_{\mathrm{abs}}$ better.
Otherwise every arm whose upper bound has fallen below $\mathrm{lcb}(i_{\mathrm{safe}})-\varepsilon_{\mathrm{abs}}$ is eliminated, and the arm of greatest upper confidence bound is pulled next.
A run therefore has no fixed iteration budget: it stops when its own stopping test fires, and the number of pulls varies from run to run and grows with $K$.

We use confidence $\delta=0.1$ and tolerance $\varepsilon=10^{-3}$, the latter expressed as a fraction of the return bound, so the absolute tolerance is $\varepsilon_{\mathrm{abs}}= \varepsilon \cdot B_H = 10^{-3}B_H$.
Since the empirical-Bernstein radius does the work of separating the arms, $\varepsilon$ acts only as a cap that keeps near-ties from running indefinitely, and is kept small.

\paragraph{What is measured.}
At each iteration, a run has both a \emph{pulled} arm, the exploratory one of greatest upper confidence bound, and a \emph{recommended} arm, the one of greatest lower confidence bound --- the member it would deploy if forced to stop right then.
We report the recommended arm, the pulled arm being exploratory by design.

A recommendation is scored against the best \emph{portfolio} member at the hidden valuation, that is, $\arg\max_{\pi\in\Pi}V^\pi_\theta(\sinit)$ computed exactly by fixed-policy evaluation, and not against the unrestricted optimum $V^*_\theta$.
The question is whether online selection finds the best policy the portfolio actually contains.
Two quantities are recorded per iteration.
The first is the fraction of runs whose recommendation is that best member.
The second is the shortfall of the recommended member, normalized by the spread of the portfolio at that valuation,
\[
L^\pi_\theta
=\frac{V^{\mathrm{best}}_\theta(\iota) -V^\pi_\theta(\iota)}{V^{\mathrm{best}}_\theta(\iota)-V^{\mathrm{worst}}_\theta(\iota)},
\]
so that recommending the best member scores $0$ and recommending the worst scores $1$ (with $L^\pi_\theta=0$ by convention when every member ties).
The second measure is the more informative of the two: the first counts a near-tie between two almost equally good members as an outright failure, whereas the second charges only what that confusion actually costs.

\paragraph{Sampling and aggregation.}
For each portfolio size $K$ and each seed, we draw $30$ valuations uniformly from $D$ and run the selection rule once per valuation, all from one seeded random stream.
Pooling the three seeds gives the $90$ runs behind each curve.
Because runs stop at different iterations, a curve at iteration $t$ averages over runs still active at $t$ together with runs that have already committed, each of the latter contributing its own deployed outcome from its stopping time onward.
Curves are drawn on a logarithmic iteration axis with values binned accordingly, and are shown with standard-error bands.

\subsection{Datacenter benchmark}

This benchmark models a controller regulating a server room's temperature and humidity while working off a job queue, uncertain about how strongly its actions and the ambient environment actually move the system.

\paragraph{State space.}
A state $s=(T,H,L)\in S$ records a discretized temperature level $T\in\{0,\dots,10\}$, a humidity level $H\in\{0,\dots,5\}$, and a queue length $L\in\{0,\dots,4\}$, so $\lvert S\rvert=11\cdot6\cdot5=330$.
Episodes start at $s_0=(5,2,2)$, a comfortable middle of every range.

\paragraph{Action space.}
At every state, the controller picks one of five actions,
\[
A=\{\textsc{cool-high},\textsc{cool-low},\textsc{hold},\textsc{heat},\textsc{dehumidify}\}.
\]
They trade energy against comfort: \textsc{cool-high} pushes hardest on temperature and costs the most, \textsc{cool-low} is its cheaper and weaker counterpart, \textsc{hold} spends almost nothing and lets the room drift, \textsc{heat} moves temperature the other way, and \textsc{dehumidify} targets humidity instead of temperature.

\paragraph{Parameter space.}
Two parameters are unknown, $\theta=(p_c,p_e)\in\Theta=[0.5,0.9]\times[0.1,0.8]$.
The cooling effectiveness $p_c$ is how reliably an action moves the room as intended, and the environmental pressure $p_e$ is how strongly workload and ambient conditions push back.
A high $p_c$ makes the controller's actions dependable; a high $p_e$ means the room heats, humidifies, and accumulates work regardless of what the controller does.

\paragraph{Transitions.}
Every state-action pair has four successor branches, split evenly between a \emph{control} outcome governed by $p_c$ and an \emph{environment} outcome governed by $p_e$:
\[
\mathbb P_\theta(s'\mid s,a)\in
\left\{\tfrac{p_c}{2},\ \tfrac{1-p_c}{2},\ \tfrac{p_e}{2},\ \tfrac{1-p_e}{2}\right\},
\qquad
\tfrac{p_c}{2}+\tfrac{1-p_c}{2}+\tfrac{p_e}{2}+\tfrac{1-p_e}{2}=1,
\]
so the four branches sum to one at every $\theta$ and every transition probability is affine in $\theta$.
Writing $[x]_{lo}^{hi}:=\min(hi,\max(lo,x))$ for clamping a coordinate to its range, and abbreviating $T^{\pm}:=[T{\pm}1]_0^{10}$, $H^{\pm}:=[H{\pm}1]_0^{5}$, $L^{\pm}:=[L{\pm}1]_0^{4}$, \Cref{tab:datacenter-transitions} gives the successor reached by each branch.
At a boundary state, two branches can land on the same successor, in which case their probabilities are added.

\begin{table}[h]
\centering
\begin{tabular}{lcccc}
\toprule
Action & $p_c/2$ & $(1-p_c)/2$ & $p_e/2$ & $(1-p_e)/2$ \\
\midrule
\textsc{cool-high}  & $(T^-,H,L)$ & $(T^+,H,L)$ & $(T,H,L^-)$   & $(T,H^+,L^+)$ \\
\textsc{cool-low}   & $(T^-,H,L)$ & $(T,H,L)$   & $(T,H,L^-)$   & $(T^+,H,L^+)$ \\
\textsc{hold}       & $(T,H,L)$   & $(T^+,H,L)$ & $(T,H^-,L^-)$ & $(T^+,H^+,L^+)$ \\
\textsc{heat}       & $(T^+,H,L)$ & $(T^-,H,L)$ & $(T,H,L^-)$   & $(T^+,H^+,L^+)$ \\
\textsc{dehumidify} & $(T,H^-,L)$ & $(T,H^+,L)$ & $(T^-,H,L^-)$ & $(T^+,H^+,L^+)$ \\
\bottomrule
\end{tabular}
\caption{Successor state reached by each of the four branches of $\mathbb P_\theta(\cdot\mid s,a)$, for $s=(T,H,L)$, in the datacenter benchmark.}
\label{tab:datacenter-transitions}
\end{table}

\paragraph{Rewards and discount.}
The reward is deterministic and independent of $\theta$:
\[
r(s,a)=r_E(a)+r_T(T)+r_H(H)+r_L(L),
\]
combining a per-action energy cost
\[
r_E(\textsc{cool-high})=-4,\quad
r_E(\textsc{cool-low})=-2,\quad
r_E(\textsc{hold})=-1,\quad
r_E(\textsc{heat})=-2,\quad
r_E(\textsc{dehumidify})=-3,
\]
with penalties that activate as each dimension approaches its ceiling:
\[
r_T(T)=\begin{cases}0&T\le7\\-20&T\in\{8,9\}\\-50&T=10\end{cases}
\qquad
r_H(H)=\begin{cases}0&H\le3\\-15&H=4\\-40&H=5\end{cases}
\qquad
r_L(L)=\begin{cases}0&L\ne4\\-10&L=4\end{cases}
\]
The controller therefore pays continuously for energy and catastrophically for letting any dimension reach its ceiling, and the parameters decide how expensive it is to keep away from those ceilings.
The largest reward magnitude is $R_{\max}=\max_{s,a}\lvert r(s,a)\rvert=104$, attained at $T{=}10,H{=}5,L{=}4$ under \textsc{cool-high}, and the discount is $\gamma=0.95$, giving $\Vmax=R_{\max}/(1-\gamma)=2080$.

\subsection{UAV benchmark}

This benchmark models a UAV flying through a cluttered three-dimensional grid from a fixed start to a landing pad, under wind drift and actuator-failure risk.

\paragraph{State space.}
A grid state $s=(x,y,z)$ ranges over $\{0,\dots,L_x{-}1\}\times\{0,\dots,L_y{-}1\}\times\{0,\dots,L_z{-}1\}$, where $x$ is the direction of travel, $y$ the lateral axis the wind blows across, and $z$ the altitude with $z=0$ at ground level.
Three further states complete the space: an absorbing \textsc{Crash} state, a \textsc{Goal} state, and the terminal sink \textsc{done} that \textsc{Goal} moves to.
The start is $s_0=(0,\lfloor L_y/2\rfloor,1)$.

\paragraph{Action space.}
In every grid cell, the UAV chooses one of seven actions,
\[
\mathcal A=\{\textsc{E},\textsc{W},\textsc{N},\textsc{S},\textsc{UP},\textsc{DOWN},\textsc{HOVER}\},
\]
the six unit moves along $\pm x$, $\pm y$, $\pm z$ together with a no-op.
\textsc{E} advances toward the goal, \textsc{N} and \textsc{S} steer laterally, \textsc{UP} and \textsc{DOWN} trade altitude against exposure to the two hazards, and \textsc{HOVER} holds position while the disturbances still act.
\textsc{Goal} offers the single forced action \textsc{collect} leading to \textsc{done}, and \textsc{Crash} and \textsc{done} offer only a self-loop.

\paragraph{Parameter space.}
Two parameters are unknown, $\theta=(p,q)\in\Theta=[0,p_{\max}]\times[0,q_{\max}]$.
The wind intensity $p$ is the chance that a step is overridden by a one-cell lateral drift, and the actuator-drop probability $q$ is the probability that it is overridden by losing a level of altitude.
The generator requires $p_{\max}+q_{\max}<1$, so $1-p-q>0$ at every admissible valuation and every branch below is a probability with no clamping or renormalization.
The two parameters penalize opposite routes, which is what makes the benchmark interesting: wind hurts the low corridor, altitude drops hurt the high crossing.

\paragraph{Transitions.}
Writing $\mathrm{clip}(\cdot)$ for componentwise clamping into the grid bounds, every grid action has the same three-branch form
\[
\mathbb P_\theta(s'\mid s,a)=
(1-p-q)\,\mathbf 1[s'=\mathrm{redirect}(s_{\mathrm{int}})]
+p\,\mathbf 1[s'=\mathrm{redirect}(s_{\mathrm{wind}})]
+q\,\mathbf 1[s'=\mathrm{redirect}(s_{\mathrm{drop}})],
\]
where $s_{\mathrm{int}}=\mathrm{clip}(s+\mathrm{move}(a))$ is the intended move, the only branch that depends on $a$; $s_{\mathrm{wind}}=\mathrm{clip}(s+(0,1,0))$ is the lateral drift; and $s_{\mathrm{drop}}=(x,y,z{-}1)$ if $z>0$, else $s$ itself, a ground-level skid.
Coincident branches are summed, as in the datacenter benchmark.
The map $\mathrm{redirect}(x,y,z)$ sends a candidate cell to \textsc{Crash} if it lies in the obstacle set $\mathcal O$, to \textsc{Goal} if it lies in the target region $\mathcal T$, and to the grid state $(x,y,z)$ otherwise.

The layout is procedural in the grid extents.
With $y_c=\lfloor L_y/2\rfloor$, $w_x=\max(1,\lfloor L_x/6\rfloor)$, $x_w=\lfloor L_x/2\rfloor-\lfloor w_x/2\rfloor$, $x_p=\lfloor 3L_x/4\rfloor$, and $x_c=L_x-\max(2,\lfloor L_x/8\rfloor)$, the obstacle set $\mathcal O=\mathcal O_{\mathrm{slab}}\cup\mathcal O_{\mathrm{pillar}}\cup\mathcal O_{\mathrm{canopy}}$ consists of a wall spanning $x\in[x_w,x_w{+}w_x)$ at every altitude $z\le L_z{-}2$ except a two-level corridor on the centerline ($y=y_c$, $z\le1$); a two-wide pillar at $x=x_p$, $y\in\{y_c,y_c{+}1\}$, $z\le L_z{-}2$; and a canopy at the top altitude $z=L_z{-}1$ restricted to $x\ge x_c$.
The target region $\mathcal T=\{(L_x{-}1,y,0):y\in\{y_c{-}1,y_c,y_c{+}1\}\}$ is a three-cell landing pad on the far wall.
The generator checks that the start is not an obstacle and that $\mathcal T\cap\mathcal O=\emptyset$.
Together the slab, passable only through a corridor exposed to the wind, and the canopy, blocking the high route just before the goal, force a trade-off between a short low-altitude route hurt by $p$ and a longer high-altitude route hurt by $q$, so the optimal route switches with $(p,q)$.

\paragraph{Rewards and discount.}
The reward is deterministic and parameter-independent: $r(s,a)=1$ for $(\textsc{Goal},\textsc{collect})$ and $0$ for every other reachable state-action pair, so $R_{\max}=1$.
The optimal discounted value at $s_0$ is therefore exactly the discounted probability of eventually landing.
The discount is $\gamma=0.99$, kept close to $1$ because routes take on the order of $L_x$ steps and a smaller discount would wash out the reachability signal over that horizon; this gives $\Vmax=R_{\max}/(1-\gamma)=100$.

\paragraph{Scalability.}
The grid extents $(L_x,L_y,L_z)$ are free parameters of the generator, subject to $L_x\ge8$, $L_z\ge3$, and odd $L_y\ge5$ so that the centerline $y_c$ is a single column.
Every obstacle and goal formula above depends only on $(L_x,L_y,L_z)$, so the low-versus-high trade-off survives at every size, and
\[
\lvert\mathcal S\rvert=3+\bigl\lvert\{(x,y,z)\in\text{grid}:(x,y,z)\notin\mathcal O\cup\mathcal T\}\bigr\rvert=O(L_xL_yL_z).
\]
The benchmark comes in three sizes, listed in \Cref{tab:uav-sizes}, and all share $(p_{\max},q_{\max},\gamma)=(0.25,0.20,0.99)$ and the layout above, differing only in $(L_x,L_y,L_z)$.

\begin{table}[h]
\centering
\begin{tabular}{lccc}
\toprule
Name & $(L_x,L_y,L_z)$ & $\lvert\mathcal S\rvert$ \\
\midrule
\texttt{uav-small}  & $(8,5,3)$   & 98 \\
\texttt{uav-medium} & $(12,9,4)$  & 358 \\
\texttt{uav-large}  & $(24,15,6)$ & 1{,}813 \\
\bottomrule
\end{tabular}
\caption{Sizes of the \texttt{uav} benchmark family used in the experiments.}
\label{tab:uav-sizes}
\end{table}

Every grid cell offers the same seven actions, so the family admits $7^{\lvert\mathcal S\rvert-3}$ memoryless deterministic policies.
Even \texttt{uav-small} therefore rules out exhaustive search over policies.

\section{Experimental data}
\label[appendix]{app:experimental_data}

\paragraph{Regret results (RQ1).}
\Cref{tab:merged-results} gives the seed-level results summarized in \Cref{tab:portfolio-results-summary}.
For each portfolio size $K$ and each benchmark, it reports two quantities.
The ``max-regret'' column is $\widehat\Rreg_{\widehat D}(\Pi_K)$, the largest portfolio regret over that seed's $1000$ sampled valuations; these are the numbers averaged into \Cref{tab:portfolio-results-summary}, and they should be read against the benchmark's $\Vmax$, listed in the header, since the datacenter values live on a scale over twenty times larger than the UAV ones.
The ``inertia'' column is the K-means objective, the within-cluster sum of squared distances from each profile to its assigned center.
It measures how tightly the loss profiles cluster, not how well the portfolio performs, and is included only as a diagnostic: it falls steeply with $K$ on every benchmark, confirming that the candidate profiles do separate into distinguishable groups.

\paragraph{Online selection results (RQ2).}
\Cref{fig:ucb-datacenter} in the main text shows the two datacenter measures.
\Cref{fig:ucb-arm-comparison-all-1,fig:ucb-arm-comparison-all-2} show the same pair of measures for all four benchmarks, one benchmark per row, with all values of $K$ overlaid within each panel.
The left column is the fraction of the $90$ pooled runs whose recommended member is the best portfolio member at their hidden valuation; the right column is that recommendation's normalized shortfall $L^\pi_\theta$, which is $0$ when the best member is recommended and $1$ when the worst is.
Reading the two columns together is what the discussion in \Cref{sec:implementation} relies on: the left column degrades visibly as $K$ grows, since more members mean more near-ties to resolve, while the right column stays low throughout, showing that the members confused for one another are close in value and the cost of the confusion is small.

\begin{table}[t]
  \centering
  \small
  \begin{tabular}{cccccccccc}
    \toprule
     &  & \multicolumn{2}{c}{uav-small} & \multicolumn{2}{c}{uav-medium} & \multicolumn{2}{c}{uav-large} & \multicolumn{2}{c}{data-center-climate} \\
    \cmidrule(lr){3-4} \cmidrule(lr){5-6} \cmidrule(lr){7-8} \cmidrule(lr){9-10}
     &  & \multicolumn{2}{c}{$|S| = 98$} & \multicolumn{2}{c}{$|S| = 358$} & \multicolumn{2}{c}{$|S| = 1813$} & \multicolumn{2}{c}{$|S| = 330$} \\
     &  & \multicolumn{2}{c}{$\Vmax = 100$} & \multicolumn{2}{c}{$\Vmax = 100$} & \multicolumn{2}{c}{$\Vmax = 100$} & \multicolumn{2}{c}{$\Vmax = 2080$} \\
    K & seed & max-regret & inertia & max-regret & inertia & max-regret & inertia & max-regret & inertia \\
    \midrule
    1 & 0 & 0.053 & 23.239 & 0.091 & 42.668 & 0.125 & 59.498 & 14.649 & 1313595.187 \\
    1 & 1 & 0.053 & 23.239 & 0.091 & 42.668 & 0.124 & 59.498 & 13.404 & 1313595.187 \\
    1 & 2 & 0.053 & 23.239 & 0.091 & 42.668 & 0.129 & 59.498 & 14.771 & 1313595.187 \\
    1 & avg & 0.053 & 23.239 & 0.091 & 42.668 & 0.126 & 59.498 & 14.275 & 1313595.187 \\
    \midrule
    2 & 0 & 0.032 & 1.935 & 0.019 & 3.974 & 0.016 & 1.629 & 4.127 & 215193.617 \\
    2 & 1 & 0.032 & 1.935 & 0.019 & 3.974 & 0.016 & 1.629 & 3.413 & 220731.474 \\
    2 & 2 & 0.032 & 1.935 & 0.020 & 3.974 & 0.018 & 1.629 & 4.053 & 215193.617 \\
    2 & avg & 0.032 & 1.935 & 0.019 & 3.974 & 0.017 & 1.629 & 3.864 & 217039.569 \\
    \midrule
    3 & 0 & 0.014 & 0.697 & 0.005 & 1.951 & 0.016 & 0.638 & 2.813 & 72737.797 \\
    3 & 1 & 0.027 & 1.157 & 0.005 & 1.951 & 0.014 & 1.317 & 2.836 & 72737.797 \\
    3 & 2 & 0.014 & 0.697 & 0.020 & 2.152 & 0.018 & 0.638 & 2.809 & 72737.797 \\
    3 & avg & 0.018 & 0.850 & 0.010 & 2.018 & 0.016 & 0.865 & 2.820 & 72737.797 \\
    \midrule
    5 & 0 & 0.002 & 0.094 & 0.003 & 0.193 & 0.016 & 0.107 & 2.809 & 37639.989 \\
    5 & 1 & 0.002 & 0.093 & 0.003 & 0.183 & 0.014 & 0.107 & 2.836 & 40156.574 \\
    5 & 2 & 0.002 & 0.093 & 0.003 & 0.183 & 0.015 & 0.107 & 2.730 & 39166.131 \\
    5 & avg & 0.002 & 0.093 & 0.003 & 0.186 & 0.015 & 0.107 & 2.792 & 38987.564 \\
    \midrule
    7 & 0 & 0.002 & 0.015 & 0.003 & 0.032 & 0.004 & 0.044 & 2.809 & 19076.922 \\
    7 & 1 & 0.002 & 0.012 & 0.003 & 0.032 & 0.003 & 0.044 & 1.170 & 20262.995 \\
    7 & 2 & 0.002 & 0.012 & 0.003 & 0.032 & 0.015 & 0.048 & 1.392 & 19462.184 \\
    7 & avg & 0.002 & 0.013 & 0.003 & 0.032 & 0.008 & 0.045 & 1.790 & 19600.700 \\
    \midrule
    10 & 0 & 0.002 & 0.002 & 0.003 & 0.007 & 0.004 & 0.010 & 0.817 & 8950.955 \\
    10 & 1 & 0.002 & 0.002 & 0.003 & 0.007 & 0.003 & 0.011 & 0.824 & 9239.430 \\
    10 & 2 & 0.002 & 0.002 & 0.003 & 0.009 & 0.002 & 0.010 & 0.759 & 8610.714 \\
    10 & avg & 0.002 & 0.002 & 0.003 & 0.008 & 0.003 & 0.011 & 0.800 & 8933.700 \\
    \bottomrule
  \end{tabular}
  \caption{Full regret results (RQ1).
  For each portfolio size $K$, rows 0--2 report the three clustering seeds and ``avg'' their arithmetic mean.
  ``max-regret'' is the largest sampled portfolio regret.
  ``Inertia'' is the K-means within-cluster sum of squared distances between data points and the cluster centroid, not a regret measure.}
  \label{tab:merged-results}
\end{table}

\begin{figure}[t]
    \centering
    \renewcommand{\arraystretch}{1.1}
    \setlength{\tabcolsep}{2pt}
    \begin{tabular}{@{} >{\centering\arraybackslash}m{1.5em}
                        >{\centering\arraybackslash}m{0.44\linewidth}
                        >{\centering\arraybackslash}m{0.44\linewidth} @{}}
        & \small\textbf{Correct arm} & \small\textbf{Normalised regret of} \\
        & \small\textbf{recommended} & \small\textbf{recommended arm} \\[2pt]
        \rotatebox{90}{\small\textbf{Datacenter}} &
        \includegraphics[width=\linewidth]{figures/datacenter/ucb_pickup_allK_LCB.pdf} &
        \includegraphics[width=\linewidth]{figures/datacenter/ucb_regret_range_allK_LCB.pdf} \\
        \rotatebox{90}{\small\textbf{uav-small}} &
        \includegraphics[width=\linewidth]{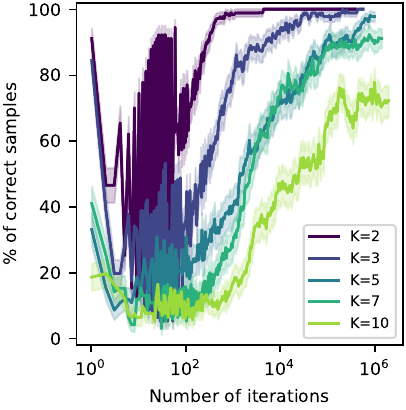} &
        \includegraphics[width=\linewidth]{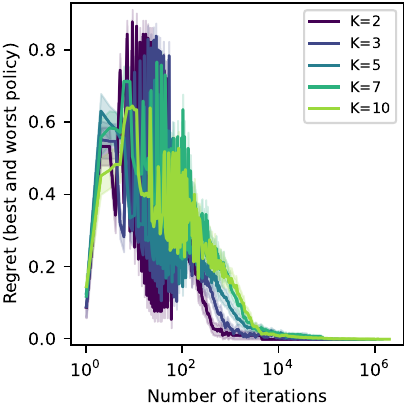} \\
    \end{tabular}
    \caption{Reported UCB deployment results across all tested values of $K$, for the datacenter and uav-small benchmark (RQ2). Each row corresponds to a benchmark; columns show, left to right, the reported correct-arm rate for the recommended arm and the regret of the recommended arm.}
    \label{fig:ucb-arm-comparison-all-1}
\end{figure}

\begin{figure}[t]
    \centering
    \renewcommand{\arraystretch}{1.1}
    \setlength{\tabcolsep}{2pt}
    \begin{tabular}{@{} >{\centering\arraybackslash}m{1.5em}
                        >{\centering\arraybackslash}m{0.44\linewidth}
                        >{\centering\arraybackslash}m{0.44\linewidth} @{}}
        & \small\textbf{Correct arm} & \small\textbf{Normalised regret of} \\
        & \small\textbf{recommended} & \small\textbf{recommended arm} \\[2pt]
        \rotatebox{90}{\small\textbf{UAV-medium}} &
        \includegraphics[width=\linewidth]{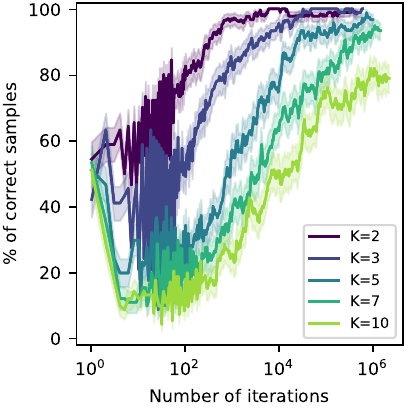} &
        \includegraphics[width=\linewidth]{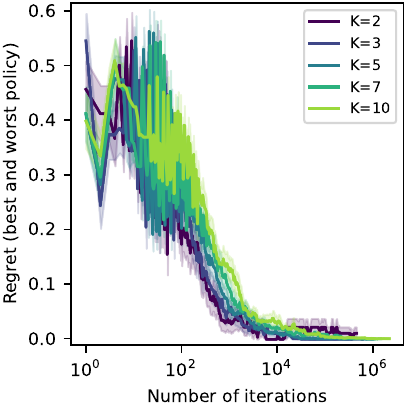} \\
        \rotatebox{90}{\small\textbf{UAV-large}} &
        \includegraphics[width=\linewidth]{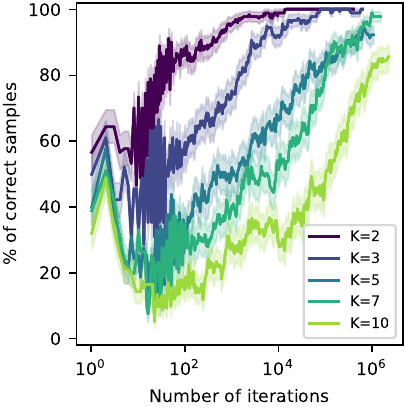} &
        \includegraphics[width=\linewidth]{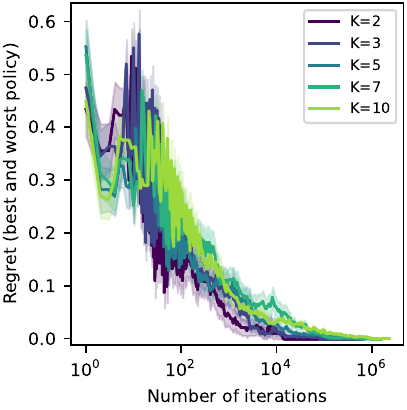} \\
    \end{tabular}
    \caption{Reported UCB deployment results across all tested values of $K$, for the UAV-medium and UAV-large benchmarks (RQ2). Each row corresponds to a benchmark; columns show, left to right, the reported correct-arm rate for the recommended arm and the regret of the recommended arm.}
    \label{fig:ucb-arm-comparison-all-2}
\end{figure}

\begin{algorithm}[t]
\caption{Paired portfolio synthesis and evaluation protocol}
\label{alg:experimental-pipeline-full}

\KwIn{
    benchmark pMDP $M$,
    domain $D$,
    portfolio budgets $\mathcal{K}$,
    seeds $\mathcal{S}$,
    empirical robust regret sample size $n$,
    UCB sample size $m$,
    UCB parameters $(H,\varepsilon,\delta)$
}

\ForEach{$K \in \mathcal{K}$}{
    \tcp{Offline candidate construction.}
    $\textsc{Cells}\leftarrow\textsc{DiscretizeDomain}(D)$\;
    $\textsc{Candidates}\leftarrow
    \textsc{ComputeCandidatePolicies}(M,\textsc{Cells})$\;
    \tcp{Offline robust candidate evaluation.}
    $\textsc{Scores}\leftarrow
    \textsc{RobustPolicyEvaluation}(M,\textsc{Cells},\textsc{Candidates})$\;
    \tcp{Offline mini-max regret evaluation.}
    ComputeMinimaxRegret(Scores, Cells, Candidates)\;
    \ForEach{$s \in \mathcal{S}$}{
        \tcp{Offline construction of portfolio.}
        $\Pi\leftarrow
        \textsc{ReduceAndCluster}(\textsc{Candidates},\textsc{Scores},K,s)$\;
        \tcp{Offline empirical robust regret evaluation of portfolio.}
        $\widehat{D}\leftarrow\textsc{SampleDomain}(D,n,s)$\;
        \ForEach{$\theta\in\widehat{D}$}{
            $V^*\leftarrow\textsc{OptimalValue}(M,\theta)$\;
            $V_\Pi\leftarrow\max_{\pi\in\Pi}
            \textsc{EvaluatePolicy}(M,\theta,\pi)$\;
            $r\leftarrow V^*-V_\Pi$\;
            $\textsc{Store}(K,s,\theta,r)$\;
        }
        Maximise empirical robust regret within one seed\;
        \tcp{Online UCB-based deployment.}
        $\widehat{D}\leftarrow\textsc{SampleDomain}(D,m,s)$\;
        \ForEach{$\theta\in\widehat{D}$}{
            Create generative model $G_\theta$\;
            $z\leftarrow\textsc{RunUCB}(G_\theta,\Pi,H,\varepsilon,\delta)$\;
            $\textsc{Store}(K,s,\theta,z)$\;
        }
    }
    Average empirical robust regret over seeds\;
    Pool UCB samples across seeds\;
}
$\textsc{GenerateReport}()$\;
\end{algorithm}

\end{document}